\documentclass[letterpaper]{article}
\usepackage[preprint]{aaai2027}
\usepackage[hyphens]{url}
\usepackage{graphicx}
\usepackage{natbib}
\usepackage{caption}
\usepackage{mathtools}
\usepackage{booktabs}
\usepackage{multirow}
\usepackage{amssymb,amsmath,amsfonts,amsthm}
\usepackage{mathrsfs}
\usepackage{algorithm, algpseudocode}
\usepackage{comment}

\theoremstyle{plain}
\newtheorem{theorem}{Theorem}
\newtheorem{corollary}{Corollary}
\newtheorem{proposition}{Proposition}
\newtheorem{lemma}{Lemma}
\theoremstyle{definition}
\newtheorem{definition}{Definition}

\theoremstyle{remark}
\newtheorem{remark}{Remark}

\newcommand{\Xb}{\mathbf{X}}

\newcommand{\Zc}{\mathcal{Z}}
\newcommand{\dd}{\mathrm{d}}
\newcommand{\Ab}{\mathbf{A}}

\newcommand{\Zb}{\mathbf{Z}}
\newcommand{\bb}{\mathbf{b}}
\newcommand{\alphab}{\boldsymbol{\alpha}}
\DeclareMathOperator*{\argmax}{arg\,max}
\title{Linear Independent Component Analysis via Optimal Transport}

\author{
    Ashutosh Jha$^{2}$,
    Michel Besserve$^{1,2}$,
    Simon Buchholz$^{1}$
}

\affiliations{
    $^1$Max Planck Institute for Intelligent Systems, T\"ubingen\\
    $^2$Institute of Artificial Intelligence, TU Braunschweig\\
    ashutoshjha3103@gmail.com
}

\begin{document}
\maketitle

\begin{abstract}
Linear Independent Component Analysis (ICA) recovers jointly independent source signals 
from their linear mixtures. To achieve this, classical ICA algorithms attempt to maximize 
non-Gaussianity, measured by negentropy, which is linked to independence by information theory. 
Because exact negentropy optimization is intractable, they rely on proxy contrast functions, 
such as fourth-order cumulants and parametric log-likelihoods. We propose instead to use the squared $L_2$-Wasserstein distance to a standard Gaussian as the ICA contrast. 
We show that the Wasserstein distance between a standard normal distribution and linear 
projections of the data  is maximized when the projection recovers an independent component, 
and that under a regularity condition on the sources this maximum is separated from every 
genuine mixture by an explicit margin. We uncover the advantageous properties of the resulting estimator: for sources with a smooth 
density it is $\sqrt{N}$-consistent and asymptotically normal with a closed-form variance, 
whereas for a source with an atom the contrast has a kink at the true direction, which makes the estimator exact with probability tending to one. The proposed OT-ICA algorithm finds this projection by gradient-based optimization. Empirical evaluation on simulated data shows that OT-ICA outperforms proxy-based methods and that the $W_2^2$ contrast provides 
more robust signal than proxy contrasts in measuring non-Gaussianity for different 
distributional mixtures of the latent variables. Application to linear causal disentanglement 
and EEG artifact removal, along with further applications detailed in the appendix, 
confirms OT-ICA can be used for applied ICA tasks without distributional assumptions.
\end{abstract}

\section{Introduction}\label{sec:intro}
Independent Component Analysis (ICA) \citep{jutten1985, comon1994} is a fundamental representation learning technique with applications that span EEG/MEG artifact separation, CDMA telecommunications, image processing \citep[Chapters 21--23]{hyvarinen2001}, econometric price discovery \citep{zema2025}, and causal representation learning \citep{shimizu2006linear, scholkopf2021toward}. Linear ICA recovers mutually independent source signals $\mathbf{Z}$ from a linear mixture $\mathbf{X} = \mathbf{A}\mathbf{Z}$. Identifiability holds under the following conditions.

\begin{theorem}[Linear ICA \citep{comon1994}]\label{th:linear_ica}
Let $\mathbf{X} \in \mathbb{R}^d$ be observed data and $\mathbf{Z} \in \mathbb{R}^d$ latent sources. 
Assume (i) $\mathbf{X} = \mathbf{A}\mathbf{Z}$ with $\mathbf{A}$ invertible, (ii) $Z_1, \dots, Z_d$ are mutually independent, 
and (iii) at most one $Z_i$ is Gaussian. Then $\mathbf{A}$ is identifiable up to column permutation and scaling.
\end{theorem}

Algorithms for learning $\mathbf{A}$ from data generally rely on the property that linear mixtures of independent non-Gaussian sources are more Gaussian than their constituents \citep{hyvarinen2001}; indeed, in the limit of many sources, convergence to a Gaussian follows from the central limit theorem.
A rigorous statement of this property was given in \citet{cardoso2022}.

\begin{theorem}[Independence and Non-Gaussianity \citep{comon1994, cardoso2022}]\label{th:pythagorean_ig_ica}
Let $\boldsymbol{Y} \in \mathbb{R}^d$ be whitened ($\operatorname{Cov}(\boldsymbol{Y})=\mathbf{I}$). The mutual information $\mathscr{I}(\boldsymbol{Y})$, which measures statistical dependence among components $Y_1,\ldots,Y_d$, satisfies:
\begin{equation}
    \mathscr{I}(Y) = G(Y) - \sum_{i} G(Y_i)
\end{equation}
where $G(\cdot)$ denotes non-Gaussianity defined as the minimal KL divergence to a Gaussian distribution. 
\end{theorem}
This result implies that minimizing $\mathscr{I}(Y)$ is equivalent to maximizing total
marginal non-Gaussianity $\sum_i G(Y_i)$. However, evaluating $G(Y_i)$ requires the
marginal density.

\textbf{Related work.} Classical algorithms therefore replace $G(Y_i)$ with proxy contrasts:
FastICA \citep{hyvarinen2001} uses logcosh negentropy approximations,
JADE \citep{cardoso1993blind} fourth-order cumulants,
InfoMax \citep{bell1995information} a parametric log-likelihood,
Picard \citep{ablin2018faster} an adaptive nonlinearity with a provably converging solver, and rank-based entropy estimation methods \citep{vasicek1976test, learnedmiller2003ica}, which leverage sample spacings to approximate differential entropy non-parametrically. Rank-based R-estimators \citep{ilmonen2011semiparametrically, hallin2015r} instead estimate the unmixing matrix from rank statistics of the same kind as our empirical contrast; those of \citet{ilmonen2011semiparametrically} are semiparametrically efficient for symmetric sources at correctly specified densities. The proxy contrasts exhibit shortcomings such as evaluating to zero on specific non-Gaussian distributions, or inducing singular Hessian matrices (see Appendix~\ref{app:proxy_limits}). These limitations suggest that novel approaches to linear ICA are worth exploring, especially given the importance of this problem in multiple applications.

Optimal Transport (OT) \citep{monge1781, kantorovich1942} has become an alternative to KL divergence-based approaches in many machine learning problems, such as generative modelling \citep{pmlr-v70-arjovsky17a}, normality testing \citep{delbarrio1999}, and distribution comparison \citep{peyre2019computational}. Closest to our work is Wasserstein projection pursuit \citep{mukherjee2026105535}, an instance of projection pursuit \citep{friedman1974projection, huber1985projection} that maximizes the Wasserstein distance between orthogonal projections of the data and a standard Gaussian to estimate a non-Gaussian signal subspace; it recovers the span of the non-Gaussian directions, whereas ICA needs each maximizer to be an individual source. Max-sliced and projection-robust Wasserstein distances \citep{deshpande2019max, paty2019subspace, lin2020projection} also maximize a transport cost over projections, but between two data distributions and for its value rather than its maximizer, and Gaussianization methods \citep{chen2000gaussianization, laparra2011iterative, dai2021sliced} apply the one-dimensional transport map to a Gaussian for density estimation rather than source recovery. OT has further been used for blind source separation via non-negative matrix factorization \citep{rolet2018blind}, for causal discovery \citep{tu2022optimal}, and to extend linear ICA to observations that are themselves probability measures, where the Wasserstein distance provides the geometry of the data rather than the contrast \citep{li2025linear}.

Finally, concurrent and independent work
\citep{laplante2026contrastfreeicacausalinference} also proposed the squared
$L_2$-Wasserstein distance to a Gaussian as an ICA contrast. They show the same bounds as in Theorem~\ref{th:main} and Corollaries~\ref{cor:max_contrast} and~\ref{cor:symmetric} and have a statistical analysis showing that the empirical contrast converges to the population contrast for all projections.
They also apply the contrast to sequential regression residuals to identify and estimate causal orders in linear non-Gaussian acyclic models.
Their preprint appeared on arXiv after our work was accepted at the 9th Workshop on Tractable Probabilistic Modeling (TPM) at UAI 2026, while our paper was still not publicly available. The margin and asymptotic
results of Section~\ref{sec:properties} are new in the present version.

\textbf{Our contribution.} In this paper, we use OT to define a contrast and use it to solve linear ICA: we quantify non-Gaussianity as the cost of transporting the (empirical) projection of the observations to a fixed standard Gaussian reference. We justify this choice in Theorem~\ref{th:main}, which implies that (in the population setting) this contrast is maximized when an independent component is recovered. This then motivates our OT-ICA algorithm, which relies on gradient ascent on the optimal transport cost. We further show that this identification fails for Wasserstein orders $1\le p<2$ and for all sufficiently large $p$ (Lemmas~\ref{lem:wp_high} and~\ref{lem:wp_low}), that the maximum is separated from every mixture by an explicit margin (Proposition~\ref{cor:margin}), and that the resulting estimator is $\sqrt{N}$-consistent with a closed-form asymptotic variance for smooth sources (Proposition~\ref{prop:rates}) and exact with probability tending to one when the target source has an atom (Proposition~\ref{prop:atoms}).

\section{The OT-ICA Framework}\label{sec:OT-framework}

We consider the standard linear ICA setting $\Xb = \Ab\Zb$ where $\Ab\in \mathbb{R}^{d\times d}$ is an invertible matrix. By preprocessing the observations $\Xb$ through centering and whitening, we can assume without loss of generality that $\mathbb{E}(\mathbf{X})=0$ and $\operatorname{Cov}[{\mathbf{X}}] = \mathbf{I}_{d\times d}$ (details in Appendix~\ref{app:info_geometry}). Then we can constrain the unmixing matrix $\mathbf{B}$ to the Orthogonal Group $\mathrm{O}(d)$ (the set of $d\times d$ real matrices satisfying $\mathbf{B}\mathbf{B}^\top = \mathbf{I}$), and seek $\mathbf{B}\in \mathrm{O}(d)$ such that $\mathbf{B}\Xb=\mathbf{B}\Ab\Zb$ has independent components, or equivalently, $\mathbf{B}$ inverts $\Ab$ up to a signed permutation matrix. 
This can be done by iteratively seeking optimal projection directions
$\mathbf{b}\in\mathbb{R}^d$, one for each row of
$\mathbf{B}$, that maximize a non-Gaussianity function $G(\mathbf{b}^\top \Xb)$ 
of the projected data. 

We will use the square of the $L_2$-Wasserstein distance $W_2$ \citep{villani2003}, 
defined as the cost of an optimal  coupling between probability distributions for $p=2$:
\begin{equation}
    W_p^p(\mu, \nu) \!\triangleq\!W_p(\mu, \nu)^p\!=\!\!\!  \inf_{\pi \in \Pi(\mu, \nu)} \int\!\! \|x - y\|^p  d\pi(x, y) ,
\end{equation}
where $\Pi(\mu,\nu)$ denotes (for measures $\mu,\nu$ on $\mathbb{R}^d$) the set of all measures $\pi$ on $\mathbb{R}^{d+d}$ with first marginal $\mu$ and second marginal $\nu$.
The Wasserstein distance is well-defined as soon as the $p$-th moment of $\mu$ and $\nu$ are finite.
For random variables X and Y we use the shorthand 
$W_p(X,Y)=W_p(\mathbb{P}_X,\mathbb{P}_Y)$.

\textbf{OT-ICA for a single component.}
We consider the contrast given by
$W_2^2\big(\mathbf{b}^\top \widetilde{\mathbf{X}}, \Gamma\big)$, 
where $\Gamma \equiv \mathcal{N}(0,1)$ denotes the standard Gaussian. This distance is finite when the second moment of $\widetilde{\mathbf{X}}$ exists.
Since  $W_2$ induces a proper metric on probability distributions \citep{villani2003}, this contrast  is non-zero on all non-Gaussian distributions.

OT-ICA maximizes this contrast over the direction $\mathbf{b}$
\begin{equation}
    \widehat{\mathbf{b}} = \argmax_{\mathbf{b} \in \mathbb{R}^d,\; \|\mathbf{b}\|_2=1}
    W_2^2\big({\mathbf{b}^\top \widetilde{\mathbf{X}}}, \Gamma\big).
    \label{eq:otica_obj}
\end{equation}

The following theorem establishes that the projection $\widehat{\mathbf{b}}^\top \Xb$ from \eqref{eq:otica_obj} indeed recovers an independent component. The bound previously appeared in \citet{johnson2005central}, who established it in the course of studying central limit theorems in Mallows distance. Our contribution is to conclude that this property implies that $W_2^2$ is an admissible ICA contrast, which Corollaries~\ref{cor:max_contrast} and~\ref{cor:symmetric} make precise.

\begin{theorem}[$W_2^2$ ICA Contrast, Proposition 2.3 in \cite{johnson2005central}]\label{th:main}
Let $\mathbf{Z} = (Z_1, \dots, Z_d)^\top$ be centered independent sources with unit
variance and at most one $Z_i$ Gaussian. Then for any weight vector
$\alphab \in \mathbb{R}^d$ with $\|\alphab\|_2 = 1$ and at least two nonzero components,
\begin{equation}
    W_2^2(\alphab^\top \mathbf{Z}, \Gamma) < \sum_{i=1}^d \alpha_i^2 W_2^2(Z_i, \Gamma).
\end{equation}
\end{theorem}

For completeness, the proof of this result (for random variables with smooth densities) is in Appendix~\ref{app:proofs}.
\newline
\textit{Proof Sketch.}
Let $\Gamma_i$ be independent standard normal random variables
and consider an optimal  coupling of $(Z_i, \Gamma_i)$
such that $W_2^2(Z_i,\Gamma_i)=\mathbb{E}(|Z_i-\Gamma_i|^2)$. Then a direct calculation shows that
\begin{align}
\begin{split}
    W_2^2\Big(\sum_i \alpha_i Z_i,\Gamma\Big)
      &= W_2^2\Big(\sum_i \alpha_i Z_i,\sum_i \alpha_i\Gamma_i\Big)
    \\ &\leq \mathbb{E}\Big(\sum_i \alpha_i(Z_i-\Gamma_i)\Big)^2
    \\ &= \sum_i \alpha_i^2\, W_2^2(Z_i,\Gamma_i).
\end{split}
\end{align}
To show the strict inequality one exploits that in 1 dimension the optimal coupling is comonotone and one infers that the constructed coupling can only be comonotone when the transport plans are linear which can only be true when $Z_i$ itself is Gaussian.

We find the following immediate corollary.
\begin{corollary}\label{cor:max_contrast}
    Under the assumptions of Theorem~\ref{th:main}
    we have
    \begin{align}
        \label{eq:argmax}
        \argmax_{\|\alphab\|_2=1} W_2^2(\alphab^\top \mathbf{Z}, \Gamma)
        \subset \{\pm\mathbf{e}_1,\ldots, \pm\mathbf{e}_d\},
    \end{align}
    i.e.\ the contrast is maximal only for independent components.
\end{corollary}
This corollary implies that
\eqref{eq:otica_obj} is maximized by $\widehat{\bb}$ such that
$\widehat{\bb}^\top \Xb=\widehat{\bb}^\top \Ab\Zb=\Zb_i$ for some $i$, i.e., when the projection recovers an independent component.

\textbf{OT-ICA for all components.}
After whitening the data, the projection directions associated with all independent components are mutually orthogonal. While components can be extracted iteratively by searching the shrinking orthogonal complement one by one, Theorem \ref{th:main} extends to the following Corollary \ref{cor:symmetric}, which justifies the joint estimation of the entire unmixing matrix $\mathbf{B} \in \mathrm{O}(d)$.

\begin{corollary}\label{cor:symmetric}
    Under the assumptions of Theorem \ref{th:main}, for any orthogonal matrix $\mathbf{B} \in \mathrm{O}(d)$ with rows $\mathbf{b}_i^\top$, the total marginal contrast satisfies:
    \begin{equation}
        \sum_{i=1}^d W_2^2(\mathbf{b}_i^\top \mathbf{Z}, \Gamma) \le \sum_{j=1}^d W_2^2(Z_j, \Gamma),
    \end{equation}
    where equality holds if and only if $\mathbf{B}$ is a signed permutation matrix.
\end{corollary}

 Corollary \ref{cor:symmetric} justifies \emph{symmetric joint optimization} \citep{hyvarinen2001}  where we learn all independent components simultaneously.

\section{Statistical properties of the $W_2^2$ Contrast}\label{sec:properties}
Corollaries~\ref{cor:max_contrast} and~\ref{cor:symmetric} establish that the contrast
identifies the sources in the population setting. In this section, we first study statistical 
properties of the corresponding empirical estimators.
Then we provide a quantitative version of Theorem~\ref{th:main} and investigate whether the Wasserstein contrast for $p\neq 2$ is suitable for linear ICA.

\textbf{Estimation of the maximizer.}
Because the contrast is evaluated only on
one-dimensional projections, the empirical $W_2$ converges at the one-dimensional rate
$\mathcal{O}(N^{-1/2})$,  avoiding the general
$O(N^{-1/d})$ rate of $d$-dimensional transport \citep{fournier2015on,manole2024sharp} and we remark that 
this is precisely the motivation of the  sliced Wasserstein distances
\citep{bonneel2015sliced,kolouri2019generalized} which consider similar projections.
These rates characterize the approximation rate of the
contrast. How accurately the recovered direction itself is estimated depends on the local shape of the contrast at a true source, which we characterize similarly to the asymptotic-variance analysis for FastICA in \citet{tichavsky2006performance}. Fix a target source $Z_\ell$. Written in source coordinates, any candidate direction with a positive loading on $Z_\ell$ extracts, up to normalization, the variable
\begin{equation}
Z_\ell+\sum_{k\ne\ell}\sigma_k Z_k,
\label{eq:interference}
\end{equation}
so $\boldsymbol\sigma\in\mathbb{R}^{d-1}$ records how much of each remaining source leaks into the extracted component, and $\boldsymbol\sigma=\mathbf{0}$ is exact recovery. We denote by $\widehat{\boldsymbol\sigma}$ a
local maximizer of  the empirical contrast 
 near $Z_\ell$. The following Proposition characterizes 
 the asymptotic distribution of $\hat{\sigma}$ and establishes that we achieve the usual $\sqrt{N}^{-1}$ convergence rate. Note that a more precise statement containing existence of a maximizer $\widehat{\boldsymbol{\sigma}}$
 is given in Theorem~\ref{th:asymptotic}.

\begin{proposition}\label{prop:rates}
Let $\psi_\ell:=\Phi^{-1}\circ F_\ell$ be the normal-score transform of $Z_\ell$, where $\Phi$ and $F_\ell$ are the c.d.f.s of $\Gamma$ and $Z_\ell$, and let $T_\ell:=\mathbb{E}[Z_\ell\psi_\ell(Z_\ell)]$ and $B_\ell:=\mathbb{E}[\psi_\ell'(Z_\ell)]$.
If $Z_\ell$ is non-Gaussian and has a smooth density $f_\ell$ with finite Fisher information $I(Z_\ell)=\int (f_\ell'/f_\ell)^2 f_\ell<\infty$, then
\begin{equation}
\sqrt{N}\,\widehat{\boldsymbol\sigma}\;\rightsquigarrow\;\mathcal{N}\big(\mathbf{0},\,V_\ell\,\mathbf{I}_{d-1}\big)
\end{equation}
with $V_\ell=(1-T_\ell^2)/(B_\ell-T_\ell)^2$. 
\end{proposition}
The proof of this result (along with the fully rigorous statement) can be found in Appendix~\ref{app:asymptotic_variance}.
Both functionals have a direct interpretation. $T_\ell$ is the correlation between 
the source and its own normal score, and it already fixes the contrast value, 
since $W_2^2(Z_\ell,\Gamma)=2-2T_\ell$ for a standardized source. Moreover, $B_\ell$ is the 
mean slope of the normal-score transform, and $B_\ell-T_\ell$ is the curvature of the 
criterion at the true direction once the whitening normalization is taken into account. 
So $V_\ell$ is a noise-to-curvature ratio, the fluctuation of the criterion's 
slope divided by the square of its curvature: the sharper the peak, the better 
the direction is resolved.

Proposition~\ref{prop:rates} applies to sources with a smooth density. A surprising finding is that 
the estimator behaves differently for sources with atoms where it is able to recover the source exactly with probability tending to one, unlike estimators maximizing a smooth contrast, such as FastICA.

\begin{proposition}\label{prop:atoms}
Suppose $Z_\ell$ has an atom. Then the contrast has a kink at the true
direction: there is $c>0$ such that
\begin{align}\begin{split}
&W_2^2\left((1+|\sigma|^2)^{-1/2}\left(Z_\ell+\sum_{k\neq \ell}\sigma_k Z_k\right),\Gamma\right)
\\
&\qquad\qquad\le W_2^2(Z_\ell,\Gamma)-c|\sigma|+o(|\sigma|).
\end{split}
\end{align}
For the empirical distribution 
(normalized to unit variance)
the bound holds up to an
additional $O_P(N^{-1/2})|\sigma|$ term. In particular,  with probability tending to one the empirical
contrast has a local maximum exactly at the true direction.
\end{proposition}
Appendix~\ref{app:atoms} proves the population bound and sketches the result for the empirical estimator, both for arbitrary $d$.

\textbf{Global separation.}
Theorem~\ref{th:main} and Corollary~\ref{cor:max_contrast} show that the true sources are aligned with maximizers of the Wasserstein contrast
but they do not provide any quantitative bound on the suboptimality of mixtures.
This is relevant because spurious local maximizers generally exist. 
Here we show that under a regularity
condition on the sources we can lower bound the gap. A random variable $Z$ is
\emph{$\kappa$-strongly log-concave} if it has a density $e^{-V}$ with $V''\ge\kappa>0$
(standardization forces $\kappa\le1$).

\begin{proposition}
    \label{cor:margin}
Let $\mathbf{Z}$ satisfy the assumptions of Theorem~\ref{th:main} and let $\alphab\in\mathbb{R}^d$ satisfy $\sum_{i=1}^d\alpha_i^2=1$. If in addition each $Z_i$ is $\kappa$-strongly log-concave for some $\kappa>0$, then, writing $v_i:=W_2^2(Z_i,\Gamma)$,
\begin{equation}
    \sum_{i=1}^d \alpha_i^2 v_i - W_2^2(\alphab^\top \mathbf{Z},\Gamma)
    \;\ge\; \frac{\sqrt{\kappa}}{2}\Big(1-\|\alphab\|_4^4\Big)\sum_{i=1}^d \alpha_i^2 v_i.
\end{equation}
\end{proposition}
The proof can be found in Appendix~\ref{app:proofs}.
Using this proposition we find writing
$M=\max_{\lVert \alphab'\rVert_2 =1}
   W_2^2(\alphab'^\top \mathbf{Z},\Gamma)$ that 
\begin{align}
   W_2^2(\alphab^\top \mathbf{Z},\Gamma)
    \leq  M - M \frac{\sqrt{\kappa}}{2}\Big(1-\|\alphab\|_4^4\Big).
\end{align}
The prefactor $1-\|\alphab\|_4^4=\sum_{i\ne j}\alpha_i^2\alpha_j^2$ vanishes for the standard basis vectors $\alphab=\pm\mathbf{e}_i$ but it grows as $\alphab$ mixes more sources. Therefore, the achievable Wasserstein contrast is well separated from the maximal value for mixtures.

The proof uses strong log-concavity only to lower bound the modulus of strong monotonicity of the transport map from $\alphab^\top\mathbf{Z}$ to $\Gamma$, and Appendix~\ref{app:proofs} states the bound in terms of that modulus, which is also positive for compactly supported sources whose density is bounded away from zero on an interval, such as the uniform. It vanishes, making the bound vacuous, when the sources have tails heavier than Gaussian or carry atoms.

\textbf{Admissibility of other Wasserstein metrics.}
Finally, we study the natural question whether the more general family of $p$-Wasserstein distances can be used as a contrast function.
We have two negative results in this direction, the first covers $p\gg2$, and the second $1\leq p<2$.
\begin{lemma}\label{lem:wp_high}
There is a $p^\ast>2$ such that for  all $p>p^\ast$ there are i.i.d.\ random variables $Z_1$, $Z_2$
    such that
    \begin{align}
        W_p(\sqrt{2}^{-1}(Z_1+Z_2),\Gamma)>
        W_p(Z_1,\Gamma).
    \end{align}
\end{lemma}
\begin{lemma}\label{lem:wp_low}
    For every $1\leq p<2$  there are i.i.d.\ random variables $Z_1$, $Z_2$
    such that
    \begin{align}
        W_p(\sqrt{2}^{-1}(Z_1+Z_2),\Gamma)>
        W_p(Z_1,\Gamma).
    \end{align}
\end{lemma}
Proofs can be found in Appendix~\ref{app:proofs}.
This shows that Corollary~\ref{cor:max_contrast} does not extend to these Wasserstein distances $W_p$, i.e., maximizing $W_p(\mathbf{b}^\top \mathbf{\tilde{X}},\Gamma)$ over the sphere in general will not recover the independent sources even in two dimensions. Indeed, we show that $\mathbf{b}= \sqrt{2}^{-1}\begin{pmatrix}
    1&1
\end{pmatrix}$ can have a larger contrast than $\mathbf{b}=\mathbf{e}_i$. The case $p=2$ is special because we can exploit orthogonality relations in the Hilbert space $L^2$.
Whether any $p>2$ satisfies a property similar to \eqref{eq:argmax} is an interesting theoretical problem. However, it seems unlikely that $W_p$ for $p>2$ can provide a better contrast than $W_2$ for linear ICA because estimation becomes more challenging as the sensitivity to outliers increases.

\section{The OT-ICA Algorithm}\label{sec:algorithm}
We now discuss how the observations in the previous sections can be leveraged to learn independent components.
OT-ICA optimizes the Wasserstein contrast
$W_2^2\big(\mathbf{b}^\top \widetilde{\mathbf{X}}, \Gamma\big)$ through gradient ascent. 
This requires a fully differentiable implementation of the Wasserstein distance. This is straightforward in the one dimensional case, where an optimal coupling is given by matching the percentiles of the distributions, such that the Wasserstein distance is given by
\begin{equation}
    W_2^2(\mu, \nu) = \int_0^1 |F_\mu^{-1}(t) - F_\nu^{-1}(t)|^2 \, dt
\end{equation}
where $F_\mu$ denotes the CDF of $\mu$. In particular, the optimal coupling for finitely many samples simply matches the ordered values from the two empirical distributions.

The sorted criterion is the following: For a fixed direction the empirical contrast of a standardized projection equals $2-\frac2N\sum_i y_{(i)}z_i$, a normal-scores rank statistic with the Gaussian bin means $z_i=N\int_{(i-1)/N}^{i/N}\Phi^{-1}(u)\,\mathrm{d}u$ as scores (Appendix~\ref{app:asymptotic_variance}); as a test of Gaussianity for a single sample this is the $W_2$ goodness-of-fit statistic of \citet{delbarrio1999}. What is new here is its behaviour as a function of the projection: maximizing it over the sphere recovers the independent components, with the estimator law stated in Propositions~\ref{prop:rates}
and~\ref{prop:atoms} above. Spacing-based ICA contrasts \citep{vasicek1976test, learnedmiller2003ica} instead estimate differential entropy and require a spacing parameter; here the reference law is fixed and no such choice arises.

Corollary~\ref{cor:symmetric} justifies joint estimation, so instead of iterative deflation we optimize all $d$ rows of $\mathbf{B}$ simultaneously by ascending the total contrast sum, restoring the $\mathrm{O}(d)$ constraint via symmetric decorrelation retraction at every step. 

Additionally, three important algorithmic components utilized in the OT-ICA framework are: 
(i)~\textbf{Analytical Gaussian Targets}, replacing sampled quantile targets of the standard normal distribution 
with closed-form bin expectations to eliminate approximation noise; (ii)~\textbf{Riemannian Gradient Retraction} \citep{absil2008optimization}, 
projecting gradients onto the tangent space of Orthogonal Group $\mathrm{O}(d)$ and retracting via symmetric decorrelation; and 
(iii)~\textbf{Gaussian Dithering} \citep{schuchman1964dither}, smoothing discrete CDFs via narrow Gaussian convolution (detail in Appendix~\ref{app:methodology}). The full OT-ICA algorithm is shown in Algorithm \ref{alg:ot_ica} below.

\begin{algorithm}[!tb]
\caption{The Optimal Transport ICA (OT-ICA) Algorithm}
\label{alg:ot_ica}
\begin{algorithmic}[1]
\Require Observed mixture $\mathbf{X} \in \mathbb{R}^{d \times N}$; initialisations $K$; iterations $T$; learning rate $\eta$; batch size $N_{\text{batch}}$; dither $\sigma_{\text{dither}}$
\Ensure Estimated unmixing matrix $\mathbf{B}_{\text{final}} \in \mathbb{R}^{d \times d}$
\State \textbf{Phase 1: Preprocessing \& Target Generation}
\State Center the data: $\mathbf{X} \gets \mathbf{X} - \mathbb{E}[\mathbf{X}]$
\State Whiten the data: $\widetilde{\mathbf{X}} \gets (\mathbf{X}\mathbf{X}^\top\!/N)^{-1/2} \mathbf{X}$
\State Compute the quantile target $\mathbf{z} \in \mathbb{R}^{N_{\text{batch}}}$ of $\Gamma$ via CDF integration
\State \textbf{Phase 2: Initialization}
\State Initialize $\mathbf{B} \in \mathbb{R}^{d \times d}$ randomly from $\mathcal{N}(0,1)$
\State Apply symmetric decorrelation: $\mathbf{B} \gets (\mathbf{B}\mathbf{B}^\top)^{-1/2} \mathbf{B}$
\State \textbf{Phase 3: Stochastic Riemannian Optimization}
\For{$t = 1$ to $T$}
    \State Sample mini-batch $\widetilde{\mathbf{X}}_{\text{batch}} \in \mathbb{R}^{d \times N_{\text{batch}}}$ from $\widetilde{\mathbf{X}}$
    \State Project: $\mathbf{Y} \gets \mathbf{B} \widetilde{\mathbf{X}}_{\text{batch}}$
    \State Dither: $\mathbf{Y}_{\text{dither}} \gets \mathbf{Y} + \mathcal{N}(0, \sigma_{\text{dither}}^2)$
    \State Sort rows of $\mathbf{Y}_{\text{dither}}$ ascending to get $\mathbf{Y}_{\text{sorted}}$
    \State Objective: $\mathcal{J} = \frac{1}{N_{\text{batch}}} \sum_{i=1}^{d} \|\mathbf{Y}_{\text{sorted}}^{(i)} -  \mathbf{z}\|_2^2$
    \State Euclidean gradient: $\boldsymbol{\Delta} = \nabla_{\mathbf{B}} \mathcal{J}$
    \State Tangent projection: $\nabla_{\mathrm{O}(d)} \mathbf{B} = \boldsymbol{\Delta} - \frac{1}{2}\big(\boldsymbol{\Delta}\mathbf{B}^\top + \mathbf{B}\boldsymbol{\Delta}^\top \big)\mathbf{B}$
    \State Ascend: $\mathbf{B}_{\text{step}} = \mathbf{B} + \eta_t \nabla_{\mathrm{O}(d)} \mathbf{B}$
    \State Retract: $\mathbf{B} \gets (\mathbf{B}_{\text{step}}\mathbf{B}_{\text{step}}^\top)^{-1/2} \mathbf{B}_{\text{step}}$
\EndFor
\State \textbf{Phase 4: Finalization}
\State \Return $\mathbf{B}_{\text{final}} \gets \mathbf{B} (\mathbf{X}\mathbf{X}^\top\!/N)^{-1/2}$
\end{algorithmic}
\end{algorithm}

In practice Phase 2 is warm-started by a deflationary sweep that extracts each row in turn
from $K$ random unit vectors, keeping the one of highest undithered full-sample contrast;
this changes only the initialisation. We use $\eta_0 \in [0.05, 1]$ decayed geometrically,
$T = 200$ to $800$, $N_{\text{batch}} = 1024$, $\sigma_{\text{dither}} \in [0, 0.1]$ and
$K = \min(4d, 150)$ in the main benchmark; per-experiment values are in Appendix~\ref{app:parameters}.

\textbf{Computational Complexity.} Per iteration, symmetric-mode
FastICA costs $\mathcal{O}(d^2N + d^3)$ while OT-ICA has complexity
$\mathcal{O}(K(d^2N_{\text{batch}} + dN_{\text{batch}}\log N_{\text{batch}} + d^3))$
where $K$ denotes the number of restarts.
The two share the same d-scaling; OT-ICA additionally incurs a log N sorting overhead and a factor K for restarts.

\section{Synthetic Experiments}\label{sec:validation}

\begin{figure*}[!tb]
    \centering
    \includegraphics[width=0.88\textwidth]{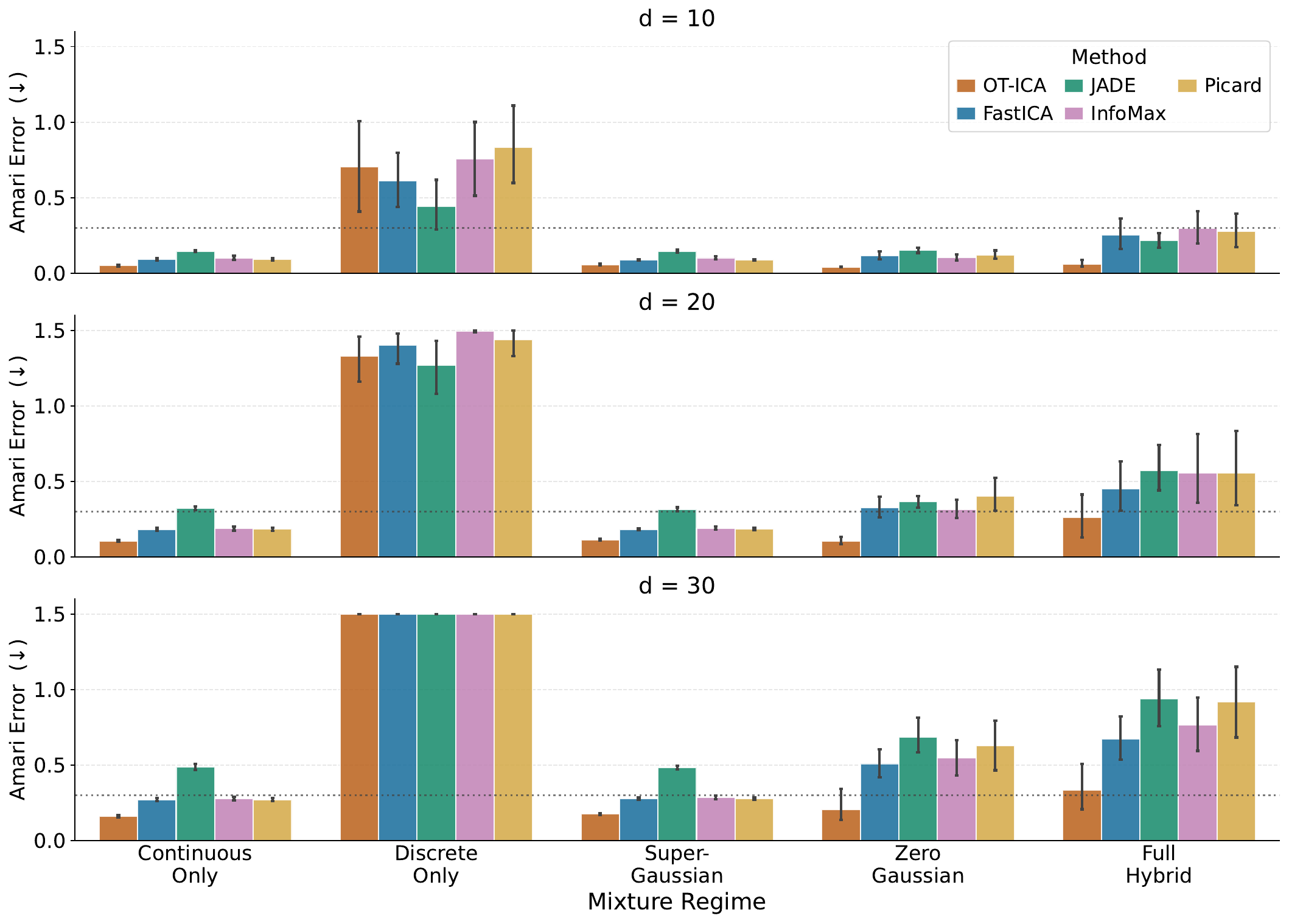}
    \caption{Mean Amari error ($\downarrow$ better, the display ceiling clipped at $1.5$) across five mixture regimes and $d \in \{10, 20, 30\}$ ($N=10{,}000$, 10 trials, 95\% CI). 
    Dotted line marks $E=0.3$, the good-separation threshold.}
    \label{fig:hybrid_mixture_results}
\end{figure*}

We first evaluate our method on synthetic data.
Separation quality is measured by the Amari Performance Index ($E$) \citep{amari1996new}, 
which measures the deviation of the global transfer matrix $\mathbf{P} = \mathbf{B}_{\text{est}}\mathbf{A}$ from a generalized permutation matrix, one
with a single non-zero entry in every row and column. We adopt the heuristic of $E < 0.1$ as near-perfect recovery, $E < 0.3$ as good separation and $E \ge 0.5$ as failure.
As a sanity check we separate four different sources with $N=10{,}000$ samples, where
OT-ICA achieves $E = 0.016$ against FastICA's $0.019$; the recovered densities and the
global transfer matrices for both algorithms are in Appendix~\ref{app:validation_table}. Note that all experiments, synthetic and applied, were conducted using \textit{Python 3.10} on an \textit{Ubuntu 22.04} workstation with an \textit{AMD Ryzen 9} CPU, $64$GB RAM, and an \textit{NVIDIA RTX 4090} GPU. Statistical significance tests for performance margins on all experiments and applications can be found in Appendices~\ref{app:experimental_setup}, \ref{app:discrete} and \ref{app:disentangle}.

We now study the empirical performance of the proposed algorithm.
We use FastICA \citep{hyvarinen2001}, JADE \citep{cardoso1993blind}, InfoMax \citep{bell1995information}, and 
Picard \citep{ablin2018faster} as baselines and consider five mixture regimes at $d \in \{10, 20, 30\}$, $N = 10{,}000$, 10 trials per condition. 
The results can be found in Figure~\ref{fig:hybrid_mixture_results}.

\textbf{Continuous and heterogeneous regimes.} OT-ICA achieves the lowest Amari error on
all four continuous and mixed-type configurations at every dimension. On \textit{Continuous
Only} mixtures (Laplace, Uniform, Student-$t$, Chi-square, Exponential and one Gaussian)
it reduces error $40$--$45\%$ over the second-best FastICA at all three dimensions, and by
$35$--$40\%$ on \textit{Strictly Super-Gaussian} and $60$--$68\%$ on \textit{Zero Gaussian}
mixtures. The margin is largest on \textit{Full Hybrid} mixtures, which add Bernoulli,
Poisson and Binomial sources: $E = 0.059$, $0.261$, $0.335$ at $d = 10, 20, 30$ against
FastICA's $0.255$, $0.451$, $0.671$, an error reduction of $2$--$4\times$
(sign test, paired by trial across the ten seeds, Appendix~\ref{app:sign_test_main}).

\textbf{Discrete sources.} On \textit{Discrete Only} mixtures (Bernoulli, Poisson, Binomial) every solver fails at $d \geq 10$, with errors above $E = 0.5$, where differences reflect the degree of failure rather than separation quality. 
There are two  hypotheses: either the
contrasts lose their signal on count-based marginals, or the signal is present but the optimisation landscape prevents solvers from reaching it. We now disentangle the different effects.

\begin{figure*}[!tb]
    \centering
    \includegraphics[width=0.9\textwidth]{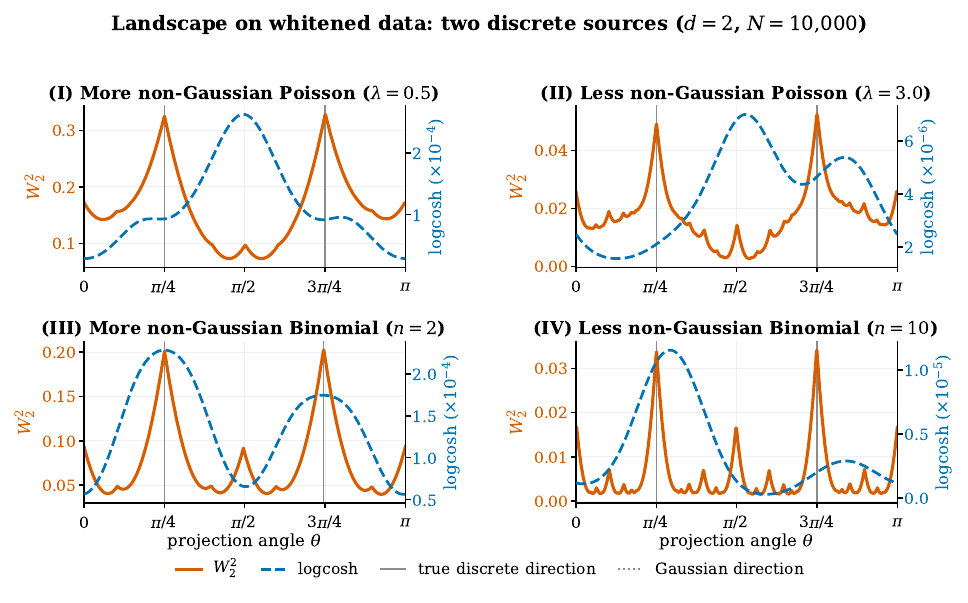}
    \caption{Contrast against projection angle for two discrete sources ($d=2$,
    $N=10^4$), evaluated on the whitened mixture. The whitened basis is rotated to
    place the first source at $\pi/4$; the second lies at $3\pi/4$. Both contrasts are
    $\pi$-periodic, so the half-period shown is the complete landscape. Panels are indexed
    (I)--(IV) and referred to by these numerals throughout this section.}
    \label{fig:optim_surfaces}
\end{figure*}

\begin{table}[!tb]
\centering
\small
\setlength{\tabcolsep}{5pt}
\begin{tabular}{lcccc}
\toprule
& \multicolumn{2}{c}{$W_2^2$} & \multicolumn{2}{c}{Logcosh} \\
\cmidrule(lr){2-3}\cmidrule(lr){4-5}
Distribution & value & $z$-score & value & $z$-score \\
\midrule
Laplace (baseline)      & 0.0329 & 412 & $1.13\times 10^{-3}$ & 1322.6 \\
Poisson ($\lambda=0.5$) & 0.3394 & 4269 & $2.22\times 10^{-4}$ & 258.6 \\
Poisson ($\lambda=3.0$) & 0.0517 & 649 & $3.88\times 10^{-6}$ & 3.8 \\
Binomial ($n=2$)        & 0.2022 & 2542 & $2.25\times 10^{-4}$ & 262.2 \\
Binomial ($n=10$)       & 0.0342 & 427 & $1.09\times 10^{-5}$ & 12.0 \\
\bottomrule
\end{tabular}
\caption{Non-Gaussianity detectability on isolated marginals ($N=10^4$). \emph{value} is
the raw contrast; \emph{$z$-score} provides a comparable metric by scoring the contrast against its own empirical null distribution on Gaussian samples. The failure to disentangle discrete sources is therefore unlikely to originate from an insufficient $W_2^2$ contrast ($z \ge 412$ in all cases); the landscape analysis that follows identifies where the maximum sits, which is the operative quantity.}
\label{tab:w2_sensitivity}
\end{table}

\begin{figure}[!tb]
    \centering
    \includegraphics[width=0.9\columnwidth]{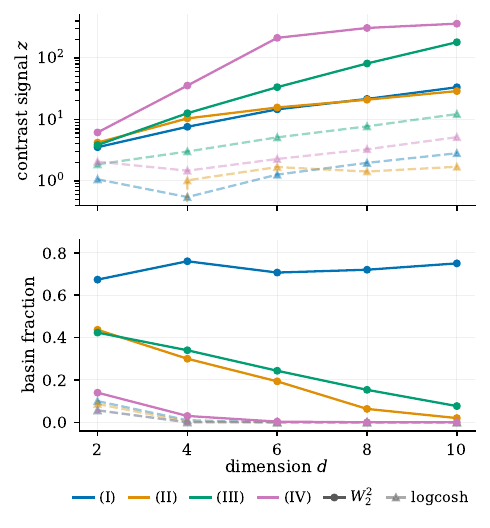}
    \caption{Contrast signal (top) and basin fraction (bottom) against dimension, for
    $W_2^2$ (solid) and logcosh (dashed). For $W_2^2$ the signal is retained while the
    basin collapses, so the objective still identifies the solution but random
    initialisation fails to reach it. For logcosh both contrast and basin are lost.}
    \label{fig:scaling_diagnostic}
\end{figure}

Table~\ref{tab:w2_sensitivity} scores \textit{Wasserstein} and \textit{logcosh} on isolated discrete marginals at low dimension ($d=2$) as a $z$-statistic against the null distribution of the same contrast on Gaussian samples. $W_2^2$ yields highly significant $z$-scores for all discrete count distributions ($z \ge 427$), on par with or exceeding the continuous Laplace baseline ($z = 412$). The logcosh proxy, however, collapses on weakly non-Gaussian parameterizations ($z = 3.8$ for Poisson $\lambda=3.0$; $z = 12.0$ for Binomial $n=10$), despite being extremely sensitive to Laplace ($z = 1322.6$): the proxy is well-specified on continuous mixtures but loses its signal on discrete ones.

Figure~\ref{fig:optim_surfaces} plots both contrasts against projection angle for two-source discrete mixtures. Projections are taken on the whitened mixture $\mathbf{w}(\theta)^\top\widetilde{\mathbf{X}}$, which removes the mixing only up to an
orthogonal factor. We rotate the whitened basis, an orthogonal change of coordinates that leaves the landscape shape unchanged, to place the first source at $\pi/4$ and the second at $3\pi/4$. The $W_2^2$ maximum falls within $0.12^\circ$ of a true source direction in every configuration.
A further interesting observation is that the Wasserstein contrast as a function of the rotation has a kink at its maximum.

This kink is not a numerical artifact but the corner from Proposition~\ref{prop:atoms}, which is why the estimator recovers the true directions with probability tending to one rather than at the $O(N^{-1/2})$ rate that holds for continuous sources. The exact-recovery statement says that an optimizer which reaches a neighbourhood of a true direction returns the true direction, not that the search finds such a neighbourhood. The two parts are complementary, and together they sharpen our diagnosis that the failure on discrete mixtures is the search algorithm and not the contrast: atoms make the target sharp, while the region from which that target is reachable shrinks with increasing dimensions.

We discuss this in more detail in Appendix~\ref{app:atoms}.
On the other hand the logcosh maximum is displaced by up to $44^\circ$ (Appendix~\ref{app:discrete}). Configuration (III), the one case in which logcosh localises correctly, is also the one in which FastICA attains its lowest Amari error, so the landscape diagnostic predicts the end-to-end outcome. 
Figure~\ref{fig:optim_surfaces} also shows that $W_2^2$ develops spurious local maxima between the true directions once every source is discrete.
This already indicates that optimization can be challenging. 

We now extend these observations from $d=2$ to higher dimensions.  Figure~\ref{fig:scaling_diagnostic} reports two quantities: the contrast
signal, the objective at the true directions expressed as a $z$-score against $1500$ random unit directions, and the basin of attraction fraction, the share of $300$ random unit vectors from which projected gradient ascent reaches a true direction. The second studies the random
initialisation OT-ICA performs in Phase 2 (Algorithm~\ref{alg:ot_ica}), so it estimates the per-restart success probability, and $K$ initialisations succeed with probability $1-(1-p)^K$. The contrast signal does not decay, growing from $z \approx 4$ at $d=2$ to $z = 29$ to $361$ at
$d=10$, although part of this rise is mechanical, since random projections become more Gaussian as $d$ increases. The fraction of initializations landing in the basin of attraction of a true component instead collapses, from $0.44$ to $0.02$ in configuration (II) and
from $0.14$ to $0.000$ in (IV), with only the strongly non-Gaussian (I) retaining its
basin; under logcosh it is empty by $d=6$ throughout. 
This reveals that the fraction of the sphere covered by the basin of attraction of the true discrete sources decreases rapidly with the dimension $d$ 
for the Wasserstein contrast and even more so for the FastICA algorithm. On the other hand, the contrast signal is strong for the Wasserstein distance while being much smaller for the FastICA contrast.

Table~\ref{tab:w2_sensitivity} supplies a natural experiment: Binomial $n=10$ has the
same $W_2^2$ detectability as Laplace ($z$-score\ $427$ against $412$) yet is the one
configuration OT-ICA fails on at $d=2$, with a single-restart success rate of $25\%$
against $67$ to $100\%$ elsewhere (Appendix~\ref{app:discrete}). This isolates landscape
geometry as the operative variable. We conclude that $W_2^2$ is the better-conditioned
contrast for discrete mixtures, and that the residual failure on \textit{Discrete Only}
mixtures is an optimisation problem: the basin around each true direction shrinks with
dimension and with decreasing non-Gaussianity, so random initialisation converges to
spurious local maxima, a challenge shared with the FastICA algorithm.

\section{Applications}\label{sec:applications}

We detail two applications below, linear causal disentanglement and EEG artifact removal;
two further applications, structural identification in price discovery and causal
discovery on the Sachs protein-signalling network, are given in
Appendix~\ref{app:price} and Appendix~\ref{app:sachs} respectively.

\subsection{Linear Causal Disentanglement}
\label{sec:disentangle}
Causal disentanglement recovers a latent linear structural causal model (SCM) together with its mixing map from
data collected across interventional environments \citep{squires2023linear}. In
environment $k$ the latents follow
$\mathbf{Z}=\mathbf{A}_k\mathbf{Z}+\boldsymbol{\Omega}_k^{1/2}\boldsymbol{\varepsilon}$ with
$\mathrm{Cov}(\boldsymbol{\varepsilon})=\mathbf{I}$, $\boldsymbol{\Omega}_k$ diagonal positive-definite and $\mathbf{A}_k$ encoding the causal structure. Observations are
$\mathbf{X}=\mathbf{G}\mathbf{Z}$ for full-column-rank
$\mathbf{G}\in\mathbb{R}^{p\times d}$. Writing
$\mathbf{M}_k=\boldsymbol{\Omega}_k^{-1/2}(\mathbf{I}-\mathbf{A}_k)$ gives
$\mathbf{X}=\mathbf{G}\mathbf{M}_k^{-1}\boldsymbol{\varepsilon}$, so assuming independent noise $\boldsymbol{\varepsilon}$, each environment is a
linear ICA problem with unmixing matrix $\mathbf{B}_k=\mathbf{M}_k\mathbf{G}^\dagger$,
where $\mathbf{G}^\dagger$ is the Moore-Penrose pseudoinverse and
$\mathbf{G}^\dagger\mathbf{G}=\mathbf{I}_d$ by full column rank. Environment $0$ corresponds to no interventions. Environment $k\geq 1$ corresponds to a perfect intervention on node $i_k$ (i.e. all parents are removed). 

\citet{squires2023linear} use only second moments, $\boldsymbol{\Theta_k}=\mathbf{B}_k^\top\mathbf{B}_k$,
which fix $\mathbf{B}_k$ up to an orthogonal factor; they recover this factor using
partial-order RQ decomposition and rank tests. 
This does not require any distributional assumptions on $\boldsymbol{\varepsilon}$.
Here we study the case where $\boldsymbol{\varepsilon}$ is non-Gaussian.
Then we can identify $\mathbf{B}_k$ up to
a signed permutation, and two structural facts then make recovery direct. Single-node
interventions edit one row, $\mathbf{M}_k=\mathbf{M}_0+\mathbf{e}_{i_k}\mathbf{c}_k^\top$,
so $\mathbf{B}_k-\mathbf{B}_0$ has exactly one non-zero row. Under \emph{perfect}
interventions, which sever a node's dependence on its parents while leaving it stochastic,
the parent terms cancel, leaving
$(\mathbf{B}_k)_{i_k,:}=\lambda_k(\mathbf{G}^\dagger)_{i_k,:}$: the intervened row of the
unmixing matrix is, up to scale, a row of $\mathbf{G}^\dagger$, and the first fact
identifies which row. The full procedure, which assembles $\mathbf{H}=\mathbf{G}^\dagger$
row by row using no rank test, is given in Appendix~\ref{app:disentangle}. Statistical significance tests of the estimated causal graphs and error distributions can be found in the same appendix.

We simulate $d=4$ latents and $p=8$ observed variables with $n=2\times10^4$ samples, one
observational and $d$ interventional environments, over $10$ seeds
(Table~\ref{tab:disentangle}). Substituting the exact unmixing matrices for the estimated
ones leaves the recovery correct to $10^{-14}$ on all $20$ seeds tested, so every error
reported below originates in the ICA step. We report \emph{target accuracy}, the fraction
of environments whose intervention target is correctly identified; \emph{F1}, the harmonic
mean of edge precision and recall against the true support of $\mathbf{A}_0$; and
\emph{SHD}, the structural Hamming distance, counting the edge insertions, deletions and
reversals needed to turn the estimated graph into the true one.

Under Laplace exogenous shocks, both solvers optimize their contrasts well and recover the causal graph. A single count-valued shock reveals FastICA driving the demixing error $\|\Delta\mathbf{H}\|_F$ to $0.125$ against OT-ICA's $0.002$, where $\|\Delta\mathbf{H}\|_F$ measures the error in estimating the pseudoinverse $\mathbf{H} = \mathbf{G}^\dagger$ of the causal matrix $\mathbf{G}$. At $d=4$, the three remaining continuous Laplace sources act as anchors, allowing FastICA to still recover the causal graph. However, when all sources are discrete, such as low-rate Poisson counts, FastICA's proxy collapses: its Amari error rises to $0.835$ (versus OT-ICA's $0.008$), intervention target accuracy drops to $0.48$ (near $1/d=0.25$ chance), graph F1 falls to $0.49$, and SHD jumps from $0$ to $7.7$. OT-ICA recovers every target and edge on all ten seeds.  A failed contrast therefore does not yield a degraded graph but an uninformative one.

Semi-synthetic shocks from PBMC 3k UMI marginals \citep{pbmc3k_dataset} locate this \textit{logcosh} failure in real data, where single-cell RNA sequencing fundamentally measures gene expression as discrete transcript counts. Genes are stratified by Fano factor $\sigma^2/\mu$, separating variation due to technical sampling (Poisson) from true biological signal (overdispersed). Overdispersed genes (Fano $\approx11$) retain heavy right tails yielding highly significant \textit{logcosh} $z$-scores (Table~\ref{tab:w2_sensitivity}), so both solvers perform well. Near-Poisson genes (Fano $1.3$ to $1.7$) lack these tails, causing \textit{logcosh} to fail. This confirms with real data that an inadequate ICA contrast fatally compromises causal recovery, establishing OT-ICA as the better default for linear causal disentanglement across both continuous and discrete biological environments (sign test across seeds, Appendix~\ref{app:disentangle}; the Laplace regime, where both solvers are expected to tie, is the one case not significant).

\begin{table}[!tb]
\centering
\small
\setlength{\tabcolsep}{4pt}
\begin{tabular}{llcccc}
\toprule
Shocks & Method & Amari & $\|\Delta\mathbf{H}\|_F$ & target & F1 \\
\midrule
\multirow{2}{*}{Laplace}
 & OT-ICA  & 0.019 & 0.066 & 1.00 & 1.00 \\
 & FastICA & 0.019 & 0.064 & 1.00 & 1.00 \\
\multirow{2}{*}{Count (int.)}
 & OT-ICA  & \textbf{0.014} & \textbf{0.002} & 1.00 & 1.00 \\
 & FastICA & 0.024 & 0.125 & 1.00 & 1.00 \\
\multirow{2}{*}{Count (all)}
 & OT-ICA  & \textbf{0.008} & \textbf{0.031} & \textbf{1.00} & \textbf{1.00} \\
 & FastICA & 0.835 & 3.400 & 0.48 & 0.49 \\
\multirow{2}{*}{PBMC, near-Poisson}
 & OT-ICA  & \textbf{0.009} & \textbf{0.021} & \textbf{1.00} & \textbf{1.00} \\
 & FastICA & 0.063 & 0.467 & 0.95 & 0.95 \\
\multirow{2}{*}{PBMC, overdispersed}
 & OT-ICA  & \textbf{0.008} & \textbf{0.028} & 1.00 & 1.00 \\
 & FastICA & 0.012 & 0.042 & 1.00 & 1.00 \\
\bottomrule
\end{tabular}
\caption{Causal disentanglement, means over $10$ seeds ($d=4$, $p=8$, $n=2\times10^4$).
$\|\Delta\mathbf{H}\|_F$ is the Frobenius error in
the demixing map, under the scaling convention of
\citet{squires2023linear}; \emph{target} is the fraction of intervention targets correctly
identified. Standard deviations, a scale-invariant demixing error, and a leakage regime
violating the single-node assumption are in Appendix~\ref{app:disentangle}.}
\label{tab:disentangle}
\end{table}

\subsection{Electroencephalography (EEG)  Artifact Removal}
As the skull acts as a linear volume conductor \citep{hyvarinen2001}: ocular blink artifacts mix linearly with neural electrical signals before reaching scalp electrodes. Blinks are impulsive and strongly super-Gaussian, so their large $W_2^2$ distance to the Gaussian
separates them from neural components without specifying a density model. We apply OT-ICA to $20$ EEG channels of the MNE sample dataset \citep{gramfort2013meg}, bandpass filtered (1--40\,Hz, zero phase) and $z$-scored over a $60$\,s segment
containing nine blinks. Components are ranked by excess kurtosis $\kappa$ (zero for Gaussian); the highest-$\kappa$ component is zeroed and the signal reconstructed (Figure~\ref{fig:eeg_artifact}). Selection uses kurtosis alone, so the concurrently recorded EOG (Electrooculogram) channel is an independent check.

\begin{figure}[!tb]
    \centering
    \includegraphics[width=0.9\columnwidth]{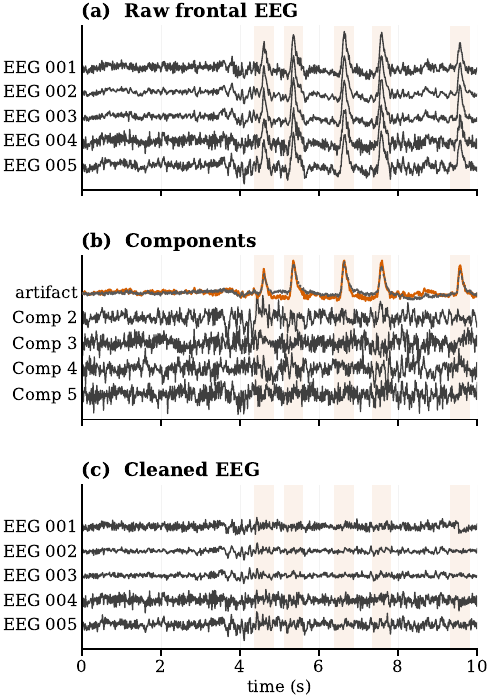}
    \caption{OT-ICA artifact removal (MNE sample dataset). Unmixing is computed on
    $20$ channels over $60$\,s; panels show five frontal channels over the $10$\,s
    window with the most blinks, shading marking $\pm250$\,ms around EOG-detected
    blinks. \textbf{(b)} highest-kurtosis component (orange, sign fixed for display)
    against the EOG channel (dashed). \textbf{(c)} reconstruction, same scale as (a).}
    \label{fig:eeg_artifact}
\end{figure}

The artifact component attains $\kappa=56.5$ against $+0.25$ for the remaining components and correlates with the EOG trace at $|r|=0.87$; it is also
the component of maximal EOG correlation, confirming that the kurtosis criterion recovers the ocular source rather than an arbitrary heavy-tailed direction. Removal reduces root-mean-square (RMS) amplitude by $74\%$ within $\pm250$\,ms of detected blinks and $17\%$
elsewhere, retaining $99.7\%$ of $13$--$30$\,Hz and $80\%$ of $8$--$12$\,Hz power outside blink windows, with the loss concentrated below $4$\,Hz where blink transients live. FastICA yields an equivalent decomposition ($\kappa=56.9$, $|r|=0.87$, $73\%$ and $16\%$ reductions), showing that OT-ICA performs equally well on this low dimensional continuous mixture application.

\section{Conclusion}\label{sec:conclusion}
OT-ICA replaces proxy contrast functions with the $W_2^2$ metric computed by quantile
sorting. Theorem~\ref{th:main} and Corollary~\ref{cor:symmetric} establish that the
contrast is maximized only at independent components, Proposition~\ref{cor:margin}
quantifies by how much, and the asymptotic analysis of
Section~\ref{sec:properties} classifies the resulting estimation rate by the
local shape of the contrast at the true direction, from $\sqrt{N}$-normality for smooth sources to
exact recovery for atomic sources. Empirical results confirm
that OT-ICA outperforms proxy-based algorithms on continuous and heterogeneous
mixtures, with the margin widening on heterogeneous mixtures where proxy contrasts fail
to detect non-Gaussianity. EEG artifact removal, Linear Causal
Disentanglement and further applications in Appendices~\ref{app:price} and~\ref{app:sachs} demonstrate application to ICA tasks without
distributional assumptions. On purely discrete sources every ICA method degrades as $d$
grows. For OT-ICA the cause is not the contrast: $W_2^2$ retains its signal
and localises the source directions, but the basin of
attraction around each direction shrinks with dimension, so random initialisation
reaches it with vanishing probability; per-iteration cost also grows with $d$. These
limitations point to concrete improvements: initialisation strategies targeted at
narrow basins, improved optimization algorithms,  Sinkhorn distances \citep{cuturi2013sinkhorn, peyre2019computational} as
a differentiable substitute for Gaussian dithering, and amortized optimal transport to
reduce per-inference sorting cost. Beyond these, OT-ICA's distribution-free contrast
opens new directions in causal discovery, including extensions to LiNGAM
\citep{shimizu2006linear}, Causal Component Analysis \citep{liang2023causal}, and
identifiable nonlinear ICA \citep{buchholz2022function}.

\section{Use of Large Language Models}
We used Claude Opus (Anthropic) in preparing this paper and document its role here.
\emph{Writing:} checking grammar, phrasing and \LaTeX{}, and reviewing the manuscript
to suggest changes to its structure.
\emph{Theory:} the proof strategy and precise statement of Proposition~\ref{cor:margin}
emerged from extended discussion with the LLM; the authors then substantially condensed
and simplified the proof. The remaining results of Appendix~\ref{app:proofs} were
developed by the authors and streamlined with LLM assistance. For
Appendix~\ref{app:asymptotic_variance}, the LLM proposed proof strategies and references;
the authors adapted the overall strategy, extracting the main ideas from several
over-optimistic sketches. The results of Appendix~\ref{app:atoms} were developed by the
authors based on the empirical observation using general and partly imprecise claims the LLM made about atomic distributions.
\emph{Code:} writing and debugging experiment and plotting code for Jupyter notebooks, 
optimizing PyTorch operations in core script for OT-ICA.
The authors verified all statements, proofs, references and code, and take full
responsibility for the content of this paper.

\bibliography{aaai2027}

\newpage
\onecolumn
\setcounter{secnumdepth}{3}

\begin{center}
    \Large\textbf{Appendix}
\end{center}

\noindent\textbf{Code and reproducibility.} Code to reproduce every figure and table in
this paper (main text and appendix) is available at
\url{https://github.com/ashutoshjha3103/ot_in_linear_ica}.

\appendix

\section{Background and Motivation}
\label{app:info_geometry}

\subsection{Centering, Whitening, and the Orthogonal Group $\mathrm{O}(d)$ Search Space}

The ICA model $\mathbf{X} = \mathbf{A}\mathbf{S}$ has $d^2$ free parameters in the mixing matrix $\mathbf{A}$.
Two lossless preprocessing steps reduce this to a $\tfrac{d(d-1)}{2}$-dimensional rotational search.

\textbf{Centering.} Subtracting the mean, $\mathbf{X} \gets \mathbf{X} - \mathbb{E}[\mathbf{X}]$, ensures the data has zero mean.
If $\mathbf{X} = \mathbf{A}\mathbf{S}$, centering $\mathbf{X}$ is equivalent to centering $\mathbf{S}$, leaving the mixing geometry intact.

\textbf{Whitening.} Given the eigendecomposition $\mathrm{Cov}(\mathbf{X}) = \mathbf{V}\boldsymbol{\Lambda}\mathbf{V}^\top$, the whitening transform $\mathbf{Q} = \boldsymbol{\Lambda}^{-1/2}\mathbf{V}^\top$ yields $\widetilde{\mathbf{X}} = \mathbf{Q}\mathbf{X}$ with $\mathrm{Cov}(\widetilde{\mathbf{X}}) = \mathbf{I}$ \citep{hyvarinen2001}.
After whitening, the unmixing operation $\mathbf{B}\widetilde{\mathbf{X}}$ must satisfy $\mathrm{Cov}(\mathbf{B}\widetilde{\mathbf{X}}) = \mathbf{B}\mathbf{B}^\top = \mathbf{I}$, so $\mathbf{B}$ lies in Orthogonal Group $\mathrm{O}(d)$.
The original mixing matrix decomposes as $\mathbf{A} = \mathbf{Q}^{-1}\mathbf{B}$, so ICA on whitened data is lossless \citep{cardoso2022}.

\subsection{The CLT Motivation for Non-Gaussianity}

The Central Limit Theorem (CLT) provides the following intuition: a normalized sum of independent non-Gaussian variables converges to a Gaussian as the number of terms grows \citep{hyvarinen2001}. In ICA, any linear mixture $\mathbf{b}\cdot\mathbf{S}$ with multiple non-zero weights is therefore \emph{more} Gaussian than the individual sources. Recovering an independent source means finding the projection that is \emph{least} Gaussian, motivating the maximization of non-Gaussianity as the ICA objective.

\subsection{Information Geometry: Connecting Independence and Non-Gaussianity}

Following \citet{cardoso2022}, ICA can be understood geometrically via two Pythagorean identities.
Let $G(\cdot) = D_{\mathrm{KL}}(\cdot \,\|\, \mathcal{N}(\mathrm{Cov}\,\cdot))$ denote non-Gaussianity and $C(Y) = D_{\mathrm{KL}}(\mathcal{N}(\mathrm{Cov}\,Y) \,\|\, \mathcal{N}(\mathrm{diag\,Cov}\,Y))$ denote correlation.

\textbf{Product manifold identity.} For any product distribution $Q = \prod_i q_i$:
\begin{equation}
    D_{\mathrm{KL}}(P_Y \Vert Q) = \mathscr{I}(Y) + \sum_{i} D_{\mathrm{KL}}(P_{Y_i} \Vert q_i).
\end{equation}

\textbf{Gaussian manifold identity.} For any Gaussian $\mathcal{N}(\Sigma)$:
\begin{equation}
    D_{\mathrm{KL}}(P_Y \Vert \mathcal{N}(\Sigma)) = G(Y) + D_{\mathrm{KL}}(\mathcal{N}(\mathrm{Cov}\,Y) \Vert \mathcal{N}(\Sigma)).
\end{equation}

Evaluating the KL divergence from $P_Y$ to $\mathcal{N}(\mathrm{diag\,Cov}\,Y)$ via both paths yields Cardoso's identity:
\begin{equation}
    \mathscr{I}(Y) + \sum_{i} G(Y_i) = G(Y) + C(Y).
\end{equation}
After whitening, $\mathrm{Cov}(Y) = \mathbf{I}$ forces $C(Y) = 0$, giving $\mathscr{I}(Y) = G(Y) - \sum_i G(Y_i)$.
Since $G(Y)$ is invariant under orthogonal rotation, minimizing $\mathscr{I}(Y)$ (independence) equals maximizing $\sum_i G(Y_i)$ (total marginal non-Gaussianity). 

\subsection{Optimal Transport as the Non-Gaussianity Contrast}

Classical ICA algorithms approximate $G(Y_i)$ with proxies. OT-ICA instead measures it directly as the cost of transporting the empirical marginal to a standard Gaussian.

\begin{definition}[Coupling and Push-Forward]
A \emph{coupling} of distributions $\mu$ and $\nu$ is a joint distribution $\pi$ on $\mathbb{R}^d \times \mathbb{R}^d$ with marginals $\mu$ and $\nu$; the set of all such couplings is $\Pi(\mu,\nu)$.
A measurable map $T\colon\mathbb{R}^d\to\mathbb{R}^d$ \emph{pushes forward} $\mu$ to $\nu$ (written $T_\#\mu = \nu$) if $\nu(B) = \mu(T^{-1}(B))$ for every Borel set $B$.
\end{definition}

The $p$-Wasserstein distance is the minimum transport cost over all couplings:
\begin{equation}
    W_p(\mu,\nu) = \left(\inf_{\pi\in\Pi(\mu,\nu)} \int \|x-y\|^p\,d\pi(x,y)\right)^{1/p}.
\end{equation}
The structure of the optimal coupling $\pi$ varies qualitatively depending on whether the marginals are discrete, continuous, or mixed (\citep{peyre2019computational}).
In 1D, Brenier's theorem guarantees the optimal coupling is induced by the monotone quantile map $T = F_\nu^{-1}\circ F_\mu$, so $W_2^2(\mu,\nu) = \int_0^1 |F_\mu^{-1}(t) - F_\nu^{-1}(t)|^2\,dt$.

Taking $\nu = \Gamma \equiv \mathcal{N}(0,1)$, the quantity $W_2^2(\mathbb{P}_{\mathbf{b}^\top\widetilde{\mathbf{X}}},\Gamma)$ measures how far the projection $\mathbf{b}^\top\widetilde{\mathbf{X}}$ is from Gaussian.
Unlike negentropy proxies, $W_2^2$ is a true metric, geometry-aware, and computable by sorting without density estimation.

\section{Mathematical Proofs for OT-ICA Bounds}
\label{app:proofs}

\subsection{Full Proof of Theorem~\ref{th:main}: $W_2^2$ ICA Contrast}

\setcounter{theorem}{2}
\begin{theorem}[$W_2^2$ ICA Contrast]
Let $\mathbf{Z} = (Z_1, \dots, Z_d)^\top$ be  centered independent sources with unit variance and at most one $Z_i$ Gaussian. Assume they all have a smooth and strictly positive density. Then the following bound holds for  any normalized weight vector $\mathbf{\alpha}$ with at least two non-zero components
\begin{equation}
    W_2^2(\mathbf{\alpha} \cdot \mathbf{Z}, \Gamma) < \sum_{i=1}^d \alpha_i^2 W_2^2(Z_i, \Gamma).
\end{equation}
\end{theorem}

\textit{Strategy.} We first establish a weak upper bound on $W_2^2(\mathbf{\alpha}\cdot\mathbf{Z},\Gamma)$ (used as an intermediate step), then show the bound is strict for any proper mixture under the non-Gaussianity conditions.

\medskip\noindent\textit{Step 1: Upper Bound (intermediate).}

\begin{proof}
     We construct a specific joint probability distribution, or coupling, using a common source of randomness. 
Let $\mathbf{N} = (N_1, \dots, N_d)^\top$ be a vector of $d$ independent standard normal variables, $\mathbf{N} \sim \mathcal{N}(0, \mathbf{I}_d)$, 
assumed to be statistically independent of the sources $\mathbf{Z}$. For each independent latent source $Z_i$, there exists 
an optimal transport map $T_i$ that pushes forward the standard normal distribution to the source distribution, denoted as $(T_i)_{\#} \Gamma = Z_i$. 

We construct a random variable $X$ representing the mixture:
\begin{equation}
    X = \sum_{i=1}^d \alpha_i T_i(N_i).
\end{equation}
To compare this mixture to a Gaussian distribution, we construct a target variable $Y$ using the same underlying noise vector $\mathbf{N}$:
\begin{equation}
    Y = \sum_{i=1}^d \alpha_i N_i.
\end{equation}
Given that the weight vector has unit norm ($\sum_i \alpha_i^2 = 1$), the resulting variable $Y$ follows the standard normal distribution, $Y \sim \Gamma$. 
The pair $(X, Y)$ creates a valid coupling. Because $W_2^2$ is the infimum over all valid couplings, the optimal distance is less than or equal to the cost of our constructed coupling:
\begin{equation}
    W_2^2(\mathbf{\alpha} \cdot \mathbf{Z}, \Gamma) \le \mathbb{E}\left[|X - Y|^2\right] = \mathbb{E}\left[ \left( \sum_{i=1}^d \alpha_i \big(T_i(N_i) - N_i\big) \right)^2 \right].
\end{equation}
Squaring the summation produces diagonal terms ($i=j$) and cross terms ($i \neq j$):
\begin{equation}
    |X - Y|^2 = \sum_{i=1}^d \alpha_i^2 \big(T_i(N_i) - N_i\big)^2 + \sum_{i \neq j} \alpha_i \alpha_j \big(T_i(N_i) - N_i\big)\big(T_j(N_j) - N_j\big).
\end{equation}
Because the underlying noise variables $N_i$ and $N_j$ are independent, and the functions evaluate to a mean of zero, when we take the expectation, 
all cross terms vanish. We are left with the expectation of the diagonal terms:
\begin{equation}
    \mathbb{E}\left[|X - Y|^2\right] = \sum_{i=1}^d \alpha_i^2 \mathbb{E}\left[ \big(T_i(N_i) - N_i\big)^2 \right].
\end{equation}
Recognizing that $\mathbb{E}[ (T_i(N_i) - N_i)^2 ]$ is the definition of the squared Wasserstein distance between the individual source $Z_i$ and $\Gamma$, we arrive at the bound:
\begin{equation}
    W_2^2(\mathbf{\alpha} \cdot \mathbf{Z}, \Gamma) \le \sum_{i=1}^d \alpha_i^2 W_2^2(Z_i, \Gamma). \quad 
\end{equation}

\medskip\noindent\textit{Step 2: Strict Inequality (proof of Theorem~\ref{th:main}).}
In a one-dimensional setting, the $W_2$ distance equals the expected squared difference of a specific coupling 
if and only if the random variables are comonotonic. By Brenier's Theorem \citep{brenier1991}, this requires the gradients 
to be parallel at every point in the probability space:
\begin{equation}
    \nabla X(\mathbf{N}) = \lambda(\mathbf{N}) \nabla Y(\mathbf{N}).
\end{equation}
Assuming the source distributions possess smooth and strictly positive densities, Caffarelli's regularity theory \citep{caffarelli1992} 
guarantees that the optimal transport maps $T_i$ are continuously differentiable. 

We differentiate both sides to compute their Hessian matrices. On the LHS, the cross-derivatives $\frac{\partial^2 X}{\partial N_i \partial N_j}$ are zero, 
yielding a diagonal matrix of $\alpha_i T''_i(N_i)$. On the RHS, we differentiate the product $\lambda(\mathbf{N}) \mathbf{b}$, yielding a Rank-1 outer product matrix:
\begin{equation}
    \underbrace{
    \begin{bmatrix} 
    \alpha_1 T''_1(N_1) & 0 & \cdots & 0 \\
    0 & \alpha_2 T''_2(N_2) & \cdots & 0 \\
    \vdots & \vdots & \ddots & \vdots \\
    0 & 0 & \cdots & \alpha_d T''_d(N_d)
    \end{bmatrix}
    }_{\text{Hessian of } X \text{ (Diagonal)}}
    =
    \underbrace{
    \begin{bmatrix}
    \alpha_1 \frac{\partial \lambda}{\partial N_1} & \cdots & \alpha_1 \frac{\partial \lambda}{\partial N_d} \\
    \vdots & \ddots & \vdots \\
    \alpha_d \frac{\partial \lambda}{\partial N_1} & \cdots & \alpha_d \frac{\partial \lambda}{\partial N_d}
    \end{bmatrix}
    }_{\text{Outer Product } \mathbf{b} (\nabla \lambda)^\top \text{ (Rank-1)}}
\end{equation}
Examining any off-diagonal entry ($i \neq j$), the LHS dictates it must be zero, establishing $\alpha_i \frac{\partial \lambda}{\partial N_j} = 0$. 
Because we assume a mixture where multiple weights are non-zero ($\alpha_i \neq 0$), it follows that the gradient of the scaling function is 
zero ($\nabla \lambda = \mathbf{0}$). Consequently, $\lambda(\mathbf{N})$ is a global scalar constant, $\lambda$.

Equating the diagonal elements implies $T'_i(N_i) = \lambda$. Integrating this indicates the optimal transport maps are linear functions of the form:
\begin{equation}
    T_i(x) = \lambda x + c.
\end{equation}
Applying a purely linear transformation to Gaussian noise $N_i$ produces another Gaussian distribution, so linear $T_i$ means $Z_i$ is Gaussian. Since $\lambda$ is the same constant for every $i$ with $\alpha_i\neq0$, and at least two such $\alpha_i$ are non-zero by assumption, equality would force at least two sources to be Gaussian, contradicting the hypothesis that at most one $Z_i$ is Gaussian.

\begin{equation}
    W_2^2(\mathbf{\alpha} \cdot \mathbf{Z}, \Gamma) < \sum_{i=1}^d \alpha_i^2 W_2^2(Z_i, \Gamma).
\end{equation}
\end{proof}

\subsection{Proof of Proposition~\ref{cor:margin}: Quantitative Margin}
We restate the result.
\setcounter{proposition}{2}
\begin{proposition}
Let $\mathbf{Z}$ satisfy the assumptions of Theorem~\ref{th:main} and let $\alphab\in\mathbb{R}^d$ satisfy $\sum_{i=1}^d\alpha_i^2=1$. If in addition each $Z_i$ is $\kappa$-strongly log-concave for some $\kappa>0$ (density $e^{-V}$ with $V''\ge\kappa$), then
\begin{equation}
\sum_{i=1}^d \alpha_i^2\, v_i - W_2^2(S,\Gamma) \;\ge\; \frac{\sqrt{\kappa}}{2}\Big(1-\|\alphab\|_4^4\Big)\sum_{i=1}^d \alpha_i^2\, v_i,
\qquad S:=\alphab^\top\mathbf{Z},\quad v_i := W_2^2(Z_i,\Gamma).
\end{equation}
\end{proposition}

\begin{proof}
\textbf{Step 1: Notation and Preliminaries.}
We first introduce (coupled) random variables.
Let $N_1,\dots,N_d$ be i.i.d.\ standard normal, and $h_i := F_i^{-1}\circ\Phi$, where $F_i$ is the c.d.f.\ of $Z_i$.
Then $h_i(N)\sim Z_i$ and since $h_i$ is increasing, $(h_i(N_i),N_i)$ is the comonotone (optimal) coupling of $\mathcal{L}(Z_i)$ with $\Gamma$, in particular $v_i=\mathbb{E}[(h_i(N_i)-N_i)^2]$.
Moreover, let $G:=\alphab^\top \mathbf{N}\sim\Gamma$
and $S=\sum_{i=1}^d \alpha_i h_i(N_i) $.
We denote 
\begin{align}
    \delta(\alphab) := \sum_i \alpha_i^2 v_i - W_2^2(S,\Gamma)
\end{align}

 for the deficit in Theorem~\ref{th:main}'s inequality. Recall from the proof of Theorem~\ref{th:main} that
\begin{align}
    \sum_i \alpha_i^2 v_i = \mathbb{E}[(S-G)^2].
\end{align}
We denote by $T_S$ the unique monotone map 
such that $T_S(S)\sim G$ and by $h_S$ its inverse which satisfies $h_S(G)\sim S$. 
In particular 
\begin{align}
    W_2^2(S,\Gamma)=\mathbb{E}[(G-h_S(G))^2].
\end{align}
Note that the previous three displays together imply that 
\begin{align}\label{eq:deficit_rewrite}
    \delta(\alphab)= \mathbb{E}[(S-G)^2]-\mathbb{E}[(G-h_S(G))^2]
    =-2\mathbb{E}(SG)+2\mathbb{E}(Gh_S(G))
    =2\mathbb{E}\left(G(h_S(G)-S)\right)
\end{align}
where we used that $\mathbb{E}(S^2)=\mathbb{E}(G^2)=1$.

\textbf{Step 2: Lower bounding the deficit.}

The goal of this step is to derive a general lower bound on the deficit. This requires some form of regularity of  $T_S$, namely
we consider  the expansivity of $T_S$ , i.e.,
\begin{align}
  C(S)=  \inf_{x\in \mathbb{R}} T'_S(x)\geq 0.
\end{align}
Note that if $C(S)>0$ then $h_S$ is Lipschitz with Lipschitz constant $C(S)^{-1}$.
We introduce $\psi(x):=\int_0^x T_S$ which is then a $C(S)$-strongly convex function. Therefore it satisfies for all $x,y\in\mathbb{R}$
\begin{align}
    \psi(x)\geq \psi(y)+T_S(y)(x-y)+\frac{C(S)}{2}(x-y)^2.
\end{align}
We apply this with $x=S$ and $y=h_S(G)$ and find
(moving terms to the other side)
\begin{align}
  G(h_S(G)-S)=  T_S(h_S(G))(h_S(G)-S)\geq
  \psi(h_S(G))-\psi(S)+\frac{C(S)}{2}(S-h_S(G))^2.
\end{align}
Now we take the expectation of this expression and we find using \eqref{eq:deficit_rewrite} 
and $h_S(G)\sim S$ that
\begin{align}
    \frac{\delta(\alphab)}{2}\geq
    \mathbb{E}(\psi(h_S(G))-\psi(S))
    +\frac{C(S)}{2}\mathbb{E}[(S-h_S(G))^2]
    \geq \frac{C(S)}{2}\mathbb{E}[(S-\mathbb{E}(S|G))^2].
\end{align}
Here we used in the last step that $h_S(G)$ is $G$-measurable and the conditional expectation is the $L^2$-projection of $G$ measurable random variables.

\textbf{Step 3: Hermite projection bound.} 
The goal of this step is to bound 
\begin{align}
    \mathbb{E}[(S-\mathbb{E}(S|G))^2]
    =\mathbb{E}[S^2]-\mathbb{E}[\mathbb{E}(S|G)^2]
\end{align}
in our setting where $S=\sum_i \alpha_i h_i(N_i)$
and $G=\sum \alpha_i N_i$.
Here the crucial observation is that $S=\sum_i \alpha_i h_i(N_i)$ and $(N_i,G)$ is jointly Gaussian which allows an explicit expansion of $h_i(N_i)$ using Hermite Polynomials. Let us recall a few facts about Hermite polynomials that are necessary for the expansion.
Let $(H_k)_{k\ge0}$ be the Hermite polynomials.
They are orthogonal in $L^2(\Gamma)$ with $\mathbb{E}[H_k(N)^2]=k!$. 
Finally, Mehler's formula states that for $(X,Y)$ jointly Gaussian random variables with covariance $\rho$ their joint density can be written as
\begin{align}
    p(x,y)=\varphi(x)\varphi(y)\sum_{k\geq 0}\frac{\rho^k}{k!}H_k(x)H_k(y).
\end{align}
Note that this implies that
\begin{align}\label{eq:Mehler}
    \mathbb{E}(H_k(X)|Y=y)=\int H_k(x) \, p(\mathrm{d} x|y)=\sum_{l\geq 0}\frac{\rho^l}{l!}H_l(y)\int H_l(x)H_k(x)\varphi(x)\,\mathrm{d}x
    =\rho^kH_k(y)
\end{align}
where we used the orthogonality relation in the last step.
Now we apply this to our setup. We  expand $h_i$ in $L^2$ as $h_i=\sum_{k\ge0}c_{i,k}H_k$.
Note that 
\begin{align}\label{eq:c0}
    c_{i,0}&=\mathbb{E}[Z_i]=0,\\
    \label{eq:c1}
    c_{i,1}&=\mathbb{E}[h_i(N_i)H_1(N_i)]
    =\mathbb{E}[h_i(N_i)N_i]=-\frac{1}{2}
    \mathbb{E}[(h_i(N_i)-N_i)^2]+1=1-\frac{v_i}{2}.
\end{align}
Since $(N_i,G)$ is jointly Gaussian with $\mathrm{Cov}(N_i,G)=\alpha_i$ ($G=\sum_i\alpha_iN_i$), \eqref{eq:Mehler} gives $\mathbb{E}[H_k(N_i)\mid G]=\alpha_i^kH_k(G)$, whence
linearity implies
\begin{align*}
S &= \sum_{k\geq 0,i} \alpha_i c_{i,k} H_k(N_i)\,,\\
\mathbb{E}[S\mid G] &= \sum_{k\geq 0,i} \alpha_i^{k+1} c_{i,k} H_k(G)\,.
\end{align*}
The orthogonality relations now imply
\begin{align}
\begin{split}\label{eq:ES2}
\mathbb{E}[S^2]
&= \sum_{k\geq 0, i} \alpha_i^2c_{i,k}^2k!
\\
\mathbb{E}[\mathbb{E}(S|G)^2]
&=\sum_{k,l\geq 0}
\left(\sum_i \alpha_i^{k+1}c_{i,k}\right)
\left(\sum_j\alpha_j^{l+1}c_{j,l}\right)  \mathbb{E}[H_k(G)H_l(G)]
=\sum_{k\geq 0}k! \left(\sum_i \alpha_i^{k+1}c_{i,k}\right)^2.
\end{split}
\end{align}
We thus find
\begin{equation}
\mathbb{E}[S^2]-\mathbb{E}[\mathbb{E}(S|G)^2]
=\sum_{k\ge1}k!\Big[\sum_i\alpha_i^2c_{i,k}^2-\Big(\sum_i\alpha_i^{k+1}c_{i,k}\Big)^2\Big].
\end{equation}
We bound the second term using Cauchy-Schwarz and find
\begin{align}
   \Big(\sum_i\alpha_i^{k+1}c_{i,k}\Big)^2
   \leq  \Big(\sum_i\alpha_i^{2}c_{i,k}^2\Big)
    \Big(\sum_i\alpha_i^{2k}\Big).
\end{align}
Thus  the $k=1$ term is non-negative (because $\sum_i\alpha_i^2= 1$) and we obtain
\begin{align}\label{eq:expansion_end}
\mathbb{E}[S^2]-\mathbb{E}[\mathbb{E}(S|G)^2]\geq \sum_{k\ge 2}k!\Big[\sum_i\alpha_i^2c_{i,k}^2
\big(1 - \sum_i \alpha_i^{2k}\big)\Big]
\geq (1-\lVert \alphab\rVert_4^4)\sum_{k\ge 2}k!\Big[\sum_i\alpha_i^2c_{i,k}^2\Big].
\end{align}
Combining \eqref{eq:c0} and  \eqref{eq:c1} with \eqref{eq:ES2} we conclude that
\begin{align}
    \sum_{k\ge 2}k!\Big[\sum_i\alpha_i^2c_{i,k}^2\Big]
    = 1 - \sum_i\alpha_i^2c_{i,1}^2
    = 1 - \sum_i \alpha_i^2 \left(1-\frac{v_i}{2}\right)^2
    = \sum_i \alpha_i^2 \left(v_i-\frac{v_i^2}{4}\right)
    \geq \frac{1}{2}\sum_i \alpha_i^2 v_i.
\end{align}
Here we used $ 0\leq v_i\leq 2$ to conclude $v_i^2/4\leq v_i/2 $.  The relation $v_i\leq 2$ holds because the independent coupling has distance 2.
We therefore infer from the last two displays
\begin{align}
    \mathbb{E}[S^2]-\mathbb{E}[\mathbb{E}(S|G)^2]
    \geq \frac{1}{2}(1-\lVert \alphab\rVert_4^4)\sum_i \alpha_i^2 v_i.
\end{align}
Combining this with the previous step we have shown the bound
\begin{align}
\delta(\alphab)  \;\ge\; \tfrac12 C(S)\big(1-\|\alphab\|_4^4\big)\sum_i\alpha_i^2 v_i.
\label{eq:margin_general}
\end{align}

\textbf{Step 4: Relating $C(S)$ and log-concavity.} 
The only remaining step is to control the constant $C(S)$. For log-concave densities this is feasible using 
Caffarelli's Contraction Theorem in
\citet{caffarelli2000monotonicity}, a modern reference is Theorem 1 in \citet{fathi2020aproof}.
This theorem
states that the optimal push-forward transport map from measure $\mu\propto e^{-V}$ to $\nu\propto e^{-W}$
is Lipschitz with Lipschitz constant 
 $\sqrt{\sup V'' / \inf W''} $
(provided $\inf W''> 0$), in particular the optimal transport map from a standard Gaussian to a $\kappa$-strongly log-concave function is $\sqrt{\kappa}^{-1}$-Lipschitz.

To apply this, we now show log-concavity of $S$.
Note that if 
$Z_i$ is $\kappa$-strongly log-concave, the scaled summand $\alpha_iZ_i$ is $\kappa/\alpha_i^2$-strongly log-concave.
Strong log-concavity is preserved under convolution, more precisely the convolution of $\kappa_i$-strongly log-concave functions is $\kappa$-strongly log-concave  where $1/\kappa_{\mathrm{sum}}=\sum_i 1/\kappa_i$. Since $\sum_i\alpha_i^2=1$, it follows that $S$ is again $\kappa$-strongly log-concave.
 We conclude that 
 the optimal transport map $h_S$ from $\Gamma$ to $S$  is $\kappa^{-1/2}$-Lipschitz, which implies 
 that $T_S=h_S^{-1}$ is $\sqrt{\kappa}$ expanding, in particular $C(S)\ge\sqrt{\kappa}$.
Thus we finally conclude by combining this with
\eqref{eq:margin_general}
\begin{equation}
\delta(\alphab) \;\ge\; \frac{\sqrt{\kappa}}{2}\big(1-\|\alphab\|_4^4\big)\sum_i\alpha_i^2v_i.
\end{equation}

\end{proof}

\subsection{Proof of Corollary~\ref{cor:symmetric}: Symmetric Joint Maximization}
We restate the result.
\setcounter{corollary}{1}
\begin{corollary}
    Under the assumptions of Theorem \ref{th:main}, for any orthogonal matrix $\mathbf{B} \in \mathrm{O}(d)$ with rows $\mathbf{b}_i^\top$, the total marginal contrast satisfies:
    \begin{equation}
        \sum_{i=1}^d W_2^2(\mathbf{b}_i^\top \mathbf{Z}, \Gamma) \le \sum_{j=1}^d W_2^2(Z_j, \Gamma),
    \end{equation}
    where equality holds if and only if $\mathbf{B}$ is a signed permutation matrix.
\end{corollary}

\begin{proof}
Let $\mathbf{B} \in \mathrm{O}(d)$ be an orthogonal matrix with rows $\mathbf{b}_i^\top$. Because $\mathbf{B}$ is orthogonal, its squared elements sum to $1$ across both rows ($\sum_{j=1}^d B_{ij}^2 = 1$) and columns ($\sum_{i=1}^d B_{ij}^2 = 1$).
By the intermediate step of Theorem~\ref{th:main}, the contrast for each row projection is bounded by:
\begin{equation}
    W_2^2(\mathbf{b}_i^\top \mathbf{Z}, \Gamma) \le \sum_{j=1}^d B_{ij}^2 W_2^2(Z_j, \Gamma).
\end{equation}
Summing this inequality over all $d$ rows yields an upper bound on the total marginal contrast:
\begin{align}\begin{split}
    \sum_{i=1}^d W_2^2(\mathbf{b}_i^\top \mathbf{Z}, \Gamma) &\le \sum_{i=1}^d \sum_{j=1}^d B_{ij}^2 W_2^2(Z_j, \Gamma) \\
    &= \sum_{j=1}^d \left( \sum_{i=1}^d B_{ij}^2 \right) W_2^2(Z_j, \Gamma) \\
    &= \sum_{j=1}^d W_2^2(Z_j, \Gamma).
    \end{split}
\end{align}
The maximum possible sum equals the unweighted sum of the pure sources' contrasts. Because Theorem~\ref{th:main} establishes that the inequality is strict whenever a row mixes multiple sources, the global equality holds if and only if every row $\mathbf{b}_i$ contains exactly one non-zero entry (which must be $\pm 1$ due to the unit norm constraint). Thus, the sum is maximized if and only if $\mathbf{B}$ is a signed permutation matrix.
\end{proof}

\subsection{Proof of Lemma~\ref{lem:wp_high}}
For convenience we restate the result here.
\setcounter{lemma}{0}
\begin{lemma}
There is a $p^\ast>2$ such that for  all $p>p^\ast$ there are i.i.d.\ random variables $Z_1$, $Z_2$
    such that
    \begin{align}
        W_p(\sqrt{2}^{-1}(Z_1+Z_2),\Gamma)>
        W_p(Z_1,\Gamma).
    \end{align}
\end{lemma}
\begin{proof}

 Recall that $\Gamma\sim \gamma$ denotes a standard normal random variable  where $\gamma$ denotes its law. We write $\varphi$ and $\Phi$ for the density and cdf of the standard normal and use the shorthand $\bar{\Phi}(t)=\mathbb{P}(\Gamma>t)=1-\Phi(t)=\Phi(-t)$.
 
    We now consider the monotone maps $T_{a,c}:\mathbb{R}\to\mathbb{R}$ for $a,c\geq 0$ given by
    \begin{align}
        T_{a,c}(x) = \begin{cases}
            0\quad |x|<a,\\
            x+c \quad x\geq a,\\
            x-c \quad x\leq -a.
        \end{cases}
    \end{align}
    First, we note that for all $a>0$ there is a unique $c>0$ such that $T_{a,c}(\Gamma)$ has unit variance which we will denote by $c(a)$. Indeed, note that
    $\mathbb{E}(T_{a,0}(\Gamma))^2<1$ for $a>0$
    and $\lim_{c\to\infty}\mathbb{E}(T_{a,c}(\Gamma))^2\to\infty$ and this is a continuous strictly increasing function.  
    Note that 
    \begin{align}
    \begin{split}
       \mathbb{E}\left(T_{a,c}(\Gamma)^2 \right)
      & =\mathbb{E} (\boldsymbol{1}_{|\Gamma|\geq a} 
       \Gamma^2)+2\mathbb{E} (\boldsymbol{1}_{|\Gamma|\geq a} 
       |\Gamma|)c+c^2\mathbb{E} (\boldsymbol{1}_{|\Gamma|\geq a} 
       )
       \\
       &=2\Phi(-a)+2a\varphi(a)+4c\varphi(a)+2c^2\bar{\Phi}(a).
       \end{split}
    \end{align}
    Since $\varphi$ and $\Phi$ decay exponentially we find that $c(a)>a$ for $a \geq a_0$ for a sufficiently large $a_0$.
    For such $a$ we consider $Z^a=T_{a,c(a)}(\Gamma)$ and conclude that
    \begin{align}
        W_\infty(T_{a,c(a)}(\Gamma),\Gamma)=c.
    \end{align}
    Consider independent copies $\Gamma, \Gamma'$
    and set 
    \begin{align}
     X= \frac{1}{\sqrt{2}}\left(  T_{a,c}(\Gamma)+T_{a,c}(\Gamma')\right).
    \end{align}
    Then we find for $t$ such that $t/\sqrt{2}-c>a$
    \begin{align}
        \mathbb{P}(X>t)
        \geq \mathbb{P}\left(\Gamma>\frac{t}{\sqrt{2}}-c\right)^2
    \end{align}
    because $\{X>t\}$ is contained in $\{\Gamma,\Gamma'>t/\sqrt{2}-c>a\}$
    (which implies $T_{a,c}(\Gamma)=\Gamma+c$).
    We find 
    \begin{align}
        \mathbb{P}\left(\Gamma>\frac{t}{\sqrt{2}}-c\right)^2=\bar{\Phi}\big(\sqrt{2}^{-1}t-c\big)^2
        \sim \frac{\exp\left(-2\frac{(\sqrt{2}^{-1}t-c)^2}{2}\right)}{2\pi(\sqrt{2}^{-1}t-c)^2}
        =\frac{\exp\left(-\frac{t^2}{2}+\sqrt{2}ct-c^2\right)}{2\pi(\sqrt{2}^{-1}t-c)^2}
    \end{align}
    as $t\to\infty$ where $\sim$ indicates that the quotient converges to 1.
    Similarly, find for $\alpha>1$ that
    \begin{align}
        \mathbb{P}(\Gamma>t-\alpha c)
        = \bar{\Phi}(t-\alpha c)
        \sim \frac{\exp\left(-\frac{t^2}{2}+\alpha ct-\frac{\alpha^2c^2}{2}\right)}{\sqrt{2\pi}(t-\alpha c)}.
    \end{align}
    Therefore, we find that
    \begin{align}
    \begin{split}
       \lim_{t\to\infty}
       \frac{ \mathbb{P}\left(X>t\right)}{ \mathbb{P}(\Gamma>t-\alpha c)}
       &\geq 
        \lim_{t\to\infty}
       \frac{ \mathbb{P}\left(\Gamma>\frac{t}{\sqrt{2}}-c\right)^2}{ \mathbb{P}(\Gamma>t-\alpha c)}
      \\  &=\lim_{t\to\infty}\exp\left(-\frac{t^2}{2}+\sqrt{2}ct-c^2
        +\frac{t^2}{2}-\alpha ct+\frac{\alpha^2c^2}{2}\right)
        \frac{\sqrt{2\pi}(t-\alpha c)}{2\pi(\sqrt{2}^{-1}t-c)^2}
        \\
        &=\lim_{t\to\infty}\exp\left((\sqrt{2}-\alpha)ct-c^2
       +\frac{\alpha^2c^2}{2}\right)
        \frac{t-\alpha c}{\sqrt{2\pi}(\sqrt{2}^{-1}t-c)^2}
        =\infty
        \end{split}
    \end{align}
    for $\alpha<\sqrt{2}$.
    This implies that any coupling of $X$ and $\Gamma$
    must be supported on $(t_0,\infty)\times (-\infty,t_0-\alpha c)$ for $t_0$ sufficiently large and thus
    \begin{align}
        W_\infty(X,\Gamma) \geq \sqrt{2}c.
    \end{align}
   We have shown for $Z=T_{a,c(a)}(\Gamma)$ and $a\geq a_0$ that
    \begin{align}
        W_\infty(\sqrt{2}^{-1}(Z+Z'),\Gamma)
        \geq \sqrt{2}c>c=W_\infty(Z).
    \end{align}
    Since $W_p\to W_{\infty}$ as $p\to \infty$
    for fixed distributions the claim holds for $p\geq p^\ast$ and some $p^\ast$.
    Numerical experiments verify that  such counterexamples exist down to $p\approx 12$
    while $p$ close to 2 remains open.
    \end{proof}
\subsection{Proof of Lemma~\ref{lem:wp_low}}
Again, we first recall the statement of the lemma.

\begin{lemma}
    For every $1\leq p<2$  there are i.i.d.\ random variables $Z_1$, $Z_2$
    such that
    \begin{align}
        W_p(\sqrt{2}^{-1}(Z_1+Z_2),\Gamma)>
        W_p(Z_1,\Gamma).
    \end{align}
\end{lemma}
\begin{proof}
 Recall again that $\Gamma\sim \gamma$ denotes a standard normal random variable  where $\gamma$ denotes its law. We write $\varphi$ and $\Phi$ for the density and cdf of a standard normal variable and use the shorthand $\bar{\Phi}(t)=\mathbb{P}(\Gamma>t)=1-\Phi(t)=\Phi(-t)$.
Jensen's inequality implies that $\nu_p=(\mathbb{E}|\Gamma|^p)^{1/p}\le(\mathbb E \Gamma^2)^{1/2}=1$ for $p\le 2$, and $\nu_1=\sqrt{2/\pi}$.

We now construct a suitable $Z$.
Fix $R>0$ and $\varepsilon\in(0,1)$ with $\varepsilon R^2<1$, and set
\[
\sigma:=\sqrt{1-\varepsilon R^2}\in(0,1).
\]
Let $\Gamma\sim \gamma$ and, independently, let $J$ satisfy
\[
\mathbb P(J=0)=1-\varepsilon,\qquad \mathbb P(J=R)=\mathbb P(J=-R)=\tfrac{\varepsilon}{2},
\]
and define
\[
Z:=\sigma \Gamma+J .
\]
Then $\mathbb E Z=0$ and $\mathbb{E}(Z^2)=\sigma^2+\varepsilon R^2=1$, so $Z$ is standardized.
Let $Z'$ be an independent copy and 
\begin{align}
S:=\tfrac1{\sqrt2}(Z+Z').
\end{align}

Since $\tfrac{1}{\sqrt2}(\Gamma_1+\Gamma_2)\sim\gamma$ and is independent of $(J_1,J_2)$,
\begin{equation}\label{eq:S}
S\ \overset{d}{=}\ \sigma \Gamma+\widetilde J,\qquad \widetilde J:=\tfrac{1}{\sqrt2}(J_1+J_2),
\end{equation}
where $J_1,J_2$ are i.i.d.\ copies of $J$ independent of $(\Gamma_1,\Gamma_2)$, and
\[
\mathbb P\big(\widetilde J=\pm\tfrac{R}{\sqrt2}\big)=\varepsilon(1-\varepsilon),
\qquad
\mathbb P\big(\widetilde J=\pm\sqrt2 R\big)=\tfrac{\varepsilon^2}{4},
\]
the remaining mass being at $0$.

\textbf{Heuristic scaling:}
Let us gain some intuition by providing a heuristic of the scaling of the transport costs.
$Z$ is a standard Gaussian up to a small amount of mass $\varepsilon$ at distance $R$, whose transport cost
is $\asymp\varepsilon R^p$. In $S$ the lump occurs twice as often but at distance $R/\sqrt2$, so its cost is
$\asymp 2\varepsilon(R/\sqrt2)^p=2^{1-p/2}\varepsilon R^p$. The gain factor $2^{1-p/2}$ exceeds $1$ when $p<2$.
The variance constraint forces the Gaussian part to be shrunk by $1-\sigma\asymp\varepsilon R^2/2$, at a cost
$\asymp(\varepsilon R^2)^p$; this is negligible when $p>1$ and $\varepsilon$ is very small compared to $L$ (see below for precise statement). The case  $p=1$ will be studied at the end.
We now show two statements to implement this heuristic. The first upper bounds the transport cost from $Z$ by a simple statement while the second one provides a lower bound on the Wasserstein distance based on non-matching tail behaviour of the two random variables.

\textbf{Claim 1:}
For all $p\ge1$,
\begin{align}\label{eq:WPupper}
W_p^p(Z,\Gamma)\ \le\ (1-\varepsilon)(1-\sigma)^p\nu_p+\varepsilon\big(R+(1-\sigma)\nu_p^{1/p}\big)^p .
\end{align}

Let us prove this claim.
If $\mu=\sum_i\alpha_i\mu_i$ and $\nu=\sum_i\alpha_i\nu_i$ are finite mixtures, 
adding optimal couplings
of $(\mu_i,\nu_i)$ with weights $\alpha_i$ produces a coupling of $(\mu,\nu)$ (note that here we have a mixture of measures and not a sum of independent variables as in
Theorem~\ref{th:main}),
whence
\begin{align}
W_p^p(\mu,\nu)\le\sum_i\alpha_iW_p^p(\mu_i,\nu_i).
\end{align}
Apply this with
$\mu_i=\mathrm{Law}(\sigma \Gamma+\omega)$, $\omega\in\{0,R,-R\}$, weights $(1-\varepsilon,\tfrac\varepsilon2,\tfrac\varepsilon2)$,
and $\nu_i=\gamma$ for every $i$. For the centered component, the coupling $(\sigma \Gamma,\Gamma)$ gives
$W_p(\mathcal \sigma \Gamma,\Gamma)\le\|\sigma \Gamma-\Gamma\|_p=(1-\sigma)\nu_p^{1/p}$.
For the shifted components, the coupling $(\pm R+\sigma \Gamma,\Gamma)$ and the triangle inequality give
\[
W_p(\mathcal \pm R+\sigma\Gamma,\gamma)\le\|\pm R-(1-\sigma)\Gamma\|_p\le R+(1-\sigma)\nu_p^{1/p}.
\]
Adding the three terms yields the claim.

\textbf{Claim 2:}
Let $s>0$ and $A>0$ be such that $B:=\tfrac{R}{\sqrt2}-s>A$. Then
\begin{align}\label{eq:WPlower}
W_p^p(S,\gamma)\ \ge\ (B-A)^p\Big[\,2\varepsilon(1-\varepsilon)\,\Phi(s/\sigma)-2\bar\Phi(A)\,\Big].
\end{align}

Let us now show this claim. The key intuition is that a certain amount of mass has to be transported from at least $B$ to below $A$ due to different tail mass. Consider any coupling $(X,Y)$  with $X\sim S$, $Y\sim \Gamma$. If $|X|\ge B$ and $|Y|\le A$ then
$|X-Y|\ge|X|-|Y|\ge B-A$; hence
\[
\mathbb P(|X-Y|\ge B-A)\ \ge\ \mathbb P(|X|\ge B)-\mathbb P(|Y|>A),
\]
and by Markov's inequality $\mathbb E|X-Y|^p\ge(B-A)^p\,\mathbb P(|X-Y|\ge B-A)$.
Finally, by \eqref{eq:S}, keeping only the two atoms $\pm R/\sqrt2$ of $\widetilde J$,
\[
\mathbb P(|S|\ge B)\ \ge\ 2\cdot\varepsilon(1-\varepsilon)\,\mathbb P\big(\sigma G\ge -s\big)
=2\varepsilon(1-\varepsilon)\Phi(s/\sigma),
\]
while $\mathbb P(|\Gamma|>A)=2\bar\Phi(A)$.

\textbf{Combining lower and upper bound} 
We now combine the previous two claims.
Define $\theta:=2^{1-p/2}>1$. Given $\varepsilon\in(0,e^{-1})$, set
\[
R:=\ln\tfrac1\varepsilon, \qquad s:=A:=2\sqrt R ,
\]
and assume $\varepsilon$ is small enough that $\varepsilon R^2<1$ (so $\sigma$ is well defined) and
$B=\tfrac R{\sqrt2}-s>A$. 
The logarithmic scaling of $R$ with $\varepsilon$
ensures that the transport of the point masses is dominating.
The purpose of $s$ is that it is negligible compared to $R$ but increasing so that Gaussian tail bounds become exponentially small in $s$.

\emph{Lower bound.} Since $\bar\Phi(x)\leq\frac{1}{\sqrt{2\pi} x}\le\tfrac12e^{-x^2/2}$ for $x\ge1$ we find for $\sigma\le1$,
\[
\Phi(s/\sigma)\ge\Phi(2\sqrt R)\ge1-\tfrac12e^{-2R}=1-\tfrac{\varepsilon^2}{2},
\qquad 2\bar\Phi(A)\le e^{-2R}=\varepsilon^2 .
\]
Hence the bracket in \eqref{eq:WPlower} is at least
$2\varepsilon\big[(1-\varepsilon)(1-\tfrac{\varepsilon^2}2)-\tfrac{\varepsilon}{2}\big]\ge2\varepsilon(1-2\varepsilon)$, and
\begin{align}
    B-A=\frac{R}{\sqrt2}-4\sqrt{R}=\tfrac{R}{\sqrt2}\big(1-\tfrac{4\sqrt2}{R^{1/2}}\big),
\end{align}
so
\begin{equation}\label{eq:lowfinal}
W_p^p(S,\Gamma)\ \ge\ \, \left(\frac{R}{\sqrt{2}}\right)^p\Big(1-\tfrac{4\sqrt2}{R^{1/2}}\Big)^{p}2\varepsilon(1-2\varepsilon)
\ = \theta \varepsilon R^p \Big(1-\tfrac{4\sqrt2}{R^{1/2}}\Big)^{p}=:\ \theta\,\varepsilon R^p\,K(\varepsilon).
\end{equation}

\emph{Upper bound.} Using $\nu_p\le1$ and
$1-\sigma=\frac{\varepsilon R^2}{1+\sigma}\le\varepsilon R^2$, the bound \eqref{eq:WPupper} gives
\begin{equation}\label{eq:upfinal}
W_p^p(Z,\Gamma)\ \le\ (\varepsilon R^2)^p+\varepsilon\big(R+\varepsilon R^2\big)^p
=\varepsilon R^p\Big[{\varepsilon^{p-1}R^{p}+\big(1+\varepsilon R\big)^p}\Big]=:\varepsilon R^p M(\varepsilon).
\end{equation}

As $\varepsilon\to 0 $ we have $R\to\infty$, $\varepsilon R\to0$ and
$\varepsilon^{p-1}R^{p}=e^{-(p-1)R}R^{p}\to0$ for $p>1$ as $R\to \infty$, so
\[
K(\varepsilon)\longrightarrow1,\qquad M(\varepsilon)\longrightarrow1 .
\]
Since $\theta>1$, there is $\varepsilon_p>0$ such that $\theta K(\varepsilon)>M(\varepsilon)$ for all
$\varepsilon\in(0,\varepsilon_p)$. For any such $\varepsilon$, \eqref{eq:lowfinal} and \eqref{eq:upfinal} give
\[
W_p^p\!\left(\tfrac{Z+Z'}{\sqrt2},\Gamma\right)\ \ge\ \theta K(\varepsilon)\,\varepsilon R^p
\ >\ M(\varepsilon)\,\varepsilon R^p\ \ge\ W_p^p(Z,\Gamma). 
\]
This ends the proof for $1<p<2$.
For $p=1$ these bounds are not sufficient because then 
the transport of the atoms is not the dominant term. A more careful analysis reveals that also in this case 
$W_p^p\!\left(\tfrac{Z+Z'}{\sqrt2},\Gamma\right)\ > W_p^p(Z,\Gamma)$ holds. However, it is easier to instead consider $Z$ such that $\mathbb{P}(Z=\pm R)=R^{-2}/2$
and $\mathbb{P}(Z=0)=1-R^{-2}$. Note that $Z$ is standardized.
Explicit calculation reveals
\begin{align}
    W_1(Z,\Gamma)=\frac{2}{\pi}+R^{-1}+o(R^{-1})\text{ as $R\to\infty$}.
\end{align}
Indeed, note that the first term is the transport cost from a centred Dirac mass to $\Gamma$ and the second term is the transport cost of the atoms to the center given by $\mathbb{P}(|Z|=R)R=R^{-2}R=R^{-1}$.
However, for $S=\sqrt{2}^{-1}(Z+Z')$ where $Z'$ is an independent copy
\begin{align}
    W_1(S, \Gamma)=\frac{2}{\pi}+\sqrt{2}R^{-1}+o(R^{-1})\text{ as $R\to\infty$}
\end{align}
because now the (4 atoms have total mass $2R^{-1}$ 
at distance $R/\sqrt{2}$).

\end{proof}

\section{Asymptotic Variance of the OT-ICA Estimator}
\label{app:asymptotic_variance}

Theorem~\ref{th:main} is a population statement: with infinitely many samples, the true source directions are the unique maximizers of the $W_2^2$ contrast. This appendix instead studies the estimator's behavior at finite sample size $N$, in the style of the classical asymptotic-variance analysis of FastICA by \citet{tichavsky2006performance}. 
The main result can be found in Theorem~\ref{th:asymptotic} and establishes that there are local maximizers of the empirical Wasserstein contrast close to the true source $Z_\ell$ under mild regularity assumptions. Moreover, their limiting distribution is Gaussian with an explicit variance depending only on the distribution of $Z_\ell$.
It might be feasible but technically demanding to extend the results to estimators that estimate all components in parallel. The analysis in this section is restricted to sources with a density. In Appendix~\ref{app:atoms} we show that for discrete distributions we can achieve faster rates, where we asymptotically extract the independent component exactly.

The analysis is structured in several intermediate results and we will here summarize the main steps. We first introduce some notation and show that the Wasserstein contrast can be reduced to  a linear rank statistic 
with a quadratic scaling factor in Section~\ref{app:rank}. An important observation is that similar rank statistics were studied carefully in the context of linear regression which allows us to essentially
extract their asymptotic behaviour from the literature in Section~\ref{app:quad}. Then, in Section~\ref{app:quad2} 
we perform a similar but simpler analysis of the additional quadratic scaling factor which allows us to find a random local quadratic approximation of the empirical Wasserstein contrast around the true source and we characterize its distribution.
Finally, we can then exploit convexity properties of the rank statistic to relate local maximizers of the Wasserstein contrast (local minimizers of $\Psi_N$) to (sub)-gradients (see Section~\ref{app:convexity}) and then the main result 
in Section~\ref{app:main} is relatively straightforward.

\subsection{Notation and reduction to rank statistic}
\label{app:rank}
We first introduce relevant notation and the setup.
Our result concerns the estimation of a single unit at a time: We show that for each unit $\mathbf{Z}_\ell$ with 
high probability there is a local maximum of the empirical Wasserstein contrast $S^{d-1}\ni \mathbf{b}\to \widehat{W}_2^2(\mathbf{b}\cdot \tilde{\mathbf{X}}, \Gamma)$ such that
$\mathbf{b}\cdot \tilde{\mathbf{X}}\approx \mathbf{Z}_\ell$. Since we consider the noiseless linear setting we can write $\mathbf{b}\cdot \tilde{\mathbf{X}}$ as a linear combination $\mathbf{\alpha}\mathbf{Z}$ and it is therefore sufficient to analyze the empirical Wasserstein contrast as a function of $\alphab\to \widehat{W}_2^2(\alphab\cdot {\mathbf{Z}}, \Gamma)$ where $\alphab$ is restricted to the set $\alphab^\top \widehat{\Sigma}\alphab=1$.     

Since we consider the Wasserstein contrast around $\alphab=e_\ell$ it is natural to write the extracted source (up to scale) as 
\begin{align}
    Z_\ell + \sum_{k\neq \ell} \sigma_k Z_k
    = u(\sigma)^\top \mathbf{Z}
\end{align}
where 
\begin{align}\label{eq:defusigma}
    u(\sigma)= e_\ell + \sum_{k\neq \ell} \sigma_k e_k.
\end{align}
Note that this representation sweeps, up to positive scaling, all directions with a
positive loading of $Z_\ell$. 
We remark in passing that a similar analysis is possible for the analysis around $-Z_\ell$.

We consider from now on $N$ i.i.d.\ samples $\mathbf{Z}^{1},\dots,\mathbf{Z}^{N}$ of $\mathbf{Z}$ and denote by $\hat{\mathbf{Z}}$ a random variable following the empirical distribution. 
We denote the centred samples by $\mathbf{Z}^{i}_c$
and remark that they are no longer independent.
We introduce the empirical normalization 
\begin{align}\label{eq:defqsigma}
    q(\sigma):=\frac{1}{N}\sum_i(u(\sigma)^\top\mathbf{Z}^{i}_c)^2
\end{align}
so that $q(\sigma)^{-1/2}(u(\sigma)^\top \hat{\mathbf{Z}})$
has unit variance. 
The first step in the analysis is to rewrite the Wasserstein contrast where we replace 
$\Gamma$ by the empirical distribution 
$\hat{\Gamma}_N=\frac1N\sum_{i=1}^N \delta_{z_i}$ where 
\begin{align}
z_i:=%
N\int_{(i-1)/N}^{i/N} \Phi^{-1}(u)\, \mathrm{d} u.
\end{align}

Note that for any empirical measure $\hat\mu$ on $N$ points,
\begin{align}
W_2^2(\hat\mu,\Gamma)=W_2^2(\hat\mu,\hat\Gamma_N)+W_2^2(\hat\Gamma_N,\Gamma),
\qquad W_2^2(\hat\Gamma_N,\Gamma)=1-\frac1N\sum_{i=1}^N z_i^2 .
\end{align}
In one dimension $W_2^2(\mu,\nu)=\|Q_\mu-Q_\nu\|_{L^2(0,1)}^2$. The quantile function of
$\hat\mu$ is constant on each bin $I_i=((i-1)/N,i/N]$, and that of $\hat\Gamma_N$ equals
$z_i$ on $I_i$, i.e.\ it is the $L^2(0,1)$-orthogonal projection of $\Phi^{-1}$ onto
functions constant on the bins. Hence $Q_{\hat\mu}-Q_{\hat\Gamma_N}$ lies in that
subspace while $Q_{\hat\Gamma_N}-\Phi^{-1}$ is orthogonal to it, since
$\int_{I_i}(z_i-\Phi^{-1})=0$, and Pythagoras gives the first identity; the second
follows from $\int_0^1(\Phi^{-1})^2=1$ and $\int_{I_i}\Phi^{-1}=z_i/N$.

The crucial observation summarized in the following lemma is that we can actually equivalently maximize a rank based statistic of the data.

\begin{lemma}\label{lem:reduction_convexity}
Given $N$ i.i.d.\ samples $\mathbf{Z}^{1},\dots,\mathbf{Z}^{N}$ of $\mathbf{Z}$. Let $\cdot^{(i)}$ denote the order statistics of a list of real numbers (i.e., the reordered numbers form a non-decreasing sequence).
 Define
\begin{equation}
\Psi_N(\sigma):=q(\sigma)^{-1/2}\Lambda_N(\sigma),\qquad \Lambda_N(\sigma):=\frac1N\sum_{i=1}^N \{u(\sigma)^\top\mathbf{Z}\}^{(i)}z_i.
\end{equation}
Then 
\begin{align}\label{eq:w2tolinear}
    W_2^2(q(\sigma)^{-1/2}u(\sigma)\cdot \mathbf{\hat{Z}}_c, \hat{\Gamma}_N)
    =-2\Psi_N(\sigma) +\kappa_N
\end{align}
for some constant $\kappa_N$
in particular  their local extrema agree.
Moreover,
\begin{align}\label{eq:lambda+permutation}
\Lambda_N(\sigma)=\max_{\pi\in S_N}\frac1N\sum_i u(\sigma)^\top\mathbf{Z}^i z_{\pi(i)},
\end{align}
so $\Lambda_N$ is convex and piecewise linear in $\sigma$.
\end{lemma}

\begin{proof}
The optimal coupling between two  distributions
with $N$ atoms with equal weights is the unique order preserving matching. This implies that
\begin{align}
\begin{split}
      W_2^2(q(\sigma)^{-1/2}u(\sigma)\cdot \mathbf{\hat{Z}}_c, \hat{\Gamma}_N)
      &= \frac{1}{N}\sum_{i=1}^N \left(q(\sigma)^{-1/2}\{u(\sigma)^\top\mathbf{Z}_c\}^{(i)}-z_i\right)^2
      \\
      &=\frac{1}{N}q(\sigma)^{-1}\sum_{i=1}^N \left(u(\sigma)^\top\mathbf{Z}^{(i)}_c\right)^2
      + \frac{1}{N}\sum_{i=1}^N z_i^2
      -2 \frac{1}{N}q(\sigma)^{-1/2}\sum_{i=1}^N 
       \{u(\sigma)^\top\mathbf{Z}_c\}^{(i)}z_i
      \\
      &=1+\frac{1}{N}\sum_{i=1}^N z_i^2-2\Psi_N(\sigma).
      \end{split}
\end{align}
Here we used $\sum_i z_i=0$ in the last step to move back from centred to uncentred variables.
Therefore we established \eqref{eq:w2tolinear} and
\eqref{eq:lambda+permutation} follows from the rearrangement inequality. So $\Lambda_N$ is 
the maximum of affine (i.e., convex) functions and therefore itself convex. 

\end{proof}

\subsection{Quadratic approximation of $\Lambda_N$}
\label{app:quad}
The next step is to now observe that the function $\Lambda_N$  was previously investigated in the context of linear regression. Indeed, in settings where the regression error has heavy tails least squares approaches can lack robustness and one can instead consider rank based R-statistics. For example, for a homoscedastic noise model the rank of the residual and the covariates are  independent. 
A similar approach is to minimize the statistic
\begin{align}
    \sum_i R^{(i)} z_i
\end{align}
where $R_{(i)}$ are the ordered residuals and $z_i$ is an increasing sequence as above. It was shown that this leads to asymptotically normal estimators in \citet{jana,Jana2,Jaeckel,Heiler}. 
A careful look at our setup reveals that we are studying essentially the same object and we can therefore rely on their results. However, we cannot directly use the asymptotic normality because we need to take the normalization factor $q(\sigma)^{-1/2}$ into account. Nevertheless, we can rely on the intermediate result which ensures that $\Lambda_N$ can be well approximated locally by a quadratic function. 

Let us now connect to the regression setting. Our variable $Z_\ell$ corresponds to the target and 
$Z_{k\neq \ell}$ to the covariates which are combined through a linear model $\sum_{k\neq \ell} \sigma_k Z_k$ (for the actual linear model there would be $-\beta$). Clearly the true regression coefficient is $\sigma=0$ but here we consider the finite sample setting. 
We find that the following approximation holds for
$\Lambda_N$.

\begin{lemma}\label{lem:quad}
    Assume that $Z_\ell$ has a differentiable density $f_\ell$ with finite Fisher information
    \begin{align}
        I(f_\ell)=\int \left(\frac{f_\ell'}{f_\ell}\right)^2f_\ell \,\mathrm{d}x<\infty.
    \end{align}
    Let
    \begin{align}      
        B_\ell :=- \int_0^1\frac{f'_\ell(F_\ell^{-1}(u))}{f_\ell(F_\ell^{-1}(u))} \Phi^{-1}(u)\,du
        \leq \sqrt{I(f_\ell)}<\infty.
    \end{align}
    Then for every $C>0$ and $\varepsilon>0$
    \begin{align}
        \lim_{N\to \infty}\mathbb{P}\left(
            \sup_{|\sigma|\leq C/\sqrt{N}} \left|\Lambda_N(\sigma) - \Lambda_N(0)-\frac{B_\ell|\sigma|^2}{2}-\sigma\cdot \nabla \Lambda_N(0) \right|\geq \frac{\varepsilon}{N}
        \right)=0.
    \end{align} 
\end{lemma}
    Essentially this is a result from \citet{Heiler,Jaeckel,Jana2}.
    However, we have to translate the notation and there is the small caveat that their result does not directly apply to our sequence $z_i$.
    Indeed, to apply previous work we need to consider $\Lambda_N'$ where we replace $z_i$ by 
    \begin{align}
        z_i'=\Phi^{-1}\left(\frac{i}{(N+1)}\right).
    \end{align}
    We first show that this has an effect of lower order.  We denote by $b_N$ the difference vector $b_N^i = z_i-z_i'$ and 
    we write $\mathbf{Z}_{-\ell}$ for the vector $\mathbf{Z}$ where the $\ell$-th entry $Z_\ell$ is omitted.
\begin{lemma}\label{le:lambda-lambdap}
    Assume that $Z_\ell^i$ are pairwise different (which is true almost surely under our assumptions).
    Then $\Lambda_N$ and $\Lambda_N'$ are differentiable at 0 and for every $N\ge4$ 
\begin{align}
  \big|\nabla\Lambda_N'(0)-\nabla\Lambda_N(0)\big|
    &\le \frac{|b_N|_2}{\sqrt{N}}\,\left(\sum_{i} N^{-1}|\mathbf{Z}_{-\ell,c}^i|^2\right)^{1/2} , \label{eq:grad}\\
  \Big|\big[\Lambda_N'(\sigma)-\Lambda_N'(0)\big]-\big[\Lambda_N(\sigma)-\Lambda_N(0)\big]\Big|
    &\le \frac{|\sigma|}{\sqrt{N}}\,|b_N|_2\,
    \left(\sum_{i} N^{-1}|\mathbf{Z}_{-\ell,c}^i|^2\right)^{1/2}. \label{eq:fun}
\end{align}
Moreover, $\lim_{N\to \infty} |b_N|=0$.
\end{lemma}
The proof of this lemma is postponed to Appendix~\ref{app:lambdalambdap}.
    
\begin{proof}[Proof of Lemma~\ref{lem:quad}]
This result (for $\Lambda_N'$) can essentially be found in \cite{Jana2}.
The later paper \cite{Heiler} realized that since $\Lambda_N'$ is convex pointwise convergence is actually sufficient (see \cite{hjortpollard1993asymptotics} for general statements) and we therefore verify the slightly weaker assumptions of \cite{Heiler}.
Note that all the result there are for a fixed design of the $\mathbf{Z}_{-\ell}$ but we can condition on this random variable and then we show that the statement holds conditional on $\mathbf{Z}_{-\ell}$  for almost all realizations of $\mathbf{Z}_{-\ell}$ which implies the main claim by dominated convergence.
Their errors $e_i$ in correspond in our notation 
to $Z_\ell^i$  which are i.i.d.\ 
and we assume that the density of $Z_\ell$ has finite Fisher information (so (I) and (II) hold). The sequence $z_i'=\Phi^{-1}(i/(N+1))$ is generated by the non-constant, non-decreasing, and square-integrable
so (III) and (IV) hold. Finally, we need to check that
\begin{align}
    \lim_{N\to \infty} \max_{1\leq i\leq N}
    \frac{1}{N}(\mathbf{Z}_{-\ell,c}^i)^\top (\hat{\mathbf{\Sigma}}_{-\ell})^{-1}\mathbf{Z}_{-\ell,c}^i=0
\end{align}
where $\hat{\mathbf{\Sigma}}$ denotes the empirical covariance. Note that in our setting $\hat{\mathbf{\Sigma}}_{-\ell}^{-1}\to \mathrm{Id}_{d-1}$ holds for almost all $\mathbf{Z}_{-\ell}$ because $Z_k$ are independent and of unit variance. 
By the triangle inequality it is then sufficient to show
that for almost all $\mathbf{Z}_{-\ell}$ 
\begin{align}\label{eq:maxbound}
    \max_{1\leq i\leq N}
    N^{-1}|\mathbf{Z}_{-\ell}^i|^2=0.
\end{align}
Note that 
\begin{align}
    \sum_{i\geq 1} \boldsymbol{1}_{X\geq i\varepsilon}\leq X/\varepsilon
\end{align}
implies
\begin{align}
    \sum_{i\geq 1} \mathbb{P}(|\mathbf{Z}_{-\ell}^i|^2\geq i\varepsilon)\leq \frac{\mathbb{E}(|\mathbf{Z}_{-\ell}^i|^2)}{\varepsilon}
    =\frac{d}{\varepsilon}<\infty.
\end{align}
Therefore the Borel-Cantelli Lemma implies that for every $\varepsilon>0$ the bound 
\begin{align}
    |\mathbf{Z}^n_{-\ell}|^2\leq \varepsilon n
\end{align}
holds for almost all $n$ and thus \eqref{eq:maxbound} holds.
Therefore the results from \cite{Heiler} hold for almost every $\mathbf{Z}_{-\ell}$. We use the equivalent parametrisation from \cite{Jaeckel} (avoiding the empirical covariance rescaling) and their Lemma 1 then states that for all $\varepsilon>0$ and $C>0$
  \begin{align}\label{eq:boundlampdap}
        \lim_{N\to \infty}\mathbb{P}\left(
            \sup_{|\sigma|\leq C/\sqrt{N}} \left|\Lambda_N'(\sigma) - \Lambda_N'(0)-\frac{B_\ell|\sigma|^2}{2}-\sigma\cdot \nabla \Lambda_N'(0) \right|\geq \frac{\varepsilon}{N}
        \right)=0.
    \end{align} 
Note that their paper consider $\Delta=N^{1/2}\beta$
which would correspond to $\sqrt{N}\sigma$ in our notation explaining the radius $C/\sqrt{N}$ instead of $C$ and we have an additional $N^{-1}$ normalization in $\Lambda$ explaining the $\varepsilon/N$ scaling.
Now we bound using the triangle inequality 
\begin{align}
    \begin{split}
        \left|\Lambda_N(\sigma) - \Lambda_N(0)-\frac{B_\ell|\sigma|^2}{2}-\sigma\cdot \nabla \Lambda_N(0) \right|
      &  \leq 
          \left|\Lambda'_N(\sigma) - \Lambda'_N(0)-\frac{B_\ell|\sigma|^2}{2}-\sigma\cdot \nabla \Lambda'_N(0) \right|
         \\
         &\qquad+|\sigma|\big|\nabla\Lambda_N'(0)-\nabla\Lambda_N(0)\big|
          +\Big|\big[\Lambda_N'(\sigma)-\Lambda_N'(0)\big]-\big[\Lambda_N(\sigma)-\Lambda_N(0)\big]\Big|
    \end{split}
\end{align}
Then  Lemma~\ref{le:lambda-lambdap} implies that almost surely
\begin{align}
  \sup_{|\sigma|\leq C/\sqrt{N}}  |\sigma|\big|\nabla\Lambda_N'(0)-\nabla\Lambda_N(0)\big|
          +\Big|\big[\Lambda_N'(\sigma)-\Lambda_N'(0)\big]-\big[\Lambda_N(\sigma)-\Lambda_N(0)\big]\Big|
          \leq \frac{2C|b_N|_2}{N}\,\sum_{i} N^{-1}|\mathbf{Z}_{-\ell,c}^i|^2
    \to 0
\end{align}
where we used $\sum_{i} N^{-1}|\mathbf{Z}_{-\ell,c}^i|^2\to d-1$ and $|b_N|\to 0$ almost surely in the last step. 
   Together with  \eqref{eq:boundlampdap} this ends the proof.
\end{proof}

\subsection{Quadratic approximation of $\Psi_N$}
\label{app:quad2}
The next step is to find a good approximation of $\Psi_N$. 
We need to introduce the constant
\begin{align}
    T_\ell:=\int_0^1 F_\ell^{-1}(u)\Phi^{-1}(u)=\mathbb{E}[Z_\ell\psi_\ell(Z_\ell)]\leq 1
\end{align}
where $\psi_\ell(Z_\ell) = \Phi^{-1}F_\ell(Z_\ell)$.
Moreover we denote the empirical covariances
by
\begin{align}
    \hat{m}_{kk'}=N^{-1} \sum_i Z_{k,c}^iZ_{k',c}^i.
\end{align}
We show the following approximation.
\begin{lemma}\label{lem:assemble_quadratic}
Under the assumptions of Lemma~\ref{lem:quad}
 uniformly on $\|\sigma\|\le CN^{-1/2}$,
\begin{equation}
\Psi_N(\sigma)-\Psi_N(0)=\widehat m_{\ell\ell}^{-1/2}\Big[\mathcal Z_N^\top\sigma+\tfrac12\lambda\|\sigma\|^2\Big]+o_P(N^{-1}),
\end{equation}
where $\mathcal Z_{N,k}:=A_{N,k}-T_\ell\,\widehat m_{\ell k}$ ($A_N=\nabla \Lambda_N(0)$) and $\lambda:=B_\ell-T_\ell$.
\end{lemma}
 \begin{remark}
     This lemma is not strictly necessary for the proof of the main result because we later characterize the extrema of $\Psi_N$ through their derivatives. We nevertheless include it because it provides the intuition that 
     $\hat{\sigma}\approx -\Zc_N/\lambda$. 
     Later we will use the expansion of $q(\sigma)$
     and the limiting behaviour of $\Lambda_N(0)$ as crucial input of this lemma.
 \end{remark}
\begin{proof}
As usual, we denote by $O_P$ and $o_P$ bounded in probability statements. All asymptotics are uniform 
in $\sigma$ for $|\sigma|\leq C/\sqrt{N}$ for fixed $C>0$. We note that $\hat{m}_{k,k}\to 1$ and
$\hat{m}_{kk'}=O_P(N^{-1/2})$.
We can expand
\begin{align}
    q(\sigma)=\hat{m}_{\ell\ell}
    + 2\sum_{k\neq \ell} \sigma_k \hat{m}_{k\ell}
    + \sum_{k,k'\neq \ell}\sigma_k \hat{m}_{kk'}\sigma_{k'}.
\end{align}
Note that
\begin{align}
   2 \sum_{k\neq \ell} \sigma_k \hat{m}_{k\ell}
    + \sum_{k,k'\neq \ell}\sigma_k \hat{m}_{kk'}\sigma_{k'}=O_P(N^{-1})
\end{align}
and therefore $\sqrt{1+x}^{-1}=1-x/2+O(x^2)$ implies
\begin{align}\label{eq:exp_q}
    q(\sigma)^{-1/2}= \hat{m}_{\ell\ell}^{-1/2}
    \left(1- \frac{2\sum_{k\neq \ell} \sigma_k \hat{m}_{k\ell}
    + \sum_{k,k'\neq \ell}\sigma_k \hat{m}_{kk'}\sigma_{k'}}{2\hat{m}_{\ell\ell}}\right) +O_P(N^{-2})
    =: \hat{m}_{\ell\ell}^{-1/2}(1+E(\sigma)).
\end{align}
Note that using the convergence of $\hat{m}_{\ell\ell}$
and $\hat{m}_{\ell k}$ we find
\begin{align}
    E(\sigma)=-\sum_{k\neq \ell} \sigma_k \hat{m}_{k\ell}-\frac{|\sigma|^2}{2}+o_P(N^{-1})=O_P(N^{-1}).
\end{align}

We now consider the term $\Lambda_N(0)$.
Note that 
\begin{align}
    \Lambda_N(0)=
    \sum_{i=1}^N \int_{(i-1)/N}^{i/N} \Phi^{-1}(u)(Z_\ell)_{(i)}\, du
    = \sum_{i=1}^N \int_{(i-1)/N}^{i/N} \Phi^{-1}(u)\hat{F}_N^{-1}(u)\, du
    =\int \Phi^{-1}(u)\hat{F}_N^{-1}(u)\, du
\end{align}
By the Glivenko-Cantelli theorem $\lVert \hat{F}_N-F_\ell\rVert\to 0$ almost surely and therefore 
$\hat{F}_N^{-1}(u)\to F_\ell^{-1}(u)$ pointwise almost surely (note that $F^{-1}$ is continuous).
In light of this the previous display implies
\begin{align}\label{eq:lambda0}
   \Lambda_N(0)= \int \Phi^{-1}(u)F_\ell^{-1}(u)\, du+o_P(1)
   =T_\ell +o_P(1)
\end{align}
and almost surely
\begin{align}\label{eq:limlamb0}
   \lim_{N\to \infty} \Lambda_N(0)\to T_\ell.
\end{align}
Combining this with the expansion \eqref{eq:exp_q}
and Lemma~\ref{lem:quad} we get
\begin{align}
\begin{split}
    \Psi_N(\sigma)-\Psi_N(0)
   & = \hat{m}_{\ell\ell}^{-1/2}\left((1+E(\sigma)+O_P(N^{-2}))\left(\Lambda_N(0)+\frac{B_\ell|\sigma|^2}{2}+\sigma\cdot \nabla \Lambda_N(0)+o_P(N^{-1})\right)
    -\Lambda_N(0)\right)
    \\
    &= \hat{m}_{\ell\ell}^{-1/2}
    \left(\frac{(B_\ell-T_\ell)|\sigma|^2}{2}
    +\sum_{k\neq l}\left(\sigma_k \left(\nabla_k \Lambda_N(0)- T_\ell \hat{m}_{k\ell}\right)\right)\right)+o_P(N^{-1}).
    \end{split}
\end{align}
Here we used that $\nabla \Lambda_N(0)=O_P(N^{-1/2})$
to bound the product with $E(\sigma)$. This will be shown in the next lemma (which is independent of this lemma). In light of the definition of $\lambda$ and
$\mathcal Z_{N,k}$ this ends the proof.

\end{proof}

We try to characterize the approximation of $\Psi_N$ a bit more precisely by investigating $\mathcal Z_{N,k}$
and studying $\lambda$.
\begin{lemma}\label{lem:zn}
    Let as before  $\mathcal Z_{N,k}:=\nabla_k\Lambda_N(0)-T_\ell\,\widehat m_{\ell k}$. Then
    \begin{align}
      \sqrt{N}  \mathcal{Z}_{N}\overset{d}{\to} \mathcal{N}(0, (1-T_\ell^2)\mathbf{1}_{d-1}).
    \end{align}
    In particular $\nabla\Lambda_N(0)=O_P(N^{-1/2})$
\end{lemma}
 \begin{proof}
     Generally the H\'ajek representation \citep{hajek1967theory} 
     can be applied to control derivatives of $\Lambda_N$ and this implies that 
     $\nabla_k\Lambda_N(0)\approx N^{-1} Z_k^i \Phi^{-1}(F_\ell(Z_\ell^i))$. In our specific setup we can get the desired result relying only on more elementary results, namely we can apply a general central limit theorem to control the limiting behavior (since we do not consider arbitrary designs).
    Let $R_i$ be the rank of $Z_\ell^{i}$. Since $\sum_iz_i=0$ and
$\widehat m_{\ell k}=N^{-1}\sum_i(Z_\ell^{i}-\bar Z_\ell)Z_k^{i}$, we have the  identity
\[
  \mathcal Z_{N}=\frac1N\sum_{i=1}^Nw_i\,\mathbf Z_{-\ell}^{i},\qquad
  w_i:=z_{R_i}-T_\ell\big(Z_\ell^{i}-\bar Z_\ell\big).
\]
The weights $w_i$ depend only on the $Z_\ell$-sample
so we can condition on $Z_\ell$.
By independence $Z_k$ for $k\neq \ell$ are still i.i.d.\ and standardized
and moreover $\mathcal{Z}_{n}$ has independent components.  
We now study the weight structure. First we observe that for almost all samples $Z_\ell$ we find 
\begin{align}
    N^{-1}\sum_iw_i^2=N^{-1}\sum_i\left(z_i^2-2z_{R_i}T_\ell Z_{\ell,c}^i+T_\ell^2(Z_{\ell,c}^i)^2\right)
    = N^{-1}\sum_iz_i^2 - 2T_\ell\Lambda_N(0)+T_\ell^2\hat{m}_{\ell\ell}.
\end{align}
We have shown in \eqref{eq:limlamb0} that
a.s.\ $\lim_{N\to \infty}\Lambda_N(0)=T_\ell$ and
$\hat{m}_{\ell\ell}\to 1$ and finally 
$\lim_{N\to\infty} N^{-1}\sum_iz_i^2\to \int_0^1(\Phi^{-1})^2=1$. Therefore
\begin{align}
    \lim_{N\to\infty} N^{-1}\sum_iw_i^2=
    1-2T_\ell^2+T_\ell^2=1-T_\ell^2.
\end{align}
We also need to control the largest weight and we find
$\max_{1\leq i\leq N} w_i^2=o(N)$ since $z_i=O(\sqrt{\ln(n)})$ and $\max_{1\leq i\leq n} (Z^i_\ell)^2=o(N)$ (see proof of Lemma~\ref{lem:quad}).

We now want to apply the Lindeberg CLT to the sum 
$\sum_{i=1}^N w_i Z_k^i$. The random variables $w_iZ_k^i$ are independent and the variance of $w_iZ_k^i$ is $w_i^2$. 
We define the total variance as
\begin{align}
 s_N = \left(\sum_{i=1}^Nw_i^2\right)^{1/2}.
\end{align}
We find  for $\mathcal{Z}_{N,k}$
\begin{align}
    \sum_{i=1}^NE\big[\frac{(w_iZ_k^i)^2}{s_N^2}\mathbf 1\{|w_iZ_k^i|>\varepsilon s_N\}\big]\le 
    \left(\sum_{i=1}^N\frac{w_i^2 }{s_N^2}\right)E\big[Z_k^2\mathbf 1\{|Z_k|>\varepsilon s_N/\max_i|w_i|\}\big]\to1\cdot 0=0.
\end{align}
Here we used in the last step that $s_N/\max_{1\leq i\leq N} w_i\to \infty $ almost surely (for almost all  $Z_\ell$).
Therefore the Lindeberg condition is satisfied and
 we conclude that
\begin{align}
   s_N^{-1} N\mathcal{Z}_{N,k}\overset{d}{\to}N(0, 1).
\end{align}
Since $s_N^2/N\to 1-T_\ell^2$ a.s., we find by Slutsky's Theorem and using independence of $Z_k$
\begin{align}
  N^{1/2} \mathcal{Z}_{N}\overset{d}{\to}N(0, \mathrm{1}_{d-1}(1-T_\ell^2))
\end{align}
conditional on almost every $Z_\ell$. Dominated convergence of the characteristic function implies that the statement holds unconditionally.

 \end{proof}
 \begin{lemma}
     For all non-Gaussian distributions of $Z_\ell$ we have
     \begin{align}
         \lambda=B_\ell-T_\ell>0.
     \end{align}
 \end{lemma}
\begin{proof}
We show the result for $f_\ell>0$ and omit the tail bounds to control the boundary terms. The extensions are straightforward.
Recall that
\begin{align}
    T_\ell=\int_0^1 F_\ell^{-1}(u)\Phi^{-1}(u)\, du=\mathbb{E}[Z_\ell\psi_\ell(Z_\ell)]
    \\
      B_\ell =- \int_0^1\frac{f'_\ell(F_\ell^{-1}(u))}{f_\ell(F_\ell^{-1}(u))} \Phi^{-1}(u)\,du
      =\mathbb{E}[s_\ell(Z_\ell)\psi_\ell(Z_\ell)]
\end{align}
where $\psi_\ell(Z_\ell) = \Phi^{-1}F_\ell(Z_\ell)$ 
and $s_\ell=-f'_\ell / f_\ell$. Here both identities follow from the substitution $x=F_\ell^{-1}(u)$.
Since $\Gamma:=\psi_\ell(Z_\ell)\sim\mathcal N(0,1)$, note that by Cauchy-Schwarz
\begin{align}
   T_\ell^2\leq  \mathbb{E}[Z_\ell^2]\mathbb{E}[\Gamma^2]=1
\end{align}
with equality if and only if $Z_\ell $ is Gaussian.
Integration by parts implies further
\begin{align}
    B_\ell=\mathbb{E}[s_\ell(Z_\ell)\psi_\ell(Z_\ell)]
    =\int f_\ell' \psi_\ell=\int f_\ell \psi'_\ell
    =\mathbb{E}[\psi'_\ell(Z_\ell)]
\end{align}
and writing $h=\psi_\ell^{-1}$
\begin{align}
    T_\ell=\mathbb{E}[Z_\ell\psi_\ell(Z_\ell)]
    =\mathbb{E}[\Gamma\,h(\Gamma)]=\mathbb{E}[h'(\Gamma)]=
    \mathbb{E}[1/\psi_\ell'(Z_\ell)]
\end{align}
So Cauchy-Schwarz implies $B_\ell T_\ell\geq 1$
and thus $B_\ell\geq 1>T_\ell$.
\end{proof}

\subsection{Bounding directional derivatives through convexity}\label{app:convexity}
The ultimate goal is to study the local minimizers of $\Psi_N$. The quadratic approximation is not sufficient for this goal because it does not rule out local minimizers through small oscillations. 
However, we can exploit that $\Lambda_N(\sigma)$ is convex. Since it is only piecewise differentiable we need to consider directional derivative. 
For convex $f$, $s\mapsto\{f(\sigma+se)-f(\sigma)\}/s$ is non-decreasing on $(0,\infty)$, so $f'(\sigma;e)$
exists for every $\sigma,e$, and $f'(\sigma;e)+f'(\sigma;-e)\ge0$. Hence, for every $s>0$,
\begin{equation}\label{eq:sandwich}
  \frac{f(\sigma)-f(\sigma-se)}{s}\ \le\ f'(\sigma;e)\ \le\ \frac{f(\sigma+se)-f(\sigma)}{s}.
\end{equation}
A necessary condition for a local minimum is that all directional derivatives are non-negative.
Note that $q$ is continuously differentiable, so by the product rule $\Psi_N'(\sigma;e)$ exists and
\begin{equation}\label{eq:product}
\Psi_N'(\sigma;e)=q^{-1/2}(\sigma)\,\Lambda_N'(\sigma;e)+\Lambda_N(\sigma)\,\nabla q^{-1/2}(\sigma)^\top e .
\end{equation}
If $\widehat\sigma$ is a local minimizer of $\Psi_N$ with $q(\widehat\sigma)>0$, then
$\Psi_N'(\widehat\sigma;e)\ge0$ for all $e\neq 0$.
We show that our global approximation result allows to derive tight bounds for the derivative of $\Psi_N$ which then enable us to control minimizers of $\Psi_N$.
 \begin{lemma}\label{lem:lin}
Assume $F_\ell$ is non-Gaussian with finite Fisher information. For every $C<\infty$, with probability tending to one $q>0$ on $\bar B(C)$, and
\[
  \rho_N(C):=\sup_{|\sigma|\leq C\sqrt{N}^{-1}}\ \sup_{|e|=1}
  \Big|\Psi_N'(\sigma;e)-\,e^\top\big(\Zc_N+\lambda\sigma\big)\Big|=o_P(N^{-1/2}).
\]
\end{lemma}
 
\begin{proof}
As usual all $O_P,o_P$ below are uniform over $|\sigma|\leq C\sqrt{N}^{-1}$ and $|e|=1$.
 We first bound the directional derivatives of $\Lambda_N$.
We find using Lemma~\ref{lem:quad}
(recall the notation $\nabla \Lambda_N(0)=A_N$)
\begin{align}
\begin{split}
    \Lambda_N'(\sigma,e)
    \leq \frac{\Lambda_N(\sigma+se)-\Lambda_N(\sigma)}{s}
    &\leq \frac{B_\ell|\sigma+se|^2/2-B_\ell|\sigma|^2/2+se\cdot A_N}{s}+s^{-1}E  
    \\
    &= B_\ell  e^\top \sigma+e\cdot A_N+\frac{B_\ell s}{2}+s^{-1}E
    \end{split}
\end{align}
where $E$ denotes an error term which satisfies
$E=o_P(N^{-1})$ by Lemma~\ref{lem:quad}.
For a fixed $\eta>0$ we can now set $s=\eta/(B_\ell\sqrt{N})$ so that $B_\ell s/2=\eta/(2\sqrt{N})$ and then
\begin{align}
   \lim_{N\to\infty} \mathbb{P}(s^{-1}E>\eta/(2\sqrt{N}))
    = \lim_{N\to \infty}\mathbb{P}(E>\eta^2/(2B_\ell)N^{-1})\to 0
\end{align}
because $E=o_P(N^{-1})$.
Bounding similarly with the same strategy 
\begin{align}
     \Lambda_N'(\sigma,e)
    \geq \frac{\Lambda_N(\sigma)-\Lambda_N(\sigma-se)}{s}
\end{align}
we conclude that
\begin{align}
      \Lambda_N'(\sigma,e)=B_\ell  e^\top \sigma+e\cdot A_N+o_P(N^{-1/2}).
\end{align}
We can also calculate
\begin{align}
    \nabla q^{-1/2}(\sigma)
    =-\frac{1}{2q^{3/2}(\sigma)}\nabla q(\sigma). 
\end{align}
 Recall that 
 \begin{align}
     \nabla_k q(\sigma)=2\hat{m}_{\ell k}+2\sum_{k'\neq \ell} \hat{m}_{k k'}\sigma_{k'}
     =2\hat{m}_{\ell k} + O_P(N^{-1})
 \end{align}
 and 
 \begin{align}
  q(\sigma)=\hat{m}_{\ell\ell}+O_P(N^{-1}),\quad   q^{-3/2}(\sigma)=\hat{m}_{\ell\ell}^{-3/2}
     + O_P(N^{-1}).
 \end{align}
 Finally we recall that by \eqref{eq:lambda0}
 \begin{align}
     \Lambda_N(\sigma)=\Lambda_N(0)+O_P(N^{-1})=
     T_\ell+O_P(N^{-1}).
 \end{align}
Combining the last four displays with the product rule in \eqref{eq:product} we find
\begin{align}
    \Psi_N'(\sigma;e)
    = \hat{m}_{\ell\ell}^{-1/2}\left(B_\ell  e^\top \sigma+e\cdot A_N\right) 
    -T_\ell \hat{m}_{\ell\ell}^{-3/2}\hat{m}_{\ell \cdot}\cdot e
    +o_P(N^{-1/2})
    = \left(B_\ell  e^\top \sigma+e\cdot (A_N-T_\ell \hat{m}_{\ell \cdot})\right) 
    +o_P(N^{-1/2})
\end{align}
 where we used $\hat{m}_{\ell\ell}=1+O_P(\sqrt{N}^{-1})$ in the last step.
 Recalling the definition of $\mathcal{Z}_N$ this ends the proof.

\end{proof}

\subsection{Conclusion of the main result}\label{app:main}

We can now prove our main result. Note that so far we always considered an arbitrary but fixed ball
$\sigma \in B(C/\sqrt{N})$. It is convenient to replace $C$ by an increasing sequence which is possible by applying the following standard lemma
to the error terms.

\begin{lemma}\label{lem:diag}
Let $X_N(C)\ge0$ be random variables, nondecreasing in $C$, with $X_N(C)\to0$ in probability for every fixed
$C<\infty$. Then there is a deterministic sequence
$C_N\to\infty$ with $X_N(C_N)\to0$ in probability, and then
$X_N(C_N')\to0$ in probability for every $C_N'\le C_N$.
\end{lemma}
 
\begin{proof}
The proof relies on a diagonal argument.
For $k\in\mathbb N$ let 
\begin{align}
    a_N(k):=P(X_N(k)>1/k).
\end{align} 
By convergence in probability to 0 we find that $\lim_{N\to\infty}a_N(k) = 0$. Choose
$N_1<N_2<\cdots$ with $a_N(k)\le1/k$ for all $N\ge N_k$, and put $C_N:=k$ for $N_k\le N<N_{k+1}$
($C_N:=1$ for $N<N_1$). Then $C_N\to\infty$ and $a_N(C_N)\le1/C_N$ (for $N\geq N_1$). In particular, for any $\eta>0$ we find for all $N$ with
$1/C_N<\eta$, 
\begin{align}
    P(X_N(C_N)>\eta)\le a_N(C_N)\le1/C_N.
\end{align}
\end{proof}

\begin{theorem}\label{th:asymptotic}
Assume that $Z_\ell$ is non-Gaussian and has finite Fisher information. There is a deterministic sequence $C_N\to\infty$  such that,
with $\bar B_N:=\bar B(C_N)$ the following three properties hold.
\begin{enumerate}
\item[(i)] With probability tending to one every local
  minimizer of $\Psi_N$ over $\bar B_N$ lies in the interior of $\bar B_N$.
\item[(ii)] Any two local minimizers of $\Psi_N$  differ by $o_P(N^{-1/2})$.
\item[(iii)] Let $V_\ell=(1-T_\ell^2)/(B_\ell-T_\ell)^2$.
Every sequence of local minimizers $\widehat\sigma_N$ of $\Psi_N$ with
  $P(\widehat\sigma_N\in\bar B_N)\to1$ satisfies
  \[
    \sqrt N\,\widehat\sigma_N=-\sqrt N\,\Zc_N/\lambda+o_P(1)\ \overset{d}{\to}\ \mathcal N\big(0,V_\ell I_{d-1}\big).
  \]
  In particular this holds for every sequence of local minimizers with $\sqrt N\widehat\sigma_N=O_P(1)$.
\end{enumerate}
\end{theorem}

\begin{proof}
We first apply Lemma~\ref{lem:diag} to the random variables $X_N(C)=\sqrt{N}\rho_N(C)=o_P(1)$ where $\rho_N$
was introduced in Lemma~\ref{lem:lin}
and we find a sequence $C_N\to \infty$ such that
$\sqrt{N}\rho(C_N)\to 0$ in probability.

We now consider the inward derivative $e=-\sigma/|\sigma|$ and find from Lemma~\ref{lem:lin} that
\begin{align}
\Psi_N'(\sigma;e)\leq\,e^\top\big(\Zc_N+\lambda\sigma\big)+\rho_N(C_N)
    \leq |\Zc_N|-\lambda |\sigma|+\rho_N(C_N)
\end{align}
By construction of $C_N$ we have 
$\mathbb{P}(\sqrt{N}\rho_N(C_N)>1)\to 0$
and by Lemma~\ref{lem:zn}
we find 
$\mathbb{P}(\sqrt{N}|\Zc_N|>C_N\lambda/2)\to 0$.
Thus we conclude that with probability going to 1 
we find for all $\sigma\in \partial B(C_N/\sqrt{N})$
that for $N$ sufficiently large ($C_N\to \infty$)
\begin{align}
\sqrt{N}\Psi_N'(\sigma;-\sigma/|\sigma|)\leq  C_N\lambda/2-\lambda C_N+1<0.
\end{align}
In particular there cannot be a minimizer on the boundary in this case.
So every local minimizer lies in the interior and such a local minimizer exists by continuity of $\Psi_N$.

For a point $\sigma\in B(C_N/\sqrt{N})$
such that $\Zc_N+\lambda \sigma\neq 0$ we now
estimate for $e=-(\Zc_N+\lambda \sigma)/|\Zc_N+\lambda\sigma|$ 
\begin{align}
    \Psi'_N(\sigma, e)
    \leq -|\Zc_N+\lambda \sigma| + \rho_N(C_N).
\end{align}
Consider two local minima $\hat{\sigma}_1$ and $\hat{\sigma}_2$. The gradient condition $ \Psi'_N(\hat{\sigma}_i, e)\geq 0$ then implies
\begin{align}
    \lambda|\hat{\sigma}_1-\hat{\sigma}_2|
    \leq |\Zc_N+\lambda \hat{\sigma}_1|
    +|\Zc_N+\lambda \hat{\sigma}_2|\leq 2 \rho_N(C_N)
\end{align}
which ends the proof of the second item.

The same reasoning implies for every minimizer $\hat{\sigma}$ that
\begin{align}
   | \sqrt{N}\hat{\sigma}+\sqrt{N}\Zc_N/\lambda|\leq
   \sqrt{N}\rho_N(C_N)/\lambda
\end{align}
and therefore indeed
\begin{align}
    \sqrt N\,\widehat\sigma_N=-\sqrt N\,\Zc_N/\lambda+o_P(1)\ \overset{d}{\to}\ \mathcal N\big(0,V_\ell I_{d-1}\big)
\end{align}
where we used Lemma~\ref{lem:zn} to infer the convergence to the normal distribution in the last step.
\end{proof}

\subsection{Proof of equivalence of $\Lambda_N$ and $\Lambda_N'$}\label{app:lambdalambdap}

The goal of this subsection is to show that 
the difference between $\Lambda_N$ and $\Lambda_N'$ is negligible as stated in Lemma~\ref{le:lambda-lambdap}. This is a deterministic statement which requires a slightly technical but straightforward analysis.
 
\begin{proof}[Proof of Lemma~\ref{le:lambda-lambdap}]
Let us first recall the notation. We fix $N$ for the proof and do not indicate $N$ dependence here as usual. We defined
\begin{align}
    z_i&=N\int_{(i-1)/N}^{i/N}\Phi^{-1}(u)\,\mathrm{d}u,
    \\
    z'_i&=\Phi^{-1}(i/(N+1)),
    \\
    b_i&=z_i-z_i'.
\end{align}
There are two parts to the proof. 
The first is essentially an application of
Cauchy-Schwarz. Fix a $\sigma$ and assume that
$u(\sigma)\mathbf{Z}^i$ are pairwise different. 
Denote the rank of $u(\sigma)\mathbf{Z}^i$ by $R_i$.
There is a neighborhood of $\sigma$  so that all ranks $R_i$ are constant and the definition of $\Lambda$ and $\Lambda'$ implies that in that neighborhood
\begin{align}
    \Lambda_N(\sigma)=N^{-1} \sum_{i=1}^N u(\sigma)\mathbf{Z}^i_c z_{R_i},
    \quad 
    \Lambda'_N(\sigma)= N^{-1}\sum_{i=1}^N u(\sigma)\mathbf{Z}^i_c z'_{R_i}.
\end{align}
These are just linear functions of $\sigma$ and we find
\begin{align}
     \nabla_k\Lambda_N(\sigma)=N^{-1} \sum_{i=1}^N \mathbf{Z}^i_{k,c} z_{R_i},
    \quad 
   \nabla_k \Lambda'_N(\sigma)=N^{-1} \sum_{i=1}^N \mathbf{Z}^i_{k,c} z'_{R_i}.
\end{align}
We conclude that
\begin{align}\label{eq:gradLambda}
    |\nabla\Lambda_N(\sigma)-\nabla \Lambda'_N(\sigma)|
    =N^{-1}\left|\sum_{i=1}^N \mathbf{Z}^i_{-\ell,c} (z_{R_i}-z'_{R_i}) \right|
    \leq \frac{|b|}{N} \left(\sum_{i=1}^N (\mathbf{Z}^i_{-\ell,c})^2\right)^{1/2}\leq \frac{|b|}{\sqrt{N}} s_N.
\end{align}
The assumption implies that $u(0)\mathbf{Z}^i_c=\mathbf{Z}_{\ell,c}^i$ 
are pairwise different so we conclude in particular from \eqref{eq:gradLambda} that \eqref{eq:grad} holds.
For \eqref{eq:fun} we note that 
$f(t)=\Lambda_N(t\sigma)-\Lambda_N'(t\sigma)$ is almost everywhere differentiable (there are at most $N(N-1)/2$ crossings of variables) and absolutely continuous so we find
\begin{align}
    \Lambda_N(\sigma)-\Lambda_N'(\sigma)
    -(\Lambda_N(0)-\Lambda_N'(0))=f(1)-f(0)
    =\int_{0}^1 \frac{d}{dt}f(t)\,dt
    = \int_{0}^1 \sigma (\nabla\Lambda_N(t\sigma)-\nabla\Lambda'_N(t\sigma)\,dt.
\end{align}
Thus follows by applying \eqref{eq:gradLambda} almost everywhere.

The second step is to show that $|b|\to 0$ as $N\to \infty$. This is now a statement about the normal distribution and its tail behaviour.
It turns out that the $i=1$ and $i=N$ term need special analysis so we separate those.
By symmetry $z_{N+1-i}=-z_i$, $z'_{N+1-i}=-z'_i$
we find
\begin{align}
    |b_N|^2
    = 2(z_N-z'_N)^2+ \sum_{i=2}^{N-1} |z_i-z'_i|^2.
\end{align}

We first consider the second term.
Since $(i-1)/N<i/(N+1)<i/N$ we find that $z_i, z_i'\in \Phi^{-1}(I_i)$ where $I_i=[(i-1)/N,i/N]$.
Thus the error $b_i=z_i-z_i'$ is controlled by 
\begin{align}
   | b_i|\leq \Phi^{-1}(i/N)-\Phi^{-1}((i-1)/N).
\end{align}
This is sufficient for the center cells (and meaningless for the boundary cells).

 We bound  the second summand as follows
\begin{align}
\begin{split}
    \sum_{i=2}^{N-1} |z_i-z'_i|^2
 &   \leq \sum_{i=2}^{N-1} |\Phi^{-1}(i/N)-\Phi^{-1}((i-1)/N)|^2
 \\
 &=\sum_{i=2}^{N-1}
   \left( \int_{(i-1)/N}^{i/N} (\Phi^{-1})'(u) \, du\right)^2
 \\
 &\leq \sum_{i=2}^{N-1}
   \frac{1}{N} \int_{(i-1)/N}^{i/N} ((\Phi^{-1})'(u))^2 \, du
  = \frac{2}{N}\int_{0}^{(N-1)/N}\frac{1}{\varphi(\Phi^{-1}(u))^2}\, du.
   \end{split}
\end{align}
Here we used Cauchy-Schwarz to pull the square in the integral and $(\Phi^{-1})'(u)=(\varphi(\Phi^{-1}(u)))^{-1}$ (and symmetry in the last step).
Now we can substitute $t=\Phi^{-1}(u)$ and then we get
\begin{align}
    \int_{1/2}^{(N-1)/N}\frac{1}{\varphi(\Phi^{-1}(u))}\, du=\int_{0}^{t_N}\frac{1}{\varphi(t)}\, dt
\end{align}
 where $t_N=\Phi^{-1}(1-1/N)$.
To bound the last integral, note that $x^2/2\le t_N x/2$ for $0\le x\le t_N$, so
\[
  \int_0^{t_N}\frac{\dd x}{\varphi(x)}
  = \sqrt{2\pi}\int_0^{t_N} e^{x^2/2}\dd x
  \le \sqrt{2\pi}\int_0^{t_N} e^{t_N x/2}\dd x
  \le \frac{2\sqrt{2\pi}}{t_N}\,e^{t_N^2/2}
  = \frac{2}{t_N\,\varphi(t_N)} .
\]
To bound this further we use the classical tail bounds (sometimes called Mill bounds) for $t>0$
\begin{equation}\label{eq:mills}
  \frac{t}{1+t^2}\,\varphi(t) \;\le\; 1-\Phi(t) \;\le\; \frac{\varphi(t)}{t}.
\end{equation}
This implies for $t=t_N=\Phi^{-1}(1-1/N)$ that 
\begin{align}
   \frac{\varphi(t_N)}{t_N}\geq 1-\Phi(t_N)=1/N 
\end{align}
and therefore
\begin{align}
    \sum_{i=2}^{N-1} |z_i-z'_i|^2\leq \frac{2}{N}\frac{2}{t_N\varphi(t_N)}\leq \frac{4}{t_N^2}.
\end{align}

To bound the boundary cell the key inequality is the following decay bound. Let $\bar{\Phi}(t)=1-\Phi(t)=\mathbb{P}(\Gamma>t)$, then for $t,s\geq 0$
\begin{align}\label{eq:phib}
    \bar{\Phi}(t+s)\leq e^{-ts}\bar{\Phi}(t).
\end{align}
Indeed
\begin{align}
     \bar{\Phi}(t+s)=\int_t^{\infty}
     \varphi(x+s)\, dx
     =\int_t^{\infty}
     \varphi(x) e^{-sx-s^/2}\, dx
     \leq \int_t^{\infty}
     \varphi(x) e^{-st}\, dx
     =e^{-st} \bar{\Phi}(t).
\end{align}
Now we can bound using $t=\Phi^{-1}(u)$ followed by an integration by parts and \eqref{eq:phib}
\begin{align}
\begin{split}
    z_N -t_N&= N\int_{1-N^{-1}}^1 (\Phi^{-1}(u)-t_N)\, du
    =N\int_{t_N}^\infty (t-t_N)\varphi(t)\,dt
    \\
    &=N\int_{t_N}^\infty \bar{\Phi}(t)\, dt
    =N\int_{0}^\infty \bar{\Phi}(t+t_N)\, dt
    \leq N\int_{0}^\infty \bar{\Phi}(t_N)e^{-t_Nt}\, dt
    \\
    &=\frac{1}{t_N}.
    \end{split}
\end{align}
 Here we used in the last step that $\bar{\Phi}(t_N)=1/N$. We conclude that
 \begin{align}
     t_N\leq z_N\leq t_N+t_N^{-1}.
 \end{align}
Similarly, we find
\begin{align}
    \frac{1}{N+1}=\bar{\Phi}(z_N')
    =\bar{\Phi}((z_N'-t_N)+t_N)
    \leq \bar{\Phi}(t_N)e^{-(z_N'-t_N)t_N}
    =N^{-1}e^{-(z_N'-t_N)t_N}.
\end{align}
This implies
\begin{align}
   0\leq z_N'-t_N\leq \ln(1+1/N)/t_N\leq \frac{1}{Nt_N}.
\end{align}
 We conclude that 
 \begin{align}
     |z_N-z_N'|^2\leq \frac{1}{t_N^2}
 \end{align}
 and therefore combining both steps we get
 \begin{align}
     |b|^2\leq \frac{6}{t_N^2}.
 \end{align}

Finally we note that for $N$ large
\begin{align}
    t_N=\bar{\Phi}^{-1}(N^{-1})\sim \sqrt{2\ln(N)}\to \infty.
\end{align}
 This implies 
 \begin{align}
     |b|^2\leq \frac{3+o(1)}{\ln(N)}\to 0.
 \end{align}
\end{proof}

\section{Wasserstein contrast landscape for  distributions with atoms}\label{app:atoms}

In this appendix, we investigate the statistical properties of the Wasserstein contrast when considering sources with atoms. We will use the notation from the previous section. 
We  study the smoothness of the  function 
\begin{align}
    \sigma\mapsto W_2^2\big((1+|\sigma|^2)^{-1/2} u(\sigma)^\top \mathbf{Z}, \Gamma\big)
\end{align}
where as usual $\mathbf{Z}$ is centered with covariance $\mathbf{I}_d$ so that
$\mathbb{E}[(u(\sigma)^\top \mathbf{Z})^2]=1+|\sigma|^2$. The main observation is that this function is not smooth at $\sigma=0$. 
This matches the empirical observation in Figure~\ref{fig:optim_surfaces} and we sketch how this can lead to precise estimation of the independent components (when the right local maximum is found). The two lemmas below state this mechanism directly for the normalized contrast.

Let us first  gain some intuition.  We consider a point mass $\delta_0$ at 0 and let $X\sim \delta_0$. Let $Y$ be a standardized random variable
and we study the function $\sigma \to W_2^2(X+\sigma Y,\Gamma)$.
Write $T$ for the optimal transport from standard normal $\Gamma$ to $Y$. Then   
$X+\sigma Y=\sigma Y$ and an optimal coupling of $\Gamma$ and $\sigma Y$ is $\Gamma\to \sigma T(\Gamma)$ for $\sigma>0$ (because  $T$ is characterized by being the unique monotone map such that $T(\Gamma)$ has the same distribution as $Y$ and $\sigma T$ remains monotone and $\sigma T(\Gamma)$ has the same distribution as $\sigma Y$). Thus,
\begin{align}
    W_2^2(X+\sigma Y,\Gamma)=\mathbb{E}\left((\Gamma-\sigma T(\Gamma))^2\right)
    =1+\sigma^2 \mathbb{E}(Y^2)-
    2\sigma \mathbb{E}(\Gamma T(\Gamma)).
\end{align}
Note that the monotone coupling maximizes correlation (by the rearrangement inequality)
and therefore
$\mathbb{E}(\Gamma T(\Gamma))>\mathbb{E}(\Gamma )\mathbb{E}(T(\Gamma))=0$.
On the other hand, for $\sigma<0$ an optimal coupling is given by $\Gamma\to \sigma T(-\Gamma)$
(which is a non-decreasing function which has the correct distribution by symmetry of the Gaussian). 
This gives
\begin{align}
    W_2^2(X+\sigma Y,\Gamma)=\mathbb{E}\left((\Gamma-\sigma T(-\Gamma))^2\right)
    =1+\sigma^2 \mathbb{E}(Y^2)-
    2\sigma \mathbb{E}(-\Gamma T(\Gamma))
    =1+\sigma^2 -
    2|\sigma|\, \mathbb{E}(\Gamma T(\Gamma)).
\end{align}
We conclude that the Wasserstein distance has a kink at $\sigma=0$, i.e.,
\begin{align}
    W_2^2(X+\sigma Y, \Gamma)=1-2|\sigma|\mathbb{E}(\Gamma T(\Gamma))+o(|\sigma|).
\end{align}
The high-level intuition is that the point mass is dispersed by $O(|\sigma|)$ in both directions and therefore we reduce the transport distance by $O(|\sigma|)$ when coupling optimally (see Figure~\ref{fig:kink_mechanism}).

\begin{figure}[htbp]
    \centering
    \includegraphics[width=\textwidth]{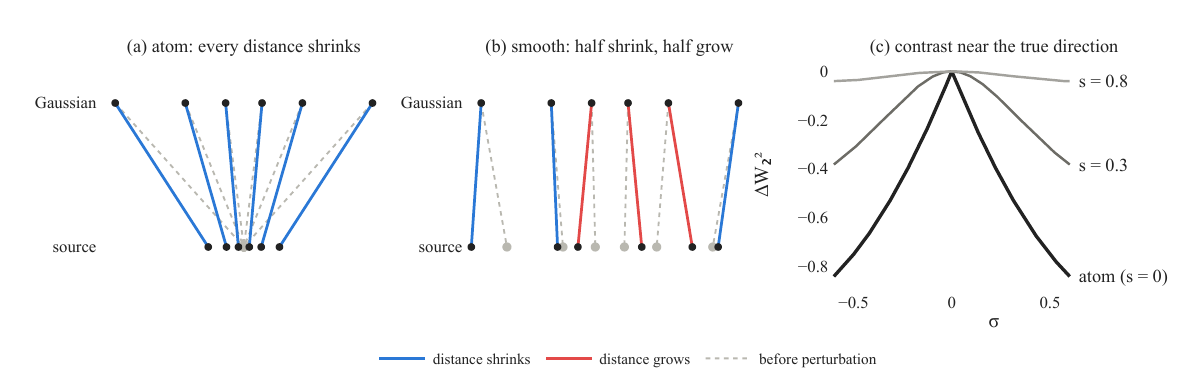}
    \caption{Why an atom gives the $W_2^2$ contrast a kink, for $X\sim\delta_0$ perturbed by $\sigma Y$ with $Y\sim\mathcal N(0,1)$ and $\sigma=0.3$. Lines join Gaussian targets (top) to source points (bottom) in the optimal, order-preserving matching; dotted lines show the matching before the perturbation. (a) All targets are matched to the atom. After the perturbation the points are matched in sorted order, so each point moves towards its target and every distance shrinks: the contrast drops linearly in $|\sigma|$. (b) For a smooth source ($X\sim\mathcal N(0,s^2)$, $s=0.8$) the same shifts are not aligned with the targets; in this draw three distances shrink and three grow, and the first-order change cancels. (c) The population contrast $W_2^2(X+\sigma Y,\Gamma)=(1-\sqrt{s^2+\sigma^2})^2$ relative to $\sigma=0$: a corner at $s=0$ that is rounded off as $s$ grows.}
    \label{fig:kink_mechanism}
\end{figure}

The same phenomenon occurs more generally.
\begin{lemma}
Let $\mathbf{Z}$ be a standardized random variable 
with independent components. Assume that the distribution of $Z_\ell$ contains an atom
 $\alpha_1 \delta_{x_1}$.
Then there is a constant $c>0$ such that
\begin{align}
    W_2^2((1+|\sigma|^2)^{-1/2}(u(\sigma)^\top \mathbf{Z}),\Gamma)\leq 
    W_2^2(Z_\ell,\Gamma)-c|\sigma|+o(|\sigma|).
\end{align}

\end{lemma}
\begin{proof}
To simplify the notation we use the shorthand
\begin{align}
   \mathbf{Z}_\sigma= \sum_{k\neq \ell} \sigma_k Z_k
\end{align}
for the perturbation (i.e., $u(\sigma)^\top \mathbf{Z}=Z_\ell+\mathbf{Z}_\sigma$).
The main idea is to argue as in the heuristic above for the atom and use a simple coupling for the remainder which induces sublinear additional cost (since we are at a critical point we would expect a quadratic scaling for the smooth part).
    We consider an optimal coupling $(T(\Gamma),\Gamma)$ of $Z_\ell$ and $\Gamma$ given by a monotone map $T$. We decompose the law $\mathcal{N}(0,1)$ of $\Gamma$ into $\alpha_1 \mu_1+(1-\alpha_1)\mu_2$ (where $\mu_1$ and $\mu_2$ are probability measures, i.e., have mass 1) so that $T(\Gamma_1)=x_1$ a.s.\
    for $\Gamma_1\sim \mu_1$. 
    We write $\alpha_2=1-\alpha_1$.
    Now we consider a coupling  $\pi_2$ which is given by the law of
    $(T(\Gamma_2)+\mathbf{Z}_\sigma,\Gamma_2)$ where $\Gamma_2\sim \mu_2$. We also define $\pi_1$ to be 
    the law of $(\bar{T}_\sigma(\Gamma_1)+x_1,\Gamma_1)$
    with $\Gamma_1\sim \mu_1$
    where $\bar{T}_\sigma$ parametrizes an optimal coupling between $\mu_1$ and the law of $\mathbf{Z}_\sigma$, i.e., the co-monotone coupling of $\Gamma_1$ and $\mathbf{Z}_\sigma$. 
    Note that $\bar{T}_{\sigma}=|\sigma|\bar{T}_{\sigma/|\sigma|}$. 
    
    Then
    $\alpha_1\pi_1+\alpha_2\pi_2$ couples $Z_\ell+\mathbf{Z}_\sigma$ and $ \Gamma$. We calculate for independent $Z_\ell$ and $\mathbf{Z}_\sigma$
    \begin{align}
    \begin{split}
        W_2^2&(Z_\ell+\mathbf{Z}_\sigma, \Gamma)
        \leq \alpha_1 \mathbb{E}\left((x_1 + \bar{T}_\sigma(\Gamma_1)-\Gamma_1)^2\right)
        +\alpha_2 \mathbb{E}\left((T(\Gamma_2)+\mathbf{Z}_\sigma-\Gamma_2)^2\right)
        \\
        &=
        \alpha_1 \mathbb{E}\left((x_1 -\Gamma_1)^2\right)
        +\alpha_2 \mathbb{E}\left((T(\Gamma_2)-\Gamma_2)^2\right)
        -2\alpha_1 \mathbb{E}(\bar{T}_\sigma(\Gamma_1)\Gamma_1)
        \\
        &\qquad-2\alpha_2 \mathbb{E}\left(\mathbf{Z}_\sigma(\Gamma_2-T(\Gamma_2))\right)
        +\alpha_1\mathbb{E}[\bar{T}_\sigma(\Gamma_1)^2]+\alpha_2\mathbb{E}(\mathbf{Z}_\sigma^2)
        \\
        &= W_2^2(Z_\ell,\Gamma)-2\alpha_1|\sigma|\, \mathbb{E}(\bar{T}_{\sigma/|\sigma|}(\Gamma_1)\Gamma_1)+|\sigma|^2.
        \end{split}
    \end{align}
    Here, we used in the second step that $\mathbb{E}[\bar{T}_\sigma(\Gamma_1)]=\mathbb{E}(\mathbf{Z}_\sigma)=0$, and in the last step that $\mathbf{Z}_\sigma$ is independent of $\Gamma_2$ with mean zero, $\bar{T}_{\sigma}=|\sigma|\bar{T}_{\sigma/|\sigma|}$, and $\mathbb{E}[\bar{T}_\sigma(\Gamma_1)^2]=\mathbb{E}(\mathbf{Z}_\sigma^2)=|\sigma|^2$.
 We conclude that
    \begin{align}
         W_2^2(Z_\ell+\mathbf{Z}_\sigma, \Gamma)
         \leq  W_2^2(Z_\ell,\Gamma)-c|\sigma|+O(|\sigma|^2)
    \end{align}
    where 
    \begin{align}
        c=2\alpha_1\inf_{|\sigma|=1} \mathbb{E}(\bar{T}_{\sigma}(\Gamma_1)\Gamma_1)>0
    \end{align}
    which is positive as it is the infimum of a continuous positive function on a compact set.
    It remains to deal with the normalization.
    Denote the optimal map from $\Gamma$ to $Z_\ell+\mathbf{Z}_\sigma$ by $T_\sigma$. Then
    \begin{align}
    \begin{split}
         W_2^2((1+|\sigma|^2)^{-1/2}(Z_\ell+\mathbf{Z}_\sigma),\Gamma)
         &=\mathbb{E}(\Gamma-(1+|\sigma|^2)^{-1/2}T_\sigma(\Gamma))^2
         \\
         &=\mathbb{E}(\Gamma-T_\sigma(\Gamma)+(1-(1+|\sigma|^2)^{-1/2})T_\sigma(\Gamma))^2
         \\
         &=W_2^2(Z_\ell+\mathbf{Z}_\sigma,\Gamma)+O(|\sigma|^2)
   \end{split} \end{align}
    where the last step used $(1+|\sigma|^2)^{-1/2}=1-|\sigma|^2/2+O(|\sigma|^4)$ and $\mathbb{E}[T_\sigma(\Gamma)^2]=1+|\sigma|^2$.

    This ends the proof.
\end{proof}
 The same argument can be extended to finite samples. 
 Actually, the argument can be made slightly more explicit by relying on the optimal matching for finite samples.
 \begin{lemma}
We denote by $\widehat{Z_\ell+\mathbf{Z}_\sigma}$ the empirical distribution of some sample size $N$ of $Z_\ell+\mathbf{Z}_\sigma$ where $Z_\ell$ and $Z_{-\ell}$ are standardized independent variables and the law of $Z_\ell$ has an atom of mass $\alpha_1$ at $x_1$.
Then there is some $c>0$ depending only on the law of $\mathbf{Z}$ such that
\begin{align}
    W_2^2\big(q(\sigma)^{-1/2}\,\widehat{Z_\ell+\mathbf{Z}_\sigma},\Gamma\big)\leq
    W_2^2\big(q(0)^{-1/2}\,\widehat{Z}_\ell,\Gamma\big)-c|\sigma|+o(|\sigma|) +O_P(N^{-1/2})|\sigma|.
\end{align}
 \end{lemma}
 \begin{remark}
     The main message of this result is that with high probability for large enough sample size $\sigma \mapsto W_2^2(q(\sigma)^{-1/2}\,\widehat{Z_\ell+\mathbf{Z}_\sigma}, \Gamma)$ will have a local maximum at $\sigma=0$, notably due to the multiplicative factor $|\sigma|$ of the $O_P()$ term. This makes exact estimation feasible in such a setting.
    This should be contrasted  with the case where we maximize a $C^2$ function. There an estimation error of $O(\sqrt{N}^{-1})$ (without multiplicative factor $|\sigma|$) is generally unavoidable. In particular this holds for fixed contrast functions used, e.g., in the FastICA algorithm. 
 \end{remark}
     
\begin{proof}[Proof sketch]
In the language of Appendix~\ref{app:asymptotic_variance} we need to show that $\Psi_N(\sigma)\geq \Psi_N(0)+c|\sigma|-O_P(N^{-1/2})|\sigma|+O_P(|\sigma|^2)$, which by Lemma~\ref{lem:reduction_convexity} is equivalent to the claimed upper bound on the Wasserstein distance.
As in the proof of Lemma~\ref{lem:assemble_quadratic}, $q(\sigma)=\hat m_{\ell\ell}+2\sum_{k\neq\ell}\sigma_k\hat m_{k\ell}+\sum_{k,k'\neq\ell}\sigma_k\hat m_{kk'}\sigma_{k'}$ with $\hat m_{k\ell}=O_P(N^{-1/2})$, so the normalization only contributes to the $O_P(N^{-1/2})|\sigma|$ and $O_P(|\sigma|^2)$ terms and it is sufficient to study $\Lambda_N$.

Consider a sample of size $N$ and fix a $\sigma$.
After reordering we can assume that 
$Z_\ell^1\leq Z_\ell^2\leq \cdots\leq Z_\ell^N$.
Moreover, we can reorder such that $\mathbf{Z}_\sigma^i\leq \mathbf{Z}_\sigma^{i+1}$ if $Z_\ell^i=x_1$ while all other ties are broken arbitrarily (this will be required below for independence but a tighter bound could be achieved by ordering them as well).
Condition on $Z_\ell$ and let $I=\{i\in [N]: Z_\ell^i=x_1\}$.  By a standard concentration bound (Hoeffding) up to exponentially small probability $|I|\approx \alpha_1 N$ for $N\gg \alpha_1^{-1}$.
We have
\begin{align}
    \Lambda_N(0) = \frac{1}{N}\sum_{i=1}^Nz_i Z_\ell^i
\end{align}
and by using the same coupling we find from \eqref{eq:lambda+permutation}
\begin{align}
    \Lambda_N(\sigma) \geq  \frac{1}{N}\sum_{i=1}^Nz_i (Z_\ell^i+\mathbf{Z}_\sigma^i)
    =\frac{1}{N}\sum_{i=1}^Nz_i Z_\ell^i+\frac{1}{N}\sum_{i\in I^c}z_i \mathbf{Z}_\sigma^i
    +\frac{1}{N}\sum_{i\in I}z_i \mathbf{Z}_\sigma^i
\end{align}
The first term is equal to $\Lambda_N(0)$. Similar to the proof of Lemma~\ref{lem:zn} it can be shown that   the second term converges in distribution to $\mathcal{N}(0,1)$
when multiplied by $N\big(|\sigma|^2\sum_{i\notin I}z_i^2\big)^{-1/2}$, so it is $O_P(N^{-1/2})|\sigma|$. Finally, the last term is given by
\begin{align}
    |\sigma| \frac{1}{N}\sum_{i\in I}z_i \mathbf{Z}_{\sigma/|\sigma|}^i
\end{align}
and taking the limit $N\to \infty$ we find convergence towards  $\alpha_1|\sigma|\mathbb{E}(\bar{T}_{\sigma/|\sigma|}(\Gamma_1)\Gamma_1)$
from the previous lemma. We omit the details because they are similar to earlier proofs.

\end{proof}

\paragraph{Connection to the experiments.} The Poisson and Binomial sources in Figure~\ref{fig:optim_surfaces} have atoms, so by the first lemma above the population contrast has a strictly positive corner at each true source direction, of a size that does not shrink with $N$. This is the mechanism behind the near-exact ($0.12^\circ$) localization reported in the main text and in Appendix~\ref{app:discrete}.

\section{Algorithmic Enhancements and Methodology}
\label{app:methodology}
This appendix details the additional components summarized in the main text utilized by the OT-ICA algorithm.

\subsection{Analytical Gaussian Targets}
Discrete approximation of the target Gaussian distribution introduces approximation noise. To eliminate this noise, 
we replace point-sampled targets with an analytical formulation, computing the expected value of a standard normal variable within each discrete quantile bin.

\begin{definition}[Analytical Gaussian Target]
Let the uniform probability bin edges for $N$ samples be defined as $p_i = \frac{i}{N}$ for $i = 0, \dots, N$, 
with corresponding Gaussian domain boundaries $z_i = \Phi^{-1}(p_i)$, where $\Phi$ is the standard normal CDF. The analytical target value $T_i$ for the $i$-th sorted sample is:
\begin{equation}
    T_i = N \int_{z_{i-1}}^{z_i} x \varphi(x) \, dx = N \big( \varphi(z_{i-1}) - \varphi(z_i) \big).
\end{equation}
\end{definition}

\subsection{Riemannian Gradient and Retraction}
Because the input data is whitened, the unmixing matrix $\mathbf{B}$ must remain in Orthogonal Group $\mathrm{O}(d)$, satisfying $\mathbf{B}\mathbf{B}^\top = \mathbf{I}$.
Computing the standard Euclidean gradient $\boldsymbol{\Delta} = \nabla_{\mathbf{B}} W_2^2$ points into the unconstrained ambient space.

\begin{figure}[!tb]
    \centering
    \includegraphics[width=0.81\textwidth]{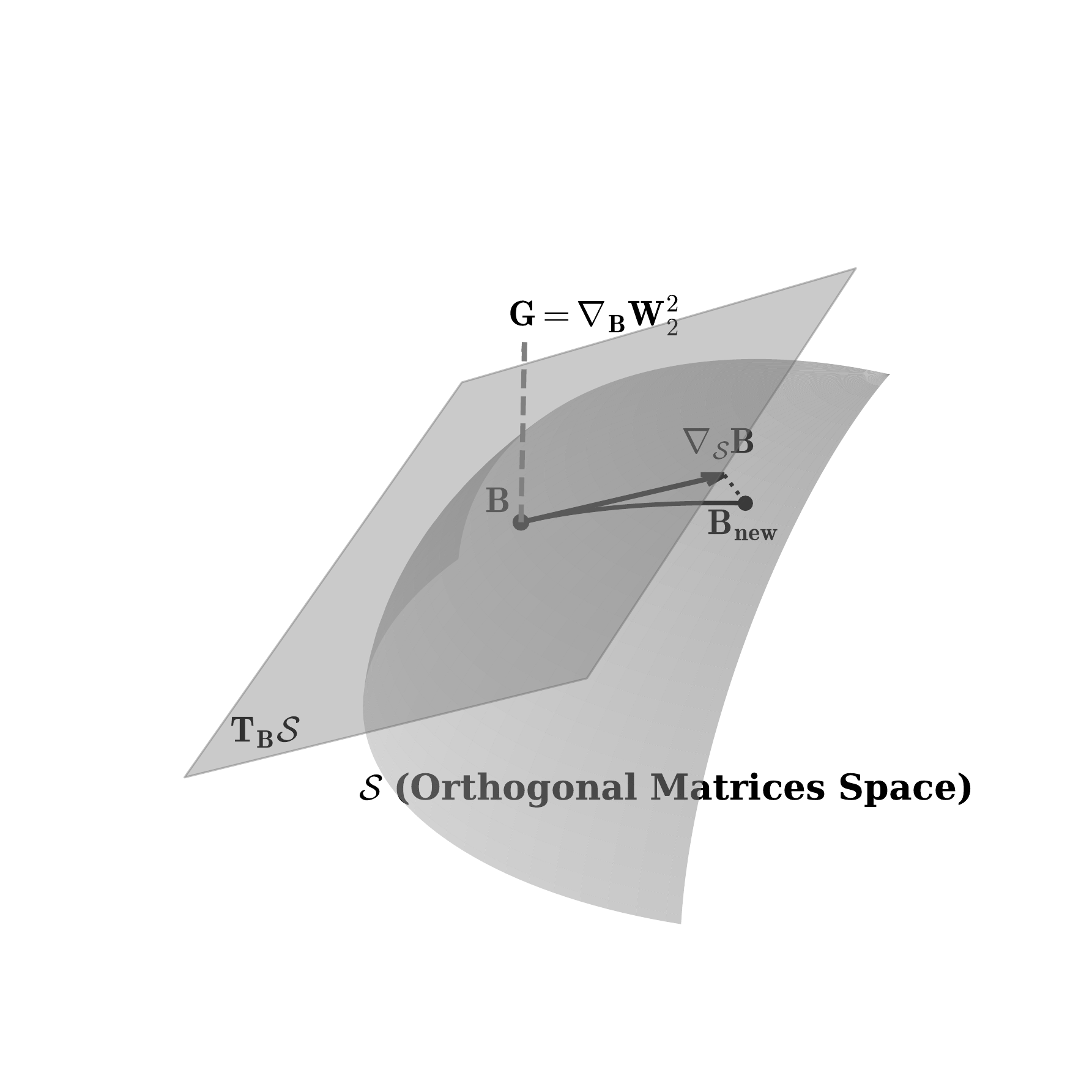}
    \caption{Geometric sketch of optimization on Orthogonal Group $\mathrm{O}(d)$. The Euclidean gradient $\boldsymbol{\Delta}$ is projected onto
    the tangent space to yield the Riemannian gradient $\nabla_{\mathrm{O}(d)} \mathbf{B}$. A retraction maps the estimate back to the orthogonal surface.}
    \label{fig:stiefel_manifold}
\end{figure}

\begin{definition}[Riemannian Gradient on Orthogonal Group $\mathrm{O}(d)$]
Given the Euclidean gradient $\boldsymbol{\Delta}$, the Riemannian gradient resulting from the projection onto the tangent space $T_{\mathbf{B}}\mathrm{O}(d)$ of Orthogonal Group $\mathrm{O}(d)$ is given by \citep{absil2008optimization}:
\begin{equation}
    \nabla_{\mathrm{O}(d)} \mathbf{B} = \boldsymbol{\Delta} - \frac{1}{2}\big(\boldsymbol{\Delta}\mathbf{B}^\top + \mathbf{B}\boldsymbol{\Delta}^\top \big)\mathbf{B}.
\end{equation}
\end{definition}

\begin{definition}[Symmetric Decorrelation Retraction]
To map the matrix back to the orthogonal surface after a tangent step ($\mathbf{B}_{\text{step}}$), we compute the overlap covariance $\mathbf{C} = \mathbf{B}_{\text{step}}\mathbf{B}_{\text{step}}^\top$ and apply the inverse square root:
\begin{equation}
    \mathbf{B}_{\text{new}} = \mathbf{C}^{-1/2} \mathbf{B}_{\text{step}}.
\end{equation}
\end{definition}

\subsection{Continuous Smoothing via Gaussian Dithering}
Applying continuous optimal transport directly to discrete mixtures introduces non-smooth step-functions. 
Adapted from signal processing \citep{schuchman1964dither}, we inject continuous, zero-mean Gaussian noise with standard deviation $\sigma_\text{dither}\in[0.01,0.1]$ (Table~\ref{tab:all_params}) into the 1D projected components prior to sorting.
This decorrelates deterministic quantization errors and convolves the discrete empirical distribution with a Gaussian kernel.

\begin{figure}[!tb]
    \centering
    \includegraphics[width=0.81\textwidth]{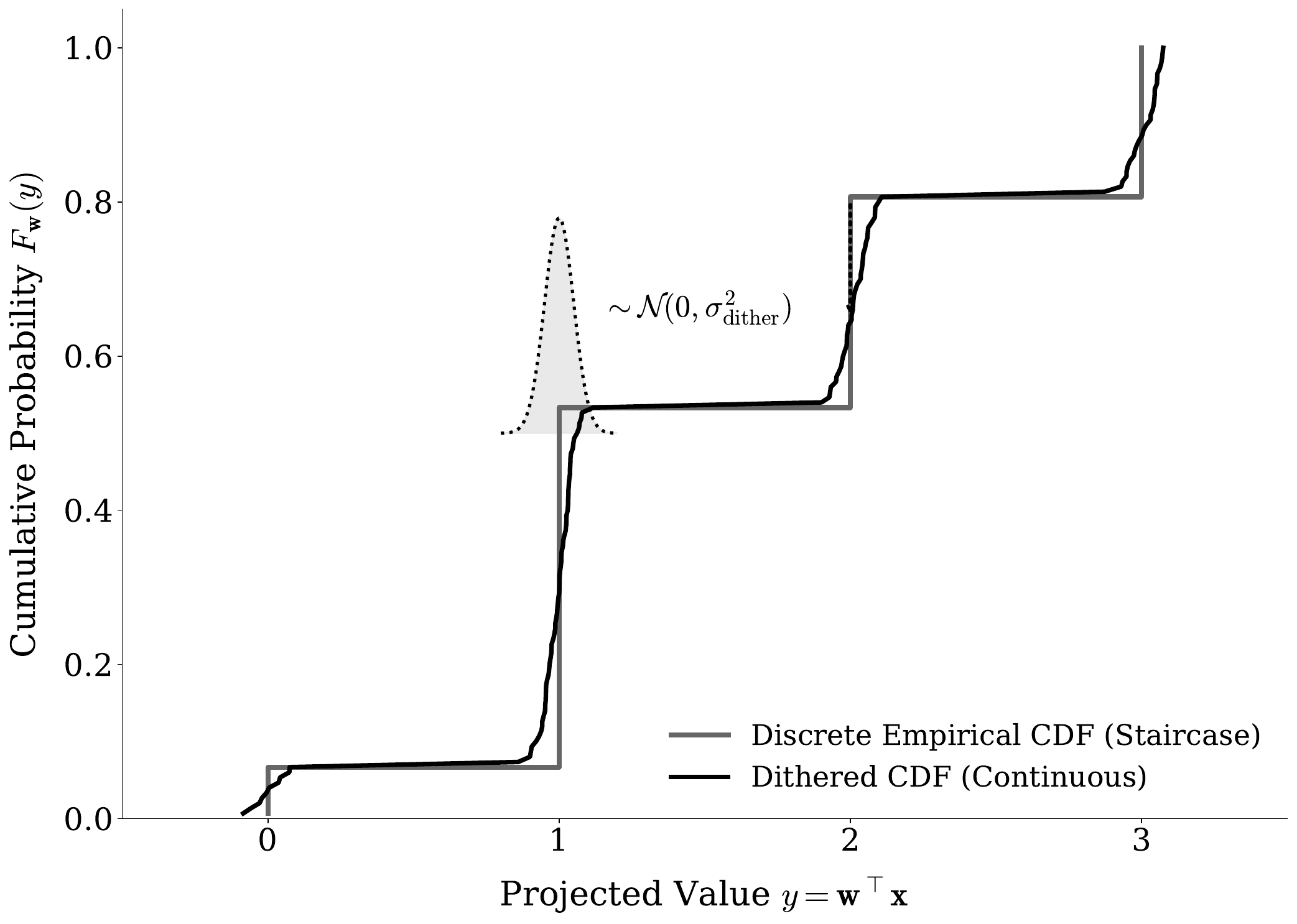}
    \caption{The non-differentiable staircase of a discrete empirical CDF (gray) is convolved with a narrow Gaussian kernel via dithering, yielding a continuous, differentiable curve (black).}
    \label{fig:gaussian_dithering}
\end{figure}

\section{Computational Scaling and Dimensionality Limits}
\label{app:scaling}

We benchmarked OT-ICA against FastICA using a linear mixture of continuous Laplace sources
across increasing dimensions with a fixed sample size of $N=10{,}000$. This experiment
isolates the effects of dimensionality on both the statistical recovery rate (Amari error)
and the computational execution time.

\subsection{Statistical Efficiency and the Dimension Limit}

\begin{figure}[htbp]
    \centering
    \includegraphics[width=1.0\textwidth]{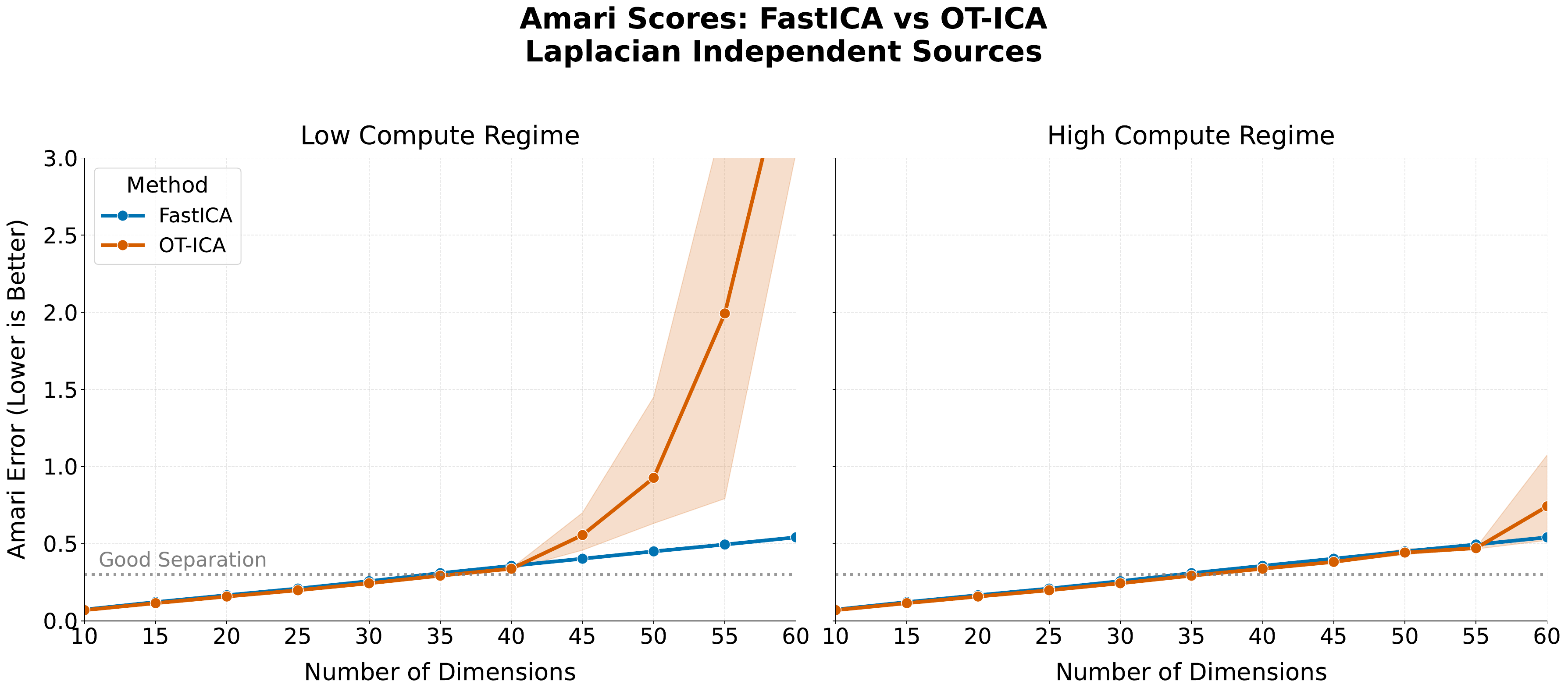}
    \caption{Amari error across dimensions for Laplacian sources at fixed $N=10{,}000$
    samples. Both OT-ICA and FastICA degrade sharply around $d=40$. As argued below, this
    is an optimisation-landscape effect rather than a statistical one: the contrast is
    estimated on one-dimensional projections, whose error does not depend on the ambient
    dimension.}
    \label{fig:amari_scaling_supp}
\end{figure}

As shown in Figure~\ref{fig:amari_scaling_supp}, OT-ICA maintains error rates closely
comparable to FastICA ($E<0.3$, indicating good separation) in lower dimensions, and both
algorithms lose unmixing quality as the dimension approaches $d=40$.

It is tempting to attribute this to the curse of dimensionality familiar from empirical
optimal transport, where the expected error of the empirical Wasserstein distance between
$d$-dimensional measures scales as $\mathcal{O}(N^{-1/d})$ \citep{fournier2015on}. That
rate does not apply here. OT-ICA never estimates a $d$-dimensional transport cost: the
contrast is $W_2^2$ between the one-dimensional projection
$\mathbf{b}^\top\widetilde{\mathbf{X}}$ and $\mathcal{N}(0,1)$, and the relevant rate is
the univariate one, $\mathcal{O}_P(N^{-1/2})$ for the contrast at a fixed direction under mild tail conditions, with no
dependence on the ambient dimension. Dimension enters only through the need for uniform
control over $\mathbf{b}\in\mathbb{S}^{d-1}$.
Slicing the problem to one dimension is therefore a genuine advantage of the transport
contrast rather than a liability.

The degradation at $d=40$ is instead consistent with the optimisation-landscape account
developed in Appendix~\ref{app:discrete}. Two observations support this. First, the contrast
signal at the true source directions does not decay with $d$
(Table~\ref{tab:discrete_sweep}), which a statistical-error explanation would require;
what decays is the basin fraction, the share of random initialisations from which gradient
ascent reaches a true direction. Second, the effect is not specific to the transport
contrast: FastICA degrades over the same range despite estimating only a low-order moment,
which no argument about empirical Wasserstein rates could explain. The shared limit is that
the search space is the orthogonal group, whose dimension grows as $d(d-1)/2$, while the
initialisation budget does not.

\subsection{Computational Execution Time}

\begin{figure}[htbp]
    \centering
    \includegraphics[width=1.0\textwidth]{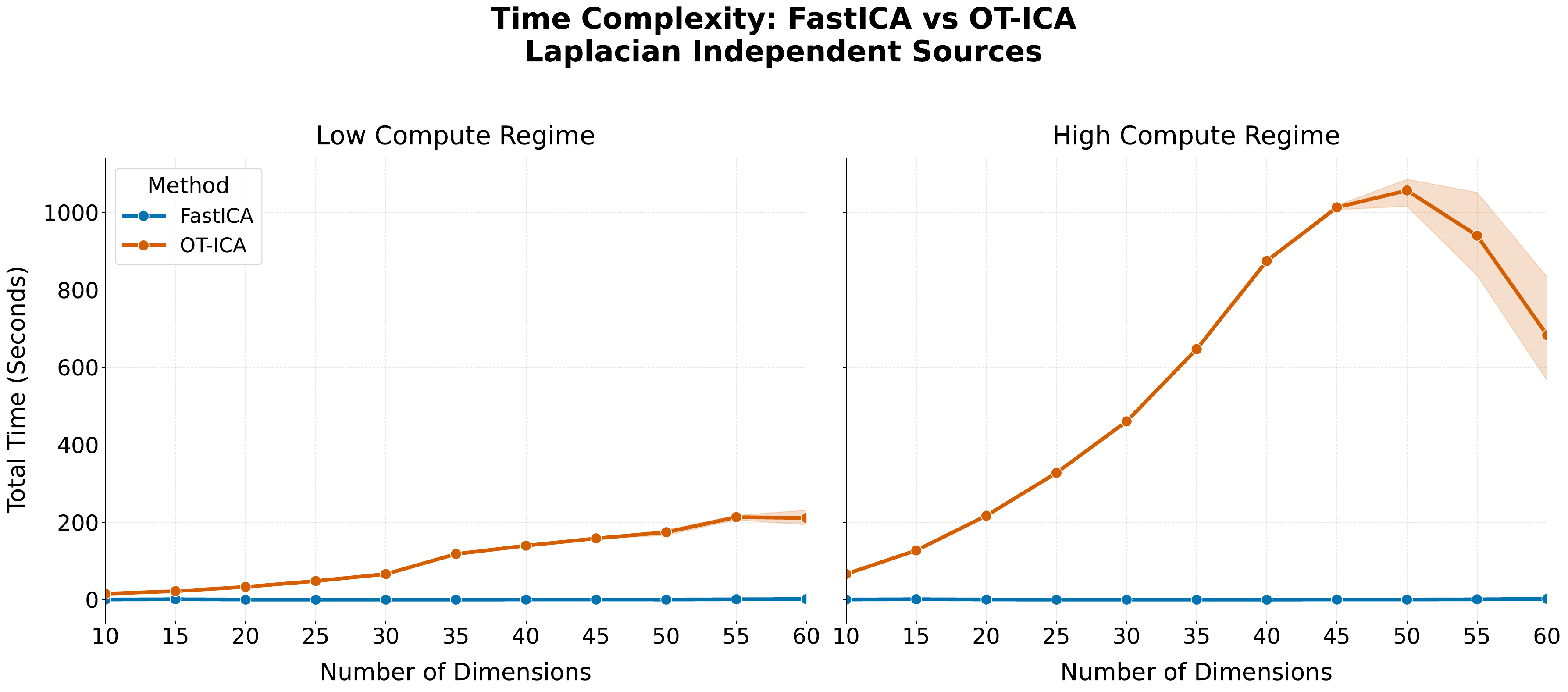}
    \caption{Total execution time (in seconds) for FastICA and OT-ICA across increasing
    dimensions.}
    \label{fig:time_scaling_supp}
\end{figure}

Figure~\ref{fig:time_scaling_supp} shows the computational trade-off required by optimal
transport. Symmetric-mode FastICA updates $d$ rows per iteration, each requiring an
$\mathcal{O}(dN)$ projection, giving $\mathcal{O}(d^2N + d^3)$ per iteration; this makes it
highly efficient on homogeneous continuous mixtures that fall within its operational
assumptions. OT-ICA incurs the same projection cost, adds a sort of the projected samples
at every gradient step, and repeats the initialisation sweep $K$ times, giving
\begin{equation}
    \mathcal{O}\big(K \cdot (d^2 N_{\text{batch}} + d N_{\text{batch}} \log N_{\text{batch}} + d^3)\big)
\end{equation}
per iteration, where $N_{\text{batch}}$ rather than $N$ appears because Algorithm~1 of the main paper is
mini-batched and the $d^3$ term is the symmetric decorrelation used for retraction,
negligible when $N_{\text{batch}} \gg d$. The two methods therefore share the same
$d$-scaling; the difference is a factor $K$ for the restarts and a $\log N_{\text{batch}}$
sorting overhead. OT-ICA trades this execution time for reliability on heterogeneous and
strongly non-Gaussian mixtures where proxy contrasts fail structurally.

\section{Experimental Methodology and Metrics}
\label{app:experimental_setup}

\subsection{The Generalized Permutation Matrix and Amari Index}
To empirically evaluate the success of source separation, we quantify the distance between the estimated unmixing matrix $\mathbf{B}_{\text{est}}$ and 
the true mixing matrix $\mathbf{A}$. Because ICA is subject to scaling and permutation ambiguities, the target recovery is a generalized permutation matrix. 
Let $\mathbf{P} = \mathbf{B}_{\text{est}} \mathbf{A}$ be the global transfer matrix. For a $d \times d$ matrix $\mathbf{P}$ with elements $p_{ij}$, 
the Amari Performance Index \citep{amari1996new} measures the structural divergence:
\begin{equation}
    E(\mathbf{P}) = \frac{1}{2d} \sum_{i=1}^d \left( \sum_{j=1}^d \frac{|p_{ij}|}{\max_k |p_{ik}|} - 1 \right) + \frac{1}{2d} \sum_{j=1}^d \left( \sum_{i=1}^d \frac{|p_{ij}|}{\max_k |p_{kj}|} - 1 \right).
\end{equation}

\subsection{Computational Resource Regimes}

Because OT-ICA utilizes stochastic gradient descent while FastICA employs fixed-point Newton iterations, we standardize resources via two defined regimes to ensure fairness:
\begin{itemize}
    \item \textbf{Low Compute Baseline:} OT-ICA is allocated restarts $\min(4 \times d, 150)$.
    FastICA is allocated 1,000 maximum iterations (and 10 random restarts in the failure-mode experiments).
    \item \textbf{High Compute Regime:} OT-ICA is allocated restarts $\min(15 \times d, 600)$.
    FastICA is allocated 5,000 maximum iterations in the scaling experiment, and 10,000 maximum iterations with 50 random restarts in the failure-mode experiments.
\end{itemize}
OT-ICA deflation/symmetric iteration counts for each experiment are listed in Table~\ref{tab:all_params}.

\subsection{Hardware Specification for All Experiments}
All experiments were conducted using Python 3.10 on an Ubuntu 22.04 workstation with an AMD Ryzen 9 CPU, 64 GB RAM, and an NVIDIA RTX 4090 GPU.

\subsection{Methodology Validation: Global Transfer Matrices}
\label{app:validation_table}

We validated OT-ICA on a 4D mixture of Laplace, Uniform, Student-$t(3)$, and Beta$(0.5, 0.5)$ sources with $N=10{,}000$, mixed by the random matrix $\mathbf{A}$ below. The global transfer matrices $\mathbf{P}=\mathbf{B}_{\text{est}}\mathbf{A}$ for both algorithms are near-permutation, with Amari errors below $0.02$.

\begin{table}[htbp]
\centering
\small
\begin{tabular}{cc}
\toprule
\multicolumn{2}{c}{\textbf{True Mixing Matrix $\mathbf{A}$}} \\
\midrule
\multicolumn{2}{c}{$\begin{bmatrix}
 1.058 &  1.520 & -0.249 &  1.012 \\
 0.042 & -0.486 &  0.388 & -0.523 \\
-0.154 & -1.572 & -0.304 & -1.234 \\
-0.006 &  0.503 &  0.053 &  0.928
\end{bmatrix}$} \\
\midrule
\textbf{OT-ICA Transfer Matrix $\mathbf{P}$} & \textbf{FastICA Transfer Matrix $\mathbf{P}$} \\
\midrule
$\begin{bmatrix}
-0.012 & -0.009 & -1.000 & -0.003 \\
 0.000 &  0.003 & -0.001 &  1.000 \\
-0.007 &  1.000 &  0.000 &  0.012 \\
-1.000 & -0.007 &  0.005 & -0.010
\end{bmatrix}$
&
$\begin{bmatrix}
-1.000 & -0.012 &  0.003 & -0.008 \\
 0.012 & -1.000 &  0.001 & -0.004 \\
-0.002 & -0.010 & -0.001 & -1.000 \\
-0.010 & -0.010 & -1.000 & -0.001
\end{bmatrix}$ \\
\midrule
\multicolumn{2}{c}{\textbf{Amari Error}: OT-ICA $= 0.0155$ \quad FastICA $= 0.0188$} \\
\bottomrule
\end{tabular}
\caption{Global transfer matrices for the 4D validation mixture. Off-diagonal entries are $\leq 0.012$ in absolute value for both algorithms, confirming near-perfect source recovery.}
\label{tab:baseline_matrices}
\end{table}

\subsection{Full Benchmark Results}
\label{app:full_results}

Table~\ref{tab:full_amari_results} reports mean Amari errors for all five methods across all mixture regimes at $d \in \{10, 20, 30\}$ ($N=10{,}000$, 10 trials). Values at 1.500 indicate total separation failure (clip ceiling). \textbf{Bold} marks the lowest error per row. 

\begin{table}[htbp]
\centering
\setlength{\tabcolsep}{10pt}
\begin{tabular}{llccccc}
\toprule
\textbf{Regime} & $d$ & \textbf{OT-ICA} & \textbf{FastICA} & \textbf{JADE} & \textbf{InfoMax} & \textbf{Picard} \\
\midrule
\multirow{3}{*}{Continuous Only}
 & 10 & \textbf{0.050} & 0.092 & 0.146 & 0.100 & 0.093 \\
 & 20 & \textbf{0.106} & 0.183 & 0.321 & 0.188 & 0.184 \\
 & 30 & \textbf{0.160} & 0.271 & 0.489 & 0.277 & 0.271 \\
\midrule
\multirow{3}{*}{Strictly Super-Gaussian}
 & 10 & \textbf{0.057} & 0.089 & 0.146 & 0.103 & 0.090 \\
 & 20 & \textbf{0.113} & 0.183 & 0.316 & 0.190 & 0.184 \\
 & 30 & \textbf{0.176} & 0.280 & 0.482 & 0.286 & 0.280 \\
\midrule
\multirow{3}{*}{Zero Gaussian}
 & 10 & \textbf{0.042} & 0.118 & 0.152 & 0.105 & 0.121 \\
 & 20 & \textbf{0.103} & 0.327 & 0.366 & 0.315 & 0.403 \\
 & 30 & \textbf{0.207} & 0.510 & 0.686 & 0.549 & 0.628 \\
\midrule
\multirow{3}{*}{Full Hybrid}
 & 10 & \textbf{0.059} & 0.255 & 0.216 & 0.297 & 0.277 \\
 & 20 & \textbf{0.261} & 0.451 & 0.571 & 0.558 & 0.555 \\
 & 30 & \textbf{0.335} & 0.671 & 0.938 & 0.765 & 0.917 \\
\midrule
\multirow{3}{*}{Discrete Only}
 & 10 & 0.706 & 0.613 & \textbf{0.443} & 0.757 & 0.835 \\
 & 20 & 1.331 & 1.402 & \textbf{1.267} & 1.495 & 1.440 \\
 & 30 & 1.500 & 1.500 & 1.500 & 1.500 & 1.500 \\
\midrule
\multirow{3}{*}{Pure Laplace ($\mu=0,\, b=1/\sqrt{2}$)}
 & 10 & \textbf{0.077} & 0.080 & 0.136 & 0.083 & 0.080 \\
 & 20 & \textbf{0.165} & 0.170 & 0.285 & 0.177 & 0.170 \\
 & 30 & 0.280 & \textbf{0.261} & 0.439 & 0.272 & 0.262 \\
\midrule
\multirow{3}{*}{Pure Uniform ($-\sqrt{3},\, \sqrt{3}$)}
 & 10 & \textbf{0.038} & 0.054 & 0.051 & 0.056 & 0.054 \\
 & 20 & \textbf{0.094} & 0.116 & 0.110 & 0.122 & 0.116 \\
 & 30 & 1.500 & 0.179 & \textbf{0.171} & 0.187 & 0.180 \\
\midrule
\multirow{3}{*}{Pure Student-$t$ ($\nu=3$)}
 & 10 & 0.068 & 0.065 & 0.134 & \textbf{0.064} & 0.065 \\
 & 20 & 0.142 & \textbf{0.135} & 0.284 & 0.136 & 0.136 \\
 & 30 & 0.219 & \textbf{0.207} & 0.436 & 0.208 & 0.208 \\
\midrule
\multirow{3}{*}{Pure Chi-square $\chi^2(k=2)$}
 & 10 & \textbf{0.039} & 0.094 & 0.120 & 0.087 & 0.094 \\
 & 20 & \textbf{0.085} & 0.205 & 0.255 & 0.190 & 0.205 \\
 & 30 & \textbf{0.129} & 0.315 & 0.392 & 0.292 & 0.316 \\
\midrule
\multirow{3}{*}{Pure Exponential ($\lambda=1$)}
 & 10 & \textbf{0.039} & 0.094 & 0.120 & 0.087 & 0.094 \\
 & 20 & \textbf{0.085} & 0.205 & 0.255 & 0.190 & 0.205 \\
 & 30 & \textbf{0.129} & 0.315 & 0.392 & 0.292 & 0.316 \\
\bottomrule
\end{tabular}
\caption{Mean Amari error ($\downarrow$ better, clipped at 1.500) for all five methods and mixture regimes ($N=10{,}000$, 10 trials). OT-ICA achieves the lowest error on all continuous and heterogeneous mixtures. Since $\chi^2(2)$ is twice an Exponential$(1)$ variable, the two pure sources coincide after standardization (with the seeds used, sample by sample), so their rows are identical.}
\label{tab:full_amari_results}
\end{table}

\subsection{Statistical Significance of the Main Benchmark}
\label{app:sign_test_main}
For each of the 15 (dimension, mixture regime) cells above, we compare OT-ICA to each
baseline with a two-sided sign test on the 10 paired trials sharing the same random draw
of sources and mixing matrix. The advantage documented in the main text is overwhelmingly
significant: of the 60 baseline comparisons across the five mixture regimes, 47 are
significant at $\alpha=0.05$ in OT-ICA's favor (11/15 vs.\ FastICA, 12/15 vs.\ JADE, 12/15
vs.\ InfoMax, 12/15 vs.\ Picard); the single exception is Discrete Only at $d=10$ against
JADE, where every solver sits near the failure ceiling. On the pure-distribution rows in
Table~\ref{tab:full_amari_results}, where the table already does not bold OT-ICA
uniformly, the sign test is consistent with the reported means throughout, including the
handful of cells where a baseline is significantly lower (Pure Uniform at $d=30$; Pure
Student-$t$), matching the entries left unbolded above. Table~\ref{tab:sign_test_summary}
gives the full breakdown by direction.

\begin{table}[htbp]
\centering
\small
\begin{tabular}{lcc|cc}
\toprule
& \multicolumn{2}{c}{Hybrid regimes (of 15)} & \multicolumn{2}{c}{Pure-distribution rows (of 15)} \\
Baseline & OT-ICA favored & baseline favored & OT-ICA favored & baseline favored \\
\midrule
FastICA & 11 & 0 & 8  & 4 \\
JADE    & 12 & 1 & 13 & 1 \\
InfoMax & 12 & 0 & 10 & 4 \\
Picard  & 12 & 0 & 8  & 4 \\
\bottomrule
\end{tabular}
\caption{Sign test (paired by trial, $n=10$), cells significant at $\alpha=0.05$, split by
direction. ``Baseline favored'' cells are where the baseline significantly beats OT-ICA;
see the discussion above for what these correspond to.}
\label{tab:sign_test_summary}
\end{table}

\section{Discrete Mixtures: Contrast Strength versus Optimisation Geometry}
\label{app:discrete}

This appendix supports the discrete-source analysis in the main text. It reports the
protocol for the two landscape diagnostics, the second landscape configuration, the
$d=2$ benchmark, the per-restart outcome distribution, and the full dimension sweep.

\subsection{Diagnostic Protocol}
\label{supp:discrete_protocol}

All diagnostics operate on whitened data, so the search space is the unit sphere in
$\mathbb{R}^d$ and the true source directions are the rows of
$\mathbf{S}_{\text{std}}\mathbf{X}_w^{+}$, normalised to unit length.

\paragraph{Contrast signal.} We evaluate the contrast at the true source directions and
at $1500$ random unit directions, and report
$z = (\bar{c}_{\text{true}} - \bar{c}_{\text{rand}})/\mathrm{sd}(c_{\text{rand}})$. This
measures whether the objective at the solution remains separable from its value at a
typical point. The statistic rises with $d$ in part because random projections become
more Gaussian by the central limit theorem, which lowers $\bar{c}_{\text{rand}}$ and
narrows its spread; the quantity is therefore evidence that the signal does not decay,
not that it improves.

\paragraph{Basin fraction.} We draw $300$ random unit vectors and run projected
gradient ascent on the contrast for $400$ steps (learning rate $0.05$, decayed by a
factor $0.999$ per step, gradient projected onto the tangent space and renormalised to
the sphere at each step). A run counts as successful if the final iterate satisfies
$\max_i |\langle \mathbf{w}, \mathbf{w}_i^{\ast}\rangle| > 0.99$ for some true direction
$\mathbf{w}_i^{\ast}$. This is the procedure OT-ICA's deflation phase performs with
$K$ parallel initialisations, so the fraction estimates the per-restart success
probability $p$, and the phase succeeds with probability $1 - (1-p)^K$.

\paragraph{Marginal detectability.} The $z$-statistics in the main text are computed
against an empirical null of $400$ Gaussian samples of size $N=10^4$ per contrast, with
the analytical Gaussian quantile target for $W_2^2$ and
$\mathbb{E}[\log\cosh \nu] = 0.3746$ evaluated by Gauss-Hermite quadrature for the
logcosh proxy.

\subsection{Landscape with One Gaussian Source}
\label{supp:landscape_gauss}

The landscapes of Figure~\ref{fig:optim_surfaces} in the main text contain two
discrete sources; Figure~\ref{fig:landscape_gauss} replaces one of them by a Gaussian.
In both cases the $W_2^2$ contrast has a corner at each true discrete direction, which
Appendix~\ref{app:atoms} explains.

\begin{figure}[htbp]
    \centering
    \includegraphics[width=\textwidth]{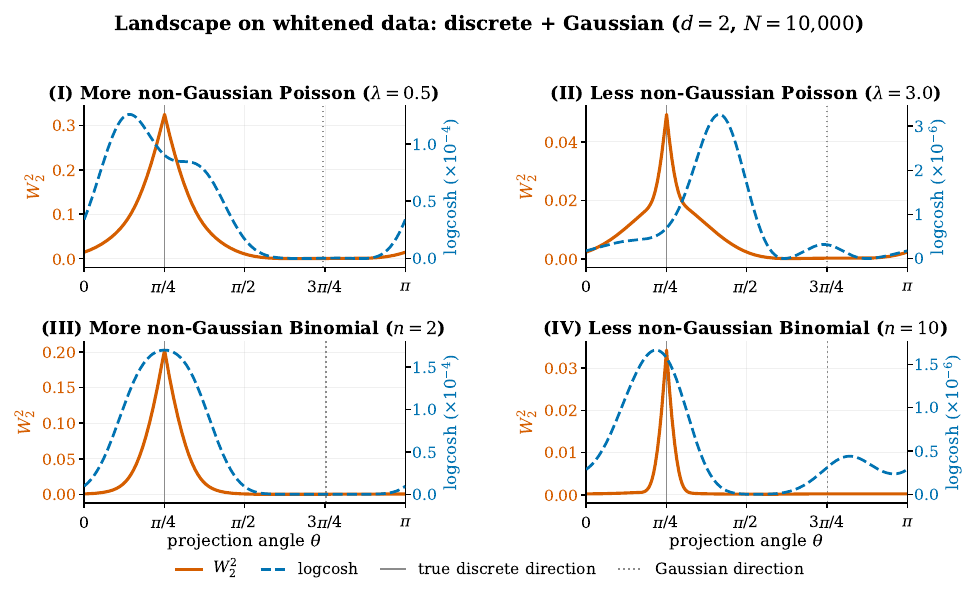}
    \caption{Contrast against projection angle for one discrete and one Gaussian source
    ($d=2$, $N=10^4$), plotted under the same conventions as the two-discrete case in
    the main text. The whitened basis is rotated to place the discrete source at
    $\pi/4$; the Gaussian direction, marked dotted, lies at $3\pi/4$. Compared with the
    two-discrete landscape, the $W_2^2$ surface here is unimodal on the half-period:
    the secondary maxima appear only once every source is discrete.}
    \label{fig:landscape_gauss}
\end{figure}

Table~\ref{tab:peak_error} gives the angular displacement between each contrast's
maximum and the nearest true discrete direction, for both source configurations.
$W_2^2$ localises to within $0.12^\circ$ in all eight cases. The logcosh proxy
localises correctly only in configuration (III), and its error is larger in the
two-discrete case, where no Gaussian source is present to anchor the surface.

\begin{table}[htbp]
\centering
\small
\setlength{\tabcolsep}{6pt}
\begin{tabular}{lcccc}
\toprule
& \multicolumn{2}{c}{discrete + Gaussian} & \multicolumn{2}{c}{two discrete} \\
\cmidrule(lr){2-3}\cmidrule(lr){4-5}
Configuration & $W_2^2$ & logcosh & $W_2^2$ & logcosh \\
\midrule
(I) Poisson ($\lambda=0.5$)  & 0.08 & 19.5 & 0.12 & 44.2 \\
(II) Poisson ($\lambda=3.0$) & 0.08 & 29.5 & 0.09 & 39.8 \\
(III) Binomial ($n=2$)       & 0.08 & 0.1  & 0.10 & 0.1  \\
(IV) Binomial ($n=10$)       & 0.08 & 5.9  & 0.09 & 7.6  \\
\bottomrule
\end{tabular}
\caption{Peak-recovery error in degrees: angular distance from each contrast's global
maximum to the nearest true discrete source direction, at $d=2$, $N=10^4$, over a grid
of $600$ angles spanning the half-period.}
\label{tab:peak_error}
\end{table}

\subsection{Separation at $d=2$}
\label{supp:discrete_d2}

\begin{figure}[htbp]
    \centering
    \includegraphics[width=0.75\textwidth]{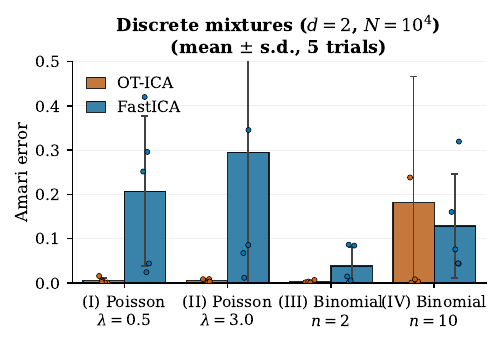}
    \caption{Amari error on purely discrete two-source mixtures, mean and standard
    deviation over five trials with individual trials overlaid. OT-ICA attains
    $0.002$ to $0.005$ in configurations (I) to (III); in (IV) its mean of $0.181$
    carries a standard deviation of $0.285$, reflecting a bimodal outcome distribution
    rather than a uniform degradation.}
    \label{fig:discrete_amari_d2}
\end{figure}

\begin{table}[htbp]
\centering
\small
\setlength{\tabcolsep}{6pt}
\begin{tabular}{lcc}
\toprule
Configuration & OT-ICA & FastICA \\
\midrule
(I) Poisson ($\lambda=0.5$)  & \textbf{0.0047}\,(0.0062) & 0.2070\,(0.1699) \\
(II) Poisson ($\lambda=3.0$) & \textbf{0.0043}\,(0.0037) & 0.2948\,(0.3961) \\
(III) Binomial ($n=2$)       & \textbf{0.0023}\,(0.0024) & 0.0382\,(0.0431) \\
(IV) Binomial ($n=10$)       & 0.1813\,(0.2850) & \textbf{0.1284}\,(0.1168) \\
\bottomrule
\end{tabular}
\caption{Mean Amari error (standard deviation) over five trials, $d=2$, $N=10^4$.
Configuration (IV) is the one case in which OT-ICA does not lead, and is also the case
with the smallest basin fraction in Table~\ref{tab:discrete_sweep}.}
\label{tab:discrete_d2}
\end{table}

A paired sign test (by trial, $n=5$) on the same runs as Table~\ref{tab:discrete_d2} gives
OT-ICA 5/5, 5/5, and 5/5 wins in configurations (I)--(III), consistent with the large
margins in the table above; in configuration (IV), the one case OT-ICA does not lead on
average, it wins 2/5 trials. At $n=5$ per configuration, a clean sweep is needed to reach
$\alpha=0.05$ ($p=0.0625$ for 5/5), so none of the four configurations individually
reaches significance, but the win pattern in (I)--(III) is directionally consistent with
the means throughout.

\begin{table}[htbp]
\centering
\small
\begin{tabular}{lcc}
\toprule
Configuration & OT-ICA wins (of 5) & $p$ \\
\midrule
(I) Poisson ($\lambda=0.5$)  & 5/5 & 0.0625 \\
(II) Poisson ($\lambda=3.0$) & 5/5 & 0.0625 \\
(III) Binomial ($n=2$)       & 5/5 & 0.0625 \\
(IV) Binomial ($n=10$)       & 2/5 & 1.000  \\
\bottomrule
\end{tabular}
\caption{Sign test (paired by trial, $n=5$), OT-ICA vs.\ FastICA, on the same runs as
Table~\ref{tab:discrete_d2}.}
\label{tab:discrete_d2_sign}
\end{table}

\subsection{Per-Restart Outcome Distribution}
\label{supp:restart}

\begin{figure}[htbp]
    \centering
    \includegraphics[width=\textwidth]{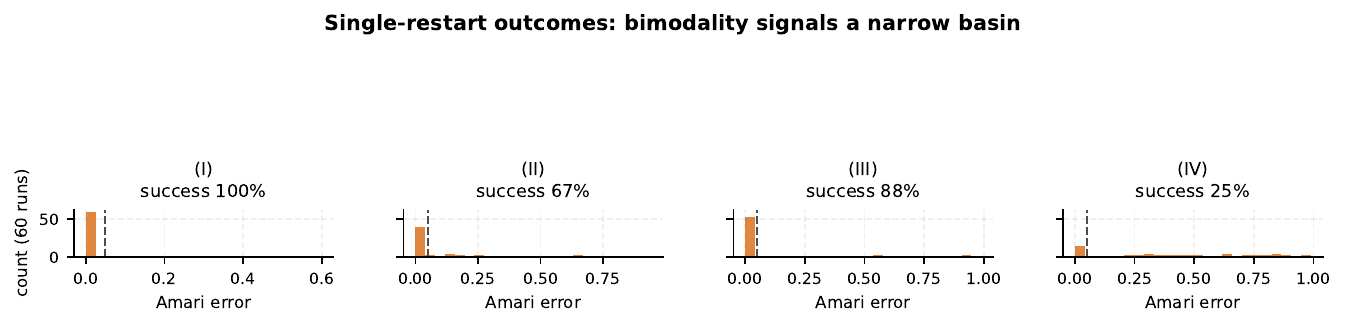}
    \caption{Amari error over $60$ runs per configuration with a single deflation
    restart ($d=2$, $N=10^4$). Filled markers fall below the success threshold of
    $0.05$, open markers above it; percentages give the success rate. Outcomes separate
    into a cluster near $10^{-3}$ and a cluster near $10^{-1}$ to $10^{0}$, with the
    band between them unpopulated, which is the signature of a narrow basin of
    attraction rather than a weak contrast.}
    \label{fig:restart_bimodality}
\end{figure}

Success rates are $100\%$, $67\%$, $88\%$ and $25\%$ for configurations (I) to (IV),
with medians $0.0034$, $0.0066$, $0.0038$ and $0.4130$. The interquartile ranges
separate the same way: $[0.0022, 0.0053]$ and $[0.0024, 0.0069]$ for the two
configurations that succeed reliably, against $[0.0033, 0.1567]$ and
$[0.0552, 0.7118]$ for (II) and (IV). Ordering the configurations by single-restart
success rate broadly reproduces their ordering by basin fraction at $d=2$ in
Table~\ref{tab:discrete_sweep}, which is the correspondence the diagnostic is intended
to establish.

\subsection{Full Dimension Sweep}
\label{supp:discrete_sweep}

Table~\ref{tab:discrete_sweep} extends the $d=2$ comparison above across dimensions $d\in \{2,4,6,8,10\}$, tracking contrast signal and basin fraction together so the two can be read against each other at each scale.

\begin{table}[htbp]
\centering
\small
\setlength{\tabcolsep}{3.5pt}
\begin{tabular}{llccccccccccc}
\toprule
& & \multicolumn{5}{c}{contrast signal $z$} & & \multicolumn{5}{c}{basin fraction} \\
\cmidrule(lr){3-7}\cmidrule(lr){9-13}
Config & Contrast & $2$ & $4$ & $6$ & $8$ & $10$ & & $2$ & $4$ & $6$ & $8$ & $10$ \\
\midrule
\multirow{2}{*}{(I)}
 & $W_2^2$   & 3.5 & 7.6  & 14.6  & 21.5  & 33.3  & & 0.673 & 0.760 & 0.707 & 0.720 & 0.750 \\
 & logcosh   & 1.1 & 0.5  & 1.3   & 2.0   & 2.8   & & 0.100 & 0.010 & 0.000 & 0.000 & 0.000 \\
\midrule
\multirow{2}{*}{(II)}
 & $W_2^2$   & 4.2 & 10.4 & 15.6  & 20.9  & 28.8  & & 0.437 & 0.300 & 0.193 & 0.063 & 0.020 \\
 & logcosh   & $-0.1$ & 1.0 & 1.7 & 1.4   & 1.7   & & 0.087 & 0.003 & 0.000 & 0.000 & 0.000 \\
\midrule
\multirow{2}{*}{(III)}
 & $W_2^2$   & 3.8 & 12.6 & 33.4  & 81.2  & 180.1 & & 0.423 & 0.340 & 0.243 & 0.153 & 0.077 \\
 & logcosh   & 1.9 & 3.0  & 5.1   & 7.7   & 12.2  & & 0.057 & 0.000 & 0.000 & 0.000 & 0.000 \\
\midrule
\multirow{2}{*}{(IV)}
 & $W_2^2$   & 6.1 & 35.5 & 212.1 & 307.3 & 360.9 & & 0.140 & 0.030 & 0.003 & 0.000 & 0.000 \\
 & logcosh   & 2.1 & 1.5  & 2.3   & 3.3   & 5.1   & & 0.057 & 0.000 & 0.000 & 0.000 & 0.000 \\
\bottomrule
\end{tabular}
\caption{Contrast signal and basin fraction against dimension $d$, $N=10^4$, one
Gaussian source and $d-1$ discrete sources per mixture. For $W_2^2$ the signal grows
with $d$ in every configuration while the basin fraction falls in three of four; for
logcosh the signal stays at or below $z=12.2$ throughout and the basin is empty from $d=6$.
Configuration (I), the most non-Gaussian, is the one whose basin does not collapse, and
is also the configuration OT-ICA separates most reliably at $d=2$
(Table~\ref{tab:discrete_d2}).}
\label{tab:discrete_sweep}
\end{table}

\section{Proxy Contrast Based ICA Algorithms}
\label{app:proxy_limits}
This appendix collects a brief overview of the baselines used in the paper.
\subsection{FastICA: Negentropy Approximation Failures}
FastICA \citep{hyvarinen2001} approximates negentropy using a non-quadratic contrast function $G(\cdot)$, such as the \textit{logcosh} function ($\mathbb{E}[G(x)] - \mathbb{E}[G(\nu)]$, where $\nu \sim \mathcal{N}(0,1)$). We demonstrate two limitations of this proxy approach.

\subsubsection{The Zero Negentropy Condition}
If a latent independent source possesses a non-Gaussian distribution such that its expected value under the contrast function equals that of a Gaussian ($\mathbb{E}[G(s_i)] = \mathbb{E}[G(\nu)]$), the approximated negentropy evaluates to zero. In this scenario, the objective function provides no gradient signal, and the algorithm fails to extract the source. As shown in Figure~\ref{fig:zero_negentropy}, this failure leads to increasing Amari error with increasing dimensions.

\begin{figure}[htbp]
    \centering
    \includegraphics[width=1.0\textwidth]{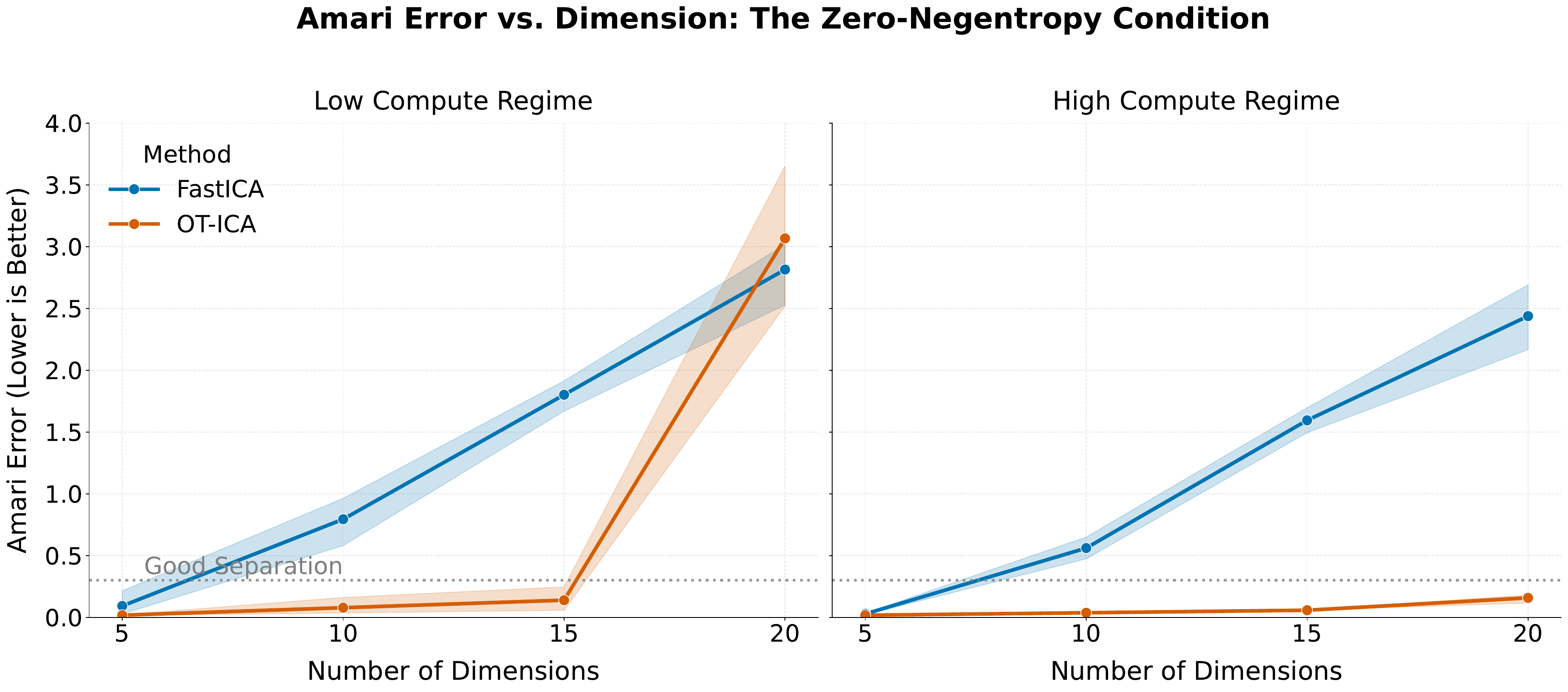} 
    \caption{Amari error comparison for an engineered zero-negentropy distribution across Low and High compute regimes. FastICA's error scales with dimension as the proxy gradient vanishes. OT-ICA maintains good separation.}
    \label{fig:zero_negentropy}
\end{figure}

\subsubsection{The Vanishing Curvature Condition}
FastICA utilizes a Newton fixed-point iteration. Defining $g(x) = G'(x)$ and $g'(x) = G''(x)$, the algorithm approximates the Hessian matrix. For a given component $i$, the diagonal entry scales according to the expectation $\mathbf{H}_{ii} \approx \mathbb{E}\big[s_i g(s_i) - g'(s_i)\big]$.

The Newton update requires applying the inverse of this Hessian ($\mathbf{H}^{-1}$). If a non-Gaussian source distribution causes the expectation in the denominator to evaluate to zero ($\mathbb{E}[s_i g(s_i) - g'(s_i)] \to 0$), the inverse becomes undefined \citep{hyvarinen2001}. This analytical divide-by-zero error makes the fixed-point update step intractable. Figure~\ref{fig:vanishing_curvature_combined} illustrates this failure mode.

\begin{figure}[!tb]
    \centering
    \begin{minipage}{1.0\textwidth}
        \centering
        \includegraphics[width=0.9\textwidth]{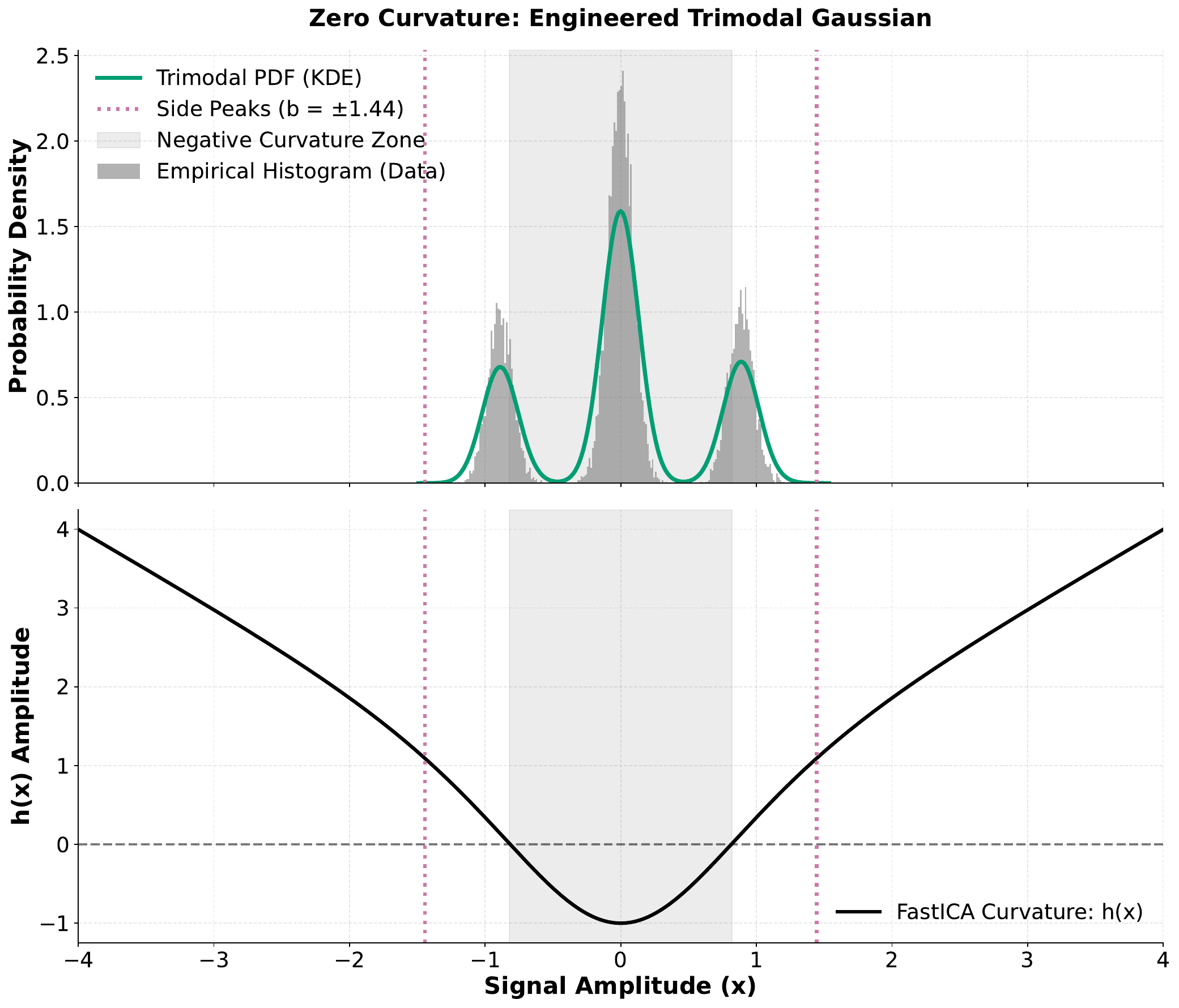}
    \end{minipage}
    \vspace{0.5cm}
    \begin{minipage}{1.0\textwidth}
        \centering
        \includegraphics[width=0.9\textwidth]{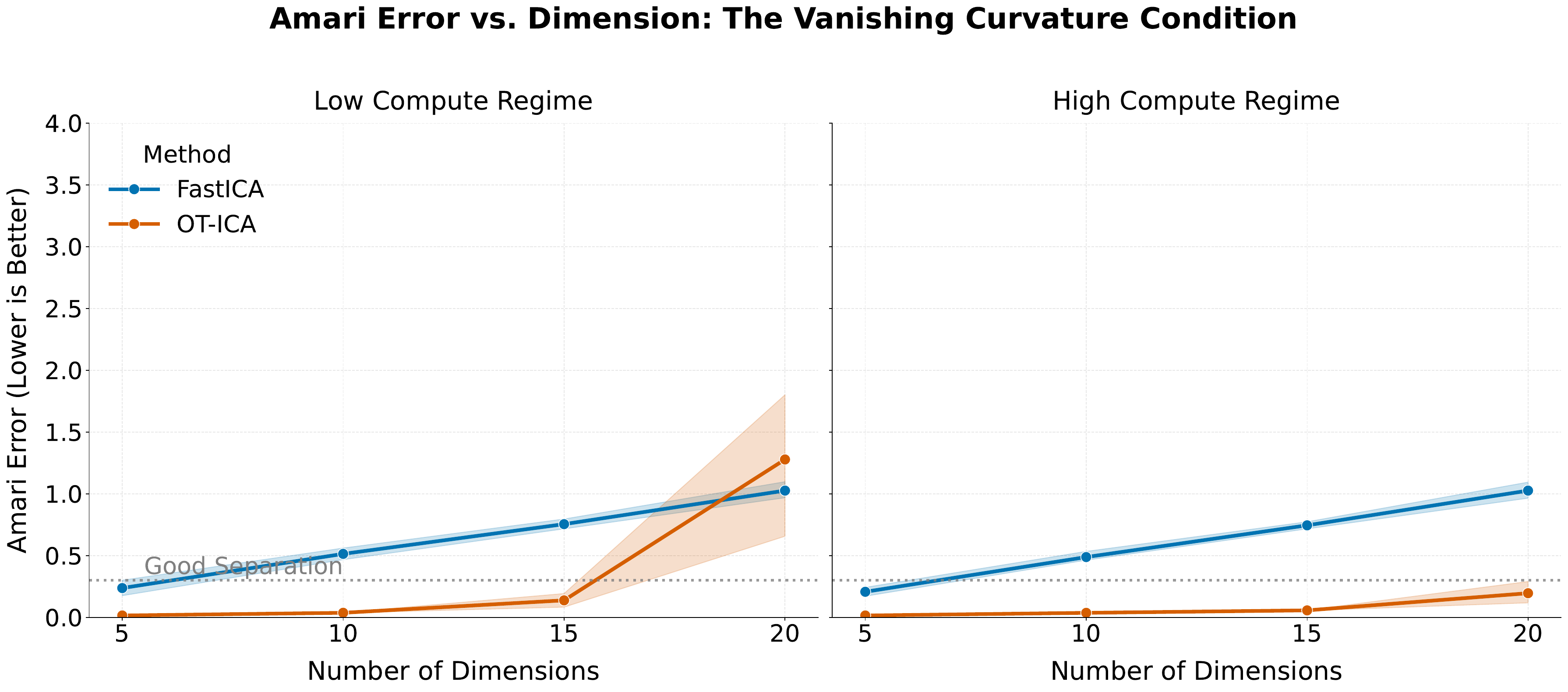}
    \end{minipage}
    \caption{The Vanishing Curvature failure mode. \textbf{Top:} A trimodal distribution engineered such that negative central curvature and positive outer curvature cancel ($\mathbb{E}[sg(s)-g'(s)] = 0$). \textbf{Bottom:} Amari error comparison on mixtures containing this distribution. FastICA's Newton solver diverges rapidly. OT-ICA maintains good separation.}
    \label{fig:vanishing_curvature_combined}
\end{figure}

Per-cell sign tests (paired by trial, $n=5$) support this reading directionally: OT-ICA
wins the majority of trials in 7 of 8 dimension/compute-level cells in the
vanishing-curvature experiment and all 8 of 8 in the zero-negentropy experiment,
consistent with the separation visible in Figures~\ref{fig:vanishing_curvature_combined}
and~\ref{fig:zero_negentropy}. At this trial count ($n=5$) individual cells cannot reach
$\alpha=0.05$ on their own.

\subsection{JADE: Sample Complexity of Fourth-Order Cumulants}
JADE \citep{cardoso1993blind} identifies sources by jointly diagonalizing a set of fourth-order cumulant matrices via Jacobi sweeps. 
Reliable estimation of the full cumulant tensor requires $\mathcal{O}(d^4)$ samples \citep{cardoso1993blind, hyvarinen2001}. 
At $d=30$, this implies approximately $810{,}000$ samples; against the $N=10{,}000$ available in our benchmark, the cumulant matrices are heavily undersampled. 
The resulting joint diagonalization converges on estimated tensors reflecting sampling noise rather than true source structure, yielding $E=0.489$ 
on continuous sources and $E=0.938$ on hybrid sources at $d=30$, both approaching the failure threshold despite JADE's consistency on well-sampled data.

\subsection{InfoMax: Score Function Misspecification}
InfoMax \citep{bell1995information} maximizes a log-likelihood under a fixed source-density model. The original logistic model implicitly treats all sources as super-Gaussian,
so sub-Gaussian components (e.g., uniform) are misspecified. \citet{lee1999extended} proposed extended InfoMax, which switches online between a super- and a sub-Gaussian density for each component.
Our InfoMax baseline is this extended variant, fitted without orthogonality constraint using the solver of \citet{ablin2018faster}.
Its density model still offers only two shapes, so components that neither describes well, such as the discrete sources in the hybrid regimes, remain misspecified; OT-ICA makes no such density assumption.
In our benchmark, InfoMax produces $E=0.765$ on hybrid sources and $E=0.549$ on Zero Gaussian sources at $d=30$, against OT-ICA's $0.335$ and $0.207$ respectively,
consistent with the misspecified-score explanation.

\subsection{Picard: The Contrast as the Binding Constraint}
Picard \citep{ablin2018faster} optimizes the InfoMax log-likelihood with a preconditioned quasi-Newton (L-BFGS) solver instead of gradient ascent.
Our Picard baseline uses its orthogonally constrained variant with the same extended tanh density model as our InfoMax baseline. In our benchmark, FastICA and Picard produce Amari errors within $0.003$ of each other
on the Continuous Only and Strictly Super-Gaussian regimes and on every pure-distribution row at each dimension (Table~\ref{tab:full_amari_results}). This near-identical performance isolates the source of failure: solver speed is not the binding constraint. 
The shared tanh nonlinearity, and not the optimizer, determines the error floor. Replacing it with an exact, distribution-free contrast reduces error by $40$--$45\%$ on continuous sources.

\section{Datasets}
\label{app:data}

Three datasets appear in the paper. All are publicly available and none required licensing
or preprocessing beyond what is stated below.

\paragraph{MNE sample dataset (EEG artifact removal).} The MNE-Python sample recording
\citep{gramfort2013meg} provides simultaneous EEG, MEG and electrooculogram (EOG) channels
from a single subject under an audiovisual paradigm, distributed openly with the
MNE-Python package. We use $20$ EEG channels over a $60$\,s segment containing nine
blinks, bandpass filtered $1$--$40$\,Hz with a zero-phase filter and $z$-scored per
channel. We chose it because the concurrently recorded EOG channel provides a reference
for the ocular artifact that is entirely independent of the excess-kurtosis criterion used
to select the component, so the evaluation does not depend on the same signal that drives
the decision. Its limitation is that scalp EEG contains many near-Gaussian neural sources,
so the identifiability condition of at most one Gaussian component holds only
approximately, and a single subject is a thin basis for a general claim; we accordingly
treat this application as a check that the transport contrast behaves correctly on a real
recording rather than as a demonstration of advantage.

\paragraph{PBMC 3k (linear causal disentanglement).} The 10x Genomics PBMC 3k dataset
\citep{pbmc3k_dataset} contains unique-molecular-identifier (UMI) counts for $2{,}700$
peripheral blood mononuclear cells across $32{,}738$ genes, accessed through
\texttt{scanpy}. We use it for its empirical count marginals only: exogenous shocks are
resampled independently per gene, so independence holds by construction while the shocks
retain realistic count structure, and the latent graph, mixing map and intervention
targets remain known. This follows the semi-synthetic protocol of
\citet{squires2023linear}, who likewise fit parameters on real measurements and then
re-simulate in order to retain ground truth. It matters here because idealised Poisson
counts and real UMI counts differ sharply in overdispersion, and overdispersion is
precisely the axis along which the logcosh proxy succeeds or fails
(Table~\ref{tab:pbmc_genes}); using real marginals is what makes that boundary
measurable rather than assumed.

\paragraph{Sachs protein-signalling network (causal discovery).} The flow-cytometry
dataset of \citet{sachs2005causal} measures $11$ phosphoproteins in $7{,}466$ human immune
cells, with a consensus network of $17$ directed edges established from the biological
literature. Intensities are log-normal, so we apply $x\mapsto\log(1+x)$ and standardize;
every variable retains at least mild non-Gaussianity afterwards (excess kurtosis $-0.10$ to $6.56$). It is
the standard benchmark for LiNGAM-style causal discovery, which makes our results
comparable with prior work, and its sample size relative to $d=11$ places it in the regime
where the transport contrast is well estimated. Two known limitations bear on the
interpretation: the mechanisms are nonlinear, and the pooled file concatenates several
interventional conditions, which violates the i.i.d.\ assumption. We retain the standard
pooled configuration and treat the experiment as a comparison of solvers under a fixed
pipeline rather than as a test of absolute structure recovery.

\section{Application: Structural Identification in Price Discovery}
\label{app:price}

The Information Share \citep{hasbrouck1995, baillie2002} attributes price discovery across
venues through the structural decomposition $u_t=\boldsymbol{\Theta}\epsilon_t$, with
$u_t$ VECM reduced-form residuals \citep{engle1987, johansen1991}. The covariance
restriction $\boldsymbol{\Omega}=\boldsymbol{\Theta}\boldsymbol{\Theta}^\top$ identifies
$\boldsymbol{\Theta}$ only up to an orthogonal rotation, which the conventional Cholesky
factorisation resolves by imposing an arbitrary ordering of markets. Writing
$\boldsymbol{\Theta}=\mathbf{S}\mathbf{Q}$ with $\mathbf{S}=\boldsymbol{\Omega}^{1/2}$
\citep{zema2025} reduces identification to recovering the rotation $\mathbf{Q}$, which
OT-ICA obtains from the non-Gaussianity of the innovations without any ordering
assumption.

We simulate a three-market VECM with one common efficient price (cointegrating rank $2$,
$T=5{,}000$), Student-$t$ innovations ($\nu\in\{5,6,7\}$), and a non-diagonal
$\boldsymbol{\Theta}_{\text{true}}$ giving target shares $[0.12,\,0.24,\,0.64]$ under
$\psi=(1,1,1)^\top$. The system is estimated by OLS with known cointegrating vectors; ICA
is applied to the residuals, with shocks assigned to markets by maximal own-market impact
\citep{kuhn1955hungarian} and signs normalised to positive own impact.

\begin{table}[htbp]
\centering
\small
\setlength{\tabcolsep}{4pt}
\begin{tabular}{@{}lcccc@{}}
\toprule
Market & True & OT-ICA & FastICA & Cholesky range \\
\midrule
1 & 0.12 & 0.1197\,(0.0223) & 0.1196\,(0.0217) & $[0.012,\,0.430]$ \\
2 & 0.24 & 0.2419\,(0.0393) & 0.2423\,(0.0462) & $[0.069,\,0.603]$ \\
3 & 0.64 & 0.6385\,(0.0447) & 0.6381\,(0.0514) & $[0.098,\,0.849]$ \\
\bottomrule
\end{tabular}
\caption{Information Share recovery, $500$ runs, $T=5{,}000$. Mean (s.d.); the final
column gives the range of admissible values across all $3!$ Cholesky orderings.}
\label{tab:simulated_is_results}
\end{table}

Over $500$ replications both solvers recover all three shares with bias below $0.003$;
Student-$t$ innovations are heavy-tailed and well within the operating range of the
logcosh proxy, so a tie is expected. The comparison that matters here is against Cholesky
rather than against FastICA: the Hasbrouck bounds across all six orderings are wide enough
to be economically decisive, since market~$1$ admits values up to $0.430$ and market~$3$
values as low as $0.098$, so the choice of ordering alone can invert the ranking of
venues. Recovering the rotation from non-Gaussianity removes that indeterminacy at
negligible cost in precision.

A paired sign test ($n=500$ replications) on $|\text{share}-\text{true share}|$ finds no
significant difference between OT-ICA and FastICA on any market, consistent with the
near-tie already reported above; the economically decisive comparison in this application
is against the Cholesky-ordering bounds, not between the two ICA solvers.

\begin{table}[htbp]
\centering
\small
\begin{tabular}{lc}
\toprule
Market & $p$ \\
\midrule
1 & 0.227 \\
2 & 0.502 \\
3 & 0.754 \\
\bottomrule
\end{tabular}
\caption{Sign test ($n=500$ paired replications), OT-ICA vs.\ FastICA, on
$|\text{share}-\text{true share}|$.}
\label{tab:price_sign}
\end{table}

\section{Application: Linear Causal Disentanglement}
\label{app:disentangle}

\paragraph{Setup.} We simulate $d=4$ latents and $p=8$ observed variables, one
observational and $d$ interventional environments, $n=2\times10^4$ samples, $10$ seeds.
$\mathbf{A}_0$ is strictly upper triangular with density $0.75$ and non-zero weights drawn
from $\pm[0.25,1]$; $\boldsymbol{\Omega}_0$ entries from $\mathrm{Unif}[2,4]$, and an
intervention on node $i$ zeroes row $i$ of $\mathbf{A}$ and redraws $\Omega_{ii}$ from
$\mathrm{Unif}[6,8]$, a disjoint range so that interventions on root nodes remain
detectable by magnitude. $\mathbf{G}$ is drawn with $\mathrm{cond}(\mathbf{G})<50$ and
rescaled so that $\mathbf{G}^\dagger$ satisfies the scaling convention of
\citet{squires2023linear}. Both solvers return unit-variance sources, which removes the
ICA scale ambiguity and leaves only a signed permutation. The full details are stated in
Algorithm~\ref{alg:disentangle}.

\begin{algorithm}[!tb]
\caption{Causal disentanglement via multi-environment ICA}
\label{alg:disentangle}
\begin{algorithmic}[1]
\Require Observations $\mathbf{X}_k\in\mathbb{R}^{p\times n}$, $k=0,\dots,K$; dimension $d$
\Ensure $\widehat{\mathbf{H}}$, targets $\{\hat{i}_k\}$, observational index $\hat{u}$, adjacency $\widehat{\mathbf{A}}$
\State $\mathbf{V}\gets$ top-$d$ left singular vectors of $[\mathbf{X}_0\,\cdots\,\mathbf{X}_K]$
\For{$k=0$ to $K$}
  \State $\mathbf{B}_k\gets\textsc{OT-ICA}(\mathbf{V}^\top\mathbf{X}_k)$, unit-variance sources
\EndFor
\State Align rows of $\{\mathbf{B}_k\}$ by assignment on $-\log(1-|\cos|)$; let $\mathbf{C}$ be the row-wise medoid
\State $\delta_{ik}\gets\|(\mathbf{B}_k)_{i,:}-\mathbf{C}_{i,:}\|$
\State Assign rows to environments maximising $\sum_i\delta_{i,k_i}$; set $\hat{i}_{k_i}\gets i$, and let $\hat{u}$ be the unassigned environment
\State $\widehat{\mathbf{H}}_{i,:}\gets(\mathbf{B}_{k_i})_{i,:}\mathbf{V}^\top$, rescaled so $\max_j|\widehat{H}_{ij}|=1$
\State $\widehat{\mathbf{G}}\gets\widehat{\mathbf{H}}^\dagger$; \; $\widehat{\mathbf{M}}_k\gets\mathbf{B}_k(\widehat{\mathbf{H}}\mathbf{V})^{-1}$
\State $\widehat{\mathbf{A}}\gets\mathbf{I}-\mathrm{diag}(\widehat{\mathbf{M}}_{\hat{u}})^{-1}\widehat{\mathbf{M}}_{\hat{u}}$
\State \Return $\widehat{\mathbf{H}}$, $\{\hat{i}_k\}$, $\hat{u}$, $\widehat{\mathbf{A}}$
\end{algorithmic}
\end{algorithm}

\paragraph{Noise regimes.} \emph{Laplace}: all shocks Laplace. \emph{Count (int.)}:
Laplace baseline with the intervened node's shock replaced by Poisson$(0.5)$.
\emph{Count (all)}: all shocks Poisson with rates $0.5$ to $1.0$. \emph{PBMC 3k}: shocks
resampled from the empirical UMI-count marginals of four genes, with interventions
modelled as binomial thinning at rate $0.25$. \emph{Leakage}: as Count (int.), but the
intervention also perturbs a neighbouring shock at correlation $0.6$, violating the
single-node assumption.

\paragraph{Validation with exact unmixing matrices.} Substituting the exact
$\mathbf{B}_k=\mathbf{M}_k\mathbf{G}^{\dagger}$ for the estimated ones, the pipeline recovers
$\|\Delta\mathbf{H}\|_F<10^{-14}$, all intervention targets, the observational
environment and the graph on all $20$ seeds tested. The remaining error in
Table~\ref{tab:disentangle_full} is therefore attributable to the ICA step alone.

\paragraph{Scale-invariant demixing error.} The Assumption 1(c) normalisation divides
each row of $\mathbf{H}$ by its largest-magnitude entry. When a row's two largest entries
are within a few percent, estimation noise can flip which is selected, rescaling the row
by $h_{j^*}/h_{k^*}$ without changing its direction or the recovered graph, yet inflating
$\|\Delta\mathbf{H}\|_F$ by up to $2\|\mathbf{h}_i\|$. We therefore draw $\mathbf{G}$ so
that the two largest entries of every row of $\mathbf{G}^\dagger$ differ by at least
$15\%$, and additionally report the scale- and sign-invariant
$\tfrac{1}{d}\sum_i\big(1-|\cos\angle(\widehat{\mathbf{h}}_i,\mathbf{h}_i)|\big)$. No run
in Table~\ref{tab:disentangle_full} exhibits a flip.

\begin{table}[htbp]
\centering
\small
\setlength{\tabcolsep}{3pt}
\begin{tabular}{llccccccc}
\toprule
Shocks & Method & Amari & $\|\Delta\mathbf{H}\|_F$ & dir.\ err. & target & F1 & SHD & rows \\
\midrule
\multirow{2}{*}{Laplace}
 & OT-ICA  & 0.019\,(0.002) & 0.066\,(0.014) & $2\!\times\!10^{-4}$ & 1.00\,(0.00) & 1.00\,(0.00) & 0.0 & 1.00 \\
 & FastICA & 0.019\,(0.002) & 0.064\,(0.018) & $2\!\times\!10^{-4}$ & 1.00\,(0.00) & 1.00\,(0.00) & 0.0 & 1.00 \\
\midrule
\multirow{2}{*}{Count (int.)}
 & OT-ICA  & 0.014\,(0.002) & 0.002\,(0.001) & $<\!10^{-4}$ & 1.00\,(0.00) & 1.00\,(0.00) & 0.0 & 1.00 \\
 & FastICA & 0.024\,(0.003) & 0.125\,(0.039) & $8\!\times\!10^{-4}$ & 1.00\,(0.00) & 1.00\,(0.00) & 0.0 & 1.00 \\
\midrule
\multirow{2}{*}{Count (all)}
 & OT-ICA  & 0.008\,(0.001) & 0.031\,(0.009) & $1\!\times\!10^{-4}$ & 1.00\,(0.00) & 1.00\,(0.00) & 0.0 & 1.00 \\
 & FastICA & 0.835\,(0.151) & 3.400\,(1.308) & 0.300\,(0.165) & 0.48\,(0.36) & 0.49\,(0.10) & 7.7 & 3.18 \\
\midrule
\multirow{2}{*}{PBMC, near-Poisson}
 & OT-ICA  & 0.009\,(0.001) & 0.021\,(0.008) & $<\!10^{-4}$ & 1.00\,(0.00) & 1.00\,(0.00) & 0.0 & 1.00 \\
 & FastICA & 0.063\,(0.034) & 0.467\,(0.843) & 0.008\,(0.017) & 0.95\,(0.16) & 0.95\,(0.09) & 0.5 & 1.03 \\
\midrule
\multirow{2}{*}{PBMC, overdisp.}
 & OT-ICA  & 0.008\,(0.001) & 0.028\,(0.015) & $1\!\times\!10^{-4}$ & 1.00\,(0.00) & 1.00\,(0.00) & 0.0 & 1.00 \\
 & FastICA & 0.012\,(0.001) & 0.042\,(0.018) & $1\!\times\!10^{-4}$ & 1.00\,(0.00) & 1.00\,(0.00) & 0.0 & 1.00 \\
\midrule
\multirow{2}{*}{Leakage}
 & OT-ICA  & 0.149\,(0.001) & 0.787\,(1.683) & 0.032\,(0.087) & 0.83\,(0.37) & 0.90\,(0.21) & 1.2 & 2.00 \\
 & FastICA & 0.157\,(0.004) & 0.890\,(1.639) & 0.033\,(0.086) & 0.83\,(0.37) & 0.90\,(0.21) & 1.2 & 2.00 \\
\bottomrule
\end{tabular}
\caption{Causal disentanglement, mean (s.d.) over $10$ seeds, $d=4$, $p=8$,
$n=2\times10^4$. \emph{dir.\ err.} is the scale- and sign-invariant demixing error;
\emph{rows} counts rows of $\mathbf{B}_k-\mathbf{B}_0$ differing from consensus, which
the single-node assumption requires to equal one.}
\label{tab:disentangle_full}
\end{table}

\paragraph{Gene selection and marginal detectability.} Shocks are resampled independently
per gene, so only the marginals enter the simulation and cross-gene correlation in the
original data is irrelevant to the independence assumption; we report it for
transparency. Genes are selected by Fano factor $\sigma^2/\mu$ within a common
mean-expression window of $[1,20]$ counts. Table~\ref{tab:pbmc_genes} scores each
selected marginal as in Table~\ref{tab:w2_sensitivity}. The overdispersed set rates at
$9.8$ to $13.4$ times a Laplace source under the logcosh proxy, comfortably within its
operating range; the near-Poisson set rates at $0.16$ to $1.15$, overlapping the range of
synthetic Poisson ($0.196$ at $\lambda=0.5$, $0.003$ at $\lambda=3.0$). The $W_2^2$
contrast rates every marginal at $3.4$ or above. Ribosomal and other constitutively
expressed genes dominate the near-Poisson set, as expected: their variability is
dominated by technical sampling rather than biological signal, so the right tail on which
a $\tanh$-shaped nonlinearity depends is absent.

\begin{table}[htbp]
\centering
\small
\setlength{\tabcolsep}{5pt}
\begin{tabular}{llccccc}
\toprule
Set & Gene & mean & Fano & $W_2^2$ rel. & logcosh rel. \\
\midrule
\multirow{4}{*}{near-Poisson}
 & NBEAL1   & 1.30 & 1.28 & 4.56 & 0.155 \\
 & TOMM7    & 1.81 & 1.61 & 5.03 & 1.146 \\
 & RPL22    & 2.14 & 1.63 & 3.46 & 0.230 \\
 & RPSAP58  & 1.71 & 1.66 & 4.38 & 0.467 \\
\midrule
\multirow{4}{*}{overdispersed}
 & TMSB10   & 13.82 & 10.60 & 9.93 & 9.779 \\
 & SAT1     & 2.75  & 10.81 & 15.74 & 12.595 \\
 & FCER1G   & 2.04  & 10.93 & 17.55 & 13.379 \\
 & COTL1    & 2.99  & 11.05 & 15.58 & 12.548 \\
\bottomrule
\end{tabular}
\caption{PBMC 3k marginals used as exogenous shocks. Detectability is relative to a
Laplace source under the same contrast, computed as in Table~\ref{tab:w2_sensitivity}.
Maximum $|\rho|$ among the selected genes is $0.345$ (near-Poisson) and $0.697$
(overdispersed); independent resampling makes this immaterial to the simulation.}
\label{tab:pbmc_genes}
\end{table}

\paragraph{Assumption violation.} The \emph{rows} column is a ground-truth-free test of
Assumption 1(b) of \citet{squires2023linear}, which requires an intervention to edit a
single row of $\mathbf{B}_k$. It equals $1.00$ wherever the model holds and $2.00$ with
zero variance under leakage, so the objection that an intervention may perturb more than
one exogenous shock is checkable from data alone. \citet{squires2023linear} test a
related condition through the spectral energy of precision-matrix differences, which
requires a tuned threshold; reading the rows of $\mathbf{B}_k$ directly requires none. At
correlation $0.6$ both solvers still identify $83\%$ of targets and attain F1 $0.90$, so
the violation degrades recovery without destroying it. The value of $3.18$ for FastICA
under Count (all) has a different cause: with the ICA step failing, the estimated rows
carry no environment structure at all.

\paragraph{Statistical significance.} A two-sided sign test (paired by seed, $n=10$) on
the Amari error and demixing-map error $\|\Delta\mathbf{H}\|_F$ underlying
Table~\ref{tab:disentangle_full} confirms the pattern is not noise: OT-ICA is
significantly better on both metrics in every regime where the main text reports a large
gap -- Count (int.), Count (all), Leakage, and both PBMC windows all reach
$p\leq0.022$, with $9$ or $10$ of $10$ seeds favoring OT-ICA. The one exception is the
Laplace regime, where both solvers are expected to perform comparably and the test
correctly finds no significant difference ($p=0.34$ and $p=0.75$) - the null
result the near-identical means in Table~\ref{tab:disentangle_full} predict. Target
accuracy, F1 and SHD are tied between solvers on nearly every seed within these regimes
and are not separately tested. Table~\ref{tab:disentangle_sign} gives per-regime win
counts and $p$-values.

\begin{table}[htbp]
\centering
\small
\begin{tabular}{lcccc}
\toprule
Regime & \multicolumn{2}{c}{Amari (wins/10, $p$)} & \multicolumn{2}{c}{$\|\Delta\mathbf{H}\|_F$ (wins/10, $p$)} \\
\midrule
Laplace              & 3/10  & 0.344   & 4/10  & 0.754   \\
Count (int.)         & 10/10 & 0.00195 & 10/10 & 0.00195 \\
Count (all)          & 10/10 & 0.00195 & 10/10 & 0.00195 \\
Leakage              & 10/10 & 0.00195 & 9/10  & 0.0215  \\
PBMC, near-Poisson    & 10/10 & 0.00195 & 10/10 & 0.00195 \\
PBMC, overdispersed   & 10/10 & 0.00195 & 9/10  & 0.0215  \\
\bottomrule
\end{tabular}
\caption{Sign test, paired by seed ($n=10$), OT-ICA vs.\ FastICA, on the ICA-step metrics
underlying Table~\ref{tab:disentangle_full}.}
\label{tab:disentangle_sign}
\end{table}

\section{Application: Causal Discovery on the Sachs Network}
\label{app:sachs}

This appendix gives the LiNGAM post-processing pipeline (Appendix~\ref{supp:pipeline}), the full results (Appendix~\ref{supp:resid}), the recovered structures (Appendix~\ref{supp:sachs_dag}) and semi-synthetic controls (Appendix~\ref{supp:synth}). The dataset is described in Appendix~\ref{app:data}.

\subsection{LiNGAM Pipeline}
\label{supp:pipeline}

Both solvers return $\mathbf{W}=\mathbf{P}\mathbf{D}(\mathbf{I}-\mathbf{B})$ up to unknown
permutation $\mathbf{P}$ and diagonal scaling $\mathbf{D}$. We apply the same three steps to
each. (i) Resolve $\mathbf{P}$ by linear assignment \emph{minimising} $\sum_i 1/|W_{ii}|$
\citep{shimizu2006linear}; maximising $\sum_i|W_{ii}|$ instead admits a near-zero diagonal
entry that the next step divides by, producing coefficients of order $10^2$. (ii) Resolve
$\mathbf{D}$ by dividing each row by its diagonal entry, giving
$\mathbf{B}=\mathbf{I}-\mathbf{W}$. (iii) Obtain the causal order by Algorithm~B of
\citet{shimizu2006linear}, which zeroes the $d(d+1)/2$ smallest entries of $\mathbf{B}$ and
grows the zero set until the result is permutable to strictly lower triangular. Step (iii)
matters: a heuristic sort by row mass is circular, since any matrix becomes more triangular
under a permutation chosen to make it so.

Neither solver receives a tuned threshold; both are scored on their $k=17$ largest
off-diagonal coefficients, at which budget precision, recall and F1 coincide. \emph{Order
consistency} is the fraction of true edges $p\to c$ with $\mathrm{rank}(p)<\mathrm{rank}(c)$
(chance $0.5$); \emph{residual dependence} is the mean $|\rho_{\text{Spearman}}|$ between
$|s_i|$ and $|s_j|$ over all pairs, a ground-truth-free diagnostic for the magnitude
dependence that linear decorrelation cannot detect. OT-ICA uses $96$ to $128$ deflation
restarts followed by $600$ Stiefel iterations (lr $0.05$, batch $1024$,
$\sigma_{\text{dither}}=0.01$); FastICA uses \texttt{logcosh} with tolerance $10^{-6}$.
Results are means over seeds $42$ to $51$; single-run figures use the median-F1 seed.

\subsection{Full Results}
\label{supp:resid}

\begin{table}[htbp]
\centering
\small
\setlength{\tabcolsep}{6pt}
\begin{tabular}{lccc}
\toprule
 & OT-ICA & FastICA & reference \\
\midrule
F1 @ top-17 $\uparrow$            & 0.353\,(0.000) & 0.353\,(0.000) & 0.152 \\
AUC $\uparrow$                    & \textbf{0.694}\,(0.002) & 0.607\,(0.001) & 0.500 \\
Order consistency $\uparrow$      & 0.471\,(0.000) & 0.412\,(0.000) & 0.500 \\
SHD $\downarrow$                  & 20.0\,(0.0) & 20.0\,(0.0) & 17.0 \\
Residual dependence $\downarrow$  & 0.058\,(0.000) & 0.057\,(0.000) & --- \\
Seed Jaccard $\uparrow$           & 0.941 & \textbf{1.000} & --- \\
\bottomrule
\end{tabular}
\caption{Sachs network, mean (s.d.) over $10$ seeds. Reference: random top-$17$ graph (F1),
chance (AUC, order), empty graph (SHD). The empty graph attains a lower SHD than either
solver, an artefact of scoring at a fixed edge budget; we therefore treat the budget-free AUC
as primary.}
\label{tab:sachs_full}
\end{table}

Three results in Table~\ref{tab:sachs_full} require comment. First, the two solvers select
the same $17$ edges, of which $6$ are consensus edges, giving identical F1 and SHD. They
differ only in how they rank the coefficients: OT-ICA attains AUC $0.694$ against $0.607$,
and a paired bootstrap over the $110$ off-diagonal entries gives
$\Delta\mathrm{AUC}=+0.087$ with $95\%$ confidence interval $[+0.018,+0.182]$ and
$P(\Delta\leq0)=0.003$. The edge selection at a hard cut is therefore insensitive to the
contrast on this dataset, while the ordering of the remaining candidates is not.

Second, causal order falls below chance for both solvers. The same pipeline recovers the
order on data simulated from the true Sachs DAG under well-specified noise
(Appendix~\ref{supp:synth}), which validates the parent-first convention end to end, so the failure
reflects the dataset's nonlinear mechanisms and pooled interventional regimes rather than the
implementation.

Third, residual dependence is equal for the two methods, and the semi-synthetic controls
show why it must be: the quantity rises from $\approx0.018$ to $\approx0.145$ for
\emph{both} solvers under location-scale noise (Table~\ref{tab:sachs_synth}), which
identifies it as a measure of model misspecification rather than solver quality. Finally,
FastICA is deterministic here, with zero standard deviation on every metric and a pairwise
Jaccard of $1.000$ between top-$17$ edge sets across seeds, whereas OT-ICA's restart-based
deflation gives $0.941$; we report this as a cost of the transport approach.

A sign test (paired by seed, $n=10$) on the remaining metrics is consistent with this
picture: OT-ICA has significantly higher order consistency on every seed, matching the
means above ($0.471$ vs.\ $0.412$). Residual dependence is the one metric where FastICA is
lower on every seed, though the two means are numerically close ($0.058$ vs.\ $0.057$). F1
and SHD are tied on every seed, as the two solvers select the same top-$17$ edges.
Table~\ref{tab:sachs_sign} reports both tests.

\begin{table}[htbp]
\centering
\small
\begin{tabular}{lcc}
\toprule
Metric & OT-ICA wins (of 10) & $p$ \\
\midrule
Order consistency   & 10/10 & 0.00195 \\
Residual dependence & 0/10  & 0.00195 \\
\bottomrule
\end{tabular}
\caption{Sign test, paired by seed ($n=10$), OT-ICA vs.\ FastICA, on the Sachs network
real-data metrics not already reported as tied.}
\label{tab:sachs_sign}
\end{table}

\subsection{Recovered Structures}
\label{supp:sachs_dag}

Figure~\ref{fig:sachs_heatmaps} shows the estimated unmixing matrices in causal order for both solvers alongside the consensus network; Figure~\ref{fig:sachs_dag} shows the corresponding recovered graphs at the top-$17$ budget. Both use the median-F1 seed and are provided for qualitative comparison alongside the quantitative results above.

\begin{figure}[htbp]
    \centering
    \includegraphics[width=\textwidth]{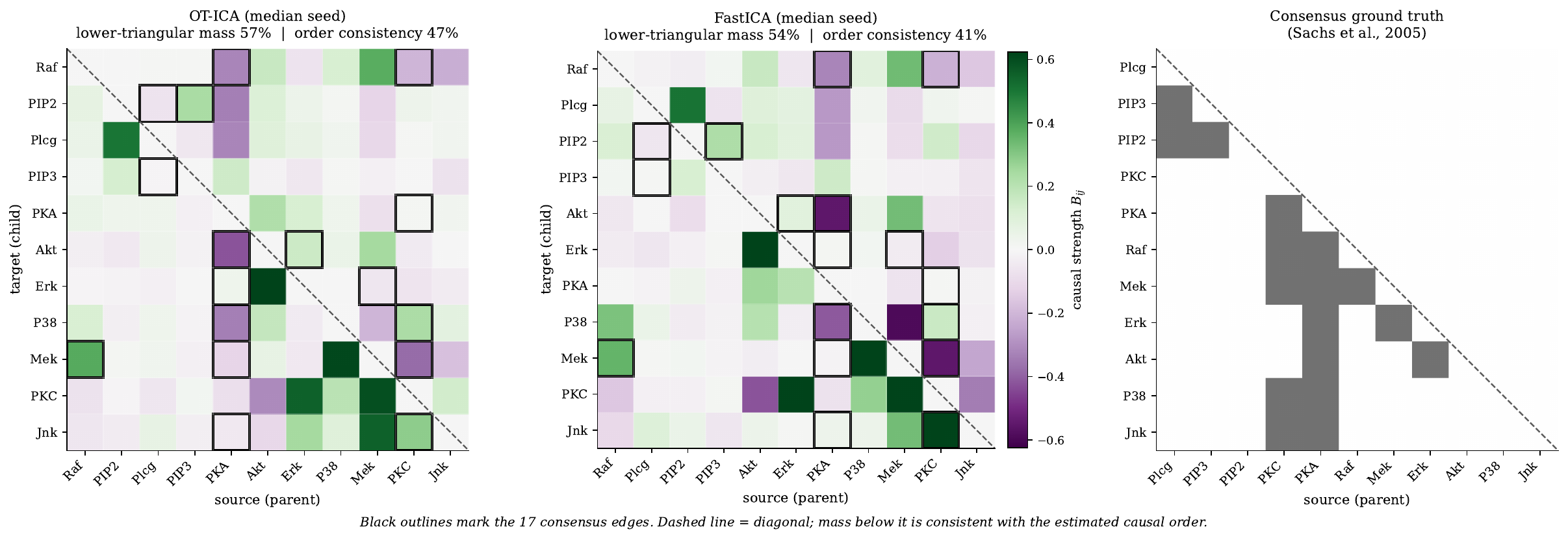}
    \caption{Estimated $\mathbf{B}$ permuted by the causal order of Algorithm~B, for OT-ICA
    (left) and FastICA (centre), with the consensus network (right). Black outlines mark
    consensus edges; mass below the dashed diagonal is consistent with the estimated order.
    Median-F1 seed.}
    \label{fig:sachs_heatmaps}
\end{figure}

\begin{figure}[htbp]
    \centering
    \includegraphics[width=\textwidth]{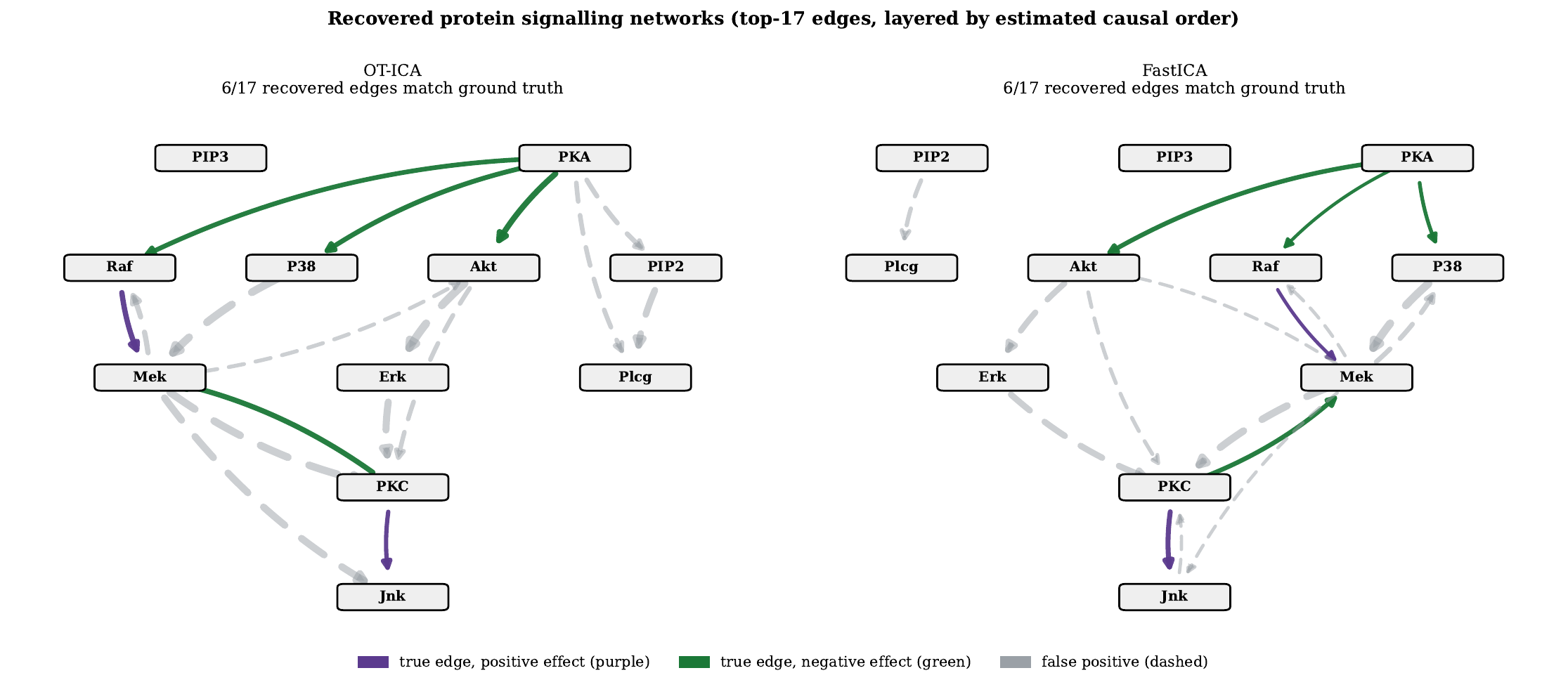}
    \caption{Recovered networks at the top-$17$ budget, layered by estimated causal order.
    Solid edges are recovered consensus edges, coloured by $\mathrm{sign}(B_{ij})$; dashed
    grey edges are false positives. Both solvers identify PKA as a high-out-degree hub.}
    \label{fig:sachs_dag}
\end{figure}

\subsection{Semi-Synthetic Controls}
\label{supp:synth}

To isolate the solver from model misspecification we fix the true $17$-edge DAG, draw
coefficients from $\pm[0.4,1.0]$, and simulate $n=2{,}000$ observations under i.i.d.\ Laplace
shocks, the continuous-discrete mixture of the main benchmark, and location-scale noise in
which a child's dispersion scales with its parents' magnitude.

\begin{table}[htbp]
\centering
\small
\setlength{\tabcolsep}{5pt}
\begin{tabular}{llcccc}
\toprule
Regime & Method & F1 $\uparrow$ & SHD $\downarrow$ & Order $\uparrow$ & Res.\ dep.\ $\downarrow$ \\
\midrule
\multirow{2}{*}{Laplace}
 & OT-ICA  & 0.994\,(0.019) & 0.2\,(0.6) & 1.000 & 0.019 \\
 & FastICA & 1.000\,(0.000) & 0.0\,(0.0) & 1.000 & 0.018 \\
\midrule
\multirow{2}{*}{Mixture}
 & OT-ICA  & 1.000\,(0.000) & 0.0\,(0.0) & 1.000 & 0.018 \\
 & FastICA & 1.000\,(0.000) & 0.0\,(0.0) & 1.000 & 0.018 \\
\midrule
\multirow{2}{*}{Loc.-scale}
 & OT-ICA  & 0.471\,(0.062) & 14.6\,(2.4) & 0.512\,(0.130) & 0.148 \\
 & FastICA & 0.506\,(0.069) & 13.1\,(1.7) & 0.465\,(0.085) & 0.142 \\
\bottomrule
\end{tabular}
\caption{True Sachs DAG, $d=11$, $n=2{,}000$, mean (s.d.) over $10$ seeds. Both solvers
saturate under the well-specified regimes; under location-scale noise every difference lies
within one standard deviation.}
\label{tab:sachs_synth}
\end{table}

At $d=11$, $n=2{,}000$ both solvers recover the graph under both well-specified regimes,
including the mixture on which OT-ICA holds a clear advantage at $d\in\{10,20,30\}$
(Table~\ref{tab:full_amari_results}); the single OT-ICA seed that drops one edge under
Laplace shocks accounts for its F1 of $0.994$. The problem is saturated at this scale, so
the experiment bounds the solver's contribution rather than discriminating between methods.
Under location-scale noise both degrade to F1 $\approx0.5$ with overlapping error bars, as
expected: that mechanism violates the additive-noise model class, and no contrast can
restore identifiability the model lacks. OT-ICA is nominally ahead on order consistency and
FastICA on F1 and SHD, but each gap is smaller than one standard deviation and none should
be read as a difference.

\paragraph{Sample efficiency.} Repeating the mixture regime across $n$ gives F1 of
$0.594$/$0.724$ at $n=125$, $0.871$/$0.924$ at $250$, $0.976$/$0.988$ at $500$,
$0.994$/$0.988$ at $1{,}000$ and $1.000$/$1.000$ at $2{,}000$ (OT-ICA/FastICA). FastICA leads
at small $n$; the curves cross around $n=1{,}000$ and converge by $n=2{,}000$. This is what the contrasts predict:
\texttt{logcosh} estimates a single expectation at the parametric rate, whereas $W_2^2$
matches all $n$ order statistics, estimating the entire quantile function. That is what makes
it distribution-free and immune to the failures of Appendix~\ref{app:proxy_limits}, and
equally what makes it the higher-variance estimator when $n$ is small, since the tails where
non-Gaussianity lives are the noisiest part of an empirical quantile function. The advantages
we report sit above this threshold ($n=10{,}000$ in the main benchmark, $7{,}466$ on Sachs),
which is the regime in which we claim them.

A sign test confirms the two qualitative claims above with no surprises: under
location-scale misspecification, no metric (F1, SHD, order consistency, or residual
dependence) shows a significant difference between solvers (all $p>0.10$), matching the
``neither solver has a significant edge'' reading already given; and at $n=125$, FastICA is
significantly ahead on F1 (OT-ICA wins $1/10$ seeds, $p=0.02$), confirming the ``FastICA
leads at small $n$'' claim above rather than merely illustrating it. At $n\geq250$ too few
seed-pairs remain untied, both solvers saturating near F1$=1$, for the test to add
information beyond the convergence already shown.

\section{Parameter Values for All Experiments and Applications}
\label{app:parameters}

Table~\ref{tab:all_params} reports, for every experiment and application in the main paper
and this appendix, the sample size $N$, dimension $d$, number of trials or seeds, and the
five OT-ICA hyperparameters referenced in the main text's Algorithm~1 discussion:
restarts $K$, deflation/symmetric iteration counts $T_{\text{defl}}/T_{\text{sym}}$,
learning rate $\eta$, and dither scale $\sigma_{\text{dither}}$. Values are taken directly
from the notebooks in the code supplement. Batch size is $N_{\text{batch}}=1024$
throughout except where noted.

\begin{table}[htbp]
\centering
\scriptsize
\setlength{\tabcolsep}{2.5pt}
\begin{tabular}{lccccccc}
\toprule
Experiment & $N$ & $d$ & trials/seeds & $K$ & $T_{\text{defl}}/T_{\text{sym}}$ & $\eta$ & $\sigma_{\text{dither}}$ \\
\midrule
Main benchmark & 10{,}000 & 10, 20, 30 & 10 & $\min(4d,150)$ & 200/400 & 0.25 & 0.01 \\
Discrete $d=2$ benchmark & 10{,}000 & 2 & 5 & 8 & 200/400 & 0.20 & 0.10 \\
Single-restart outcomes & 10{,}000 & 2 & 60/config & 1 & 200/400 & 0.20 & 0.10 \\
FastICA failure, Low Compute & 10{,}000 & 5--20 & 5 & $\min(4d,150)$ & 150/300 & 0.25 & 0.01 \\
FastICA failure, High Compute & 10{,}000 & 5--20 & 5 & $\min(15d,600)$ & 300/800 & 0.25 & 0.01 \\
Methodology validation & 10{,}000 & 4 & 1 & 50 & 200/400 & 0.25 & 0.01 \\
Causal disentanglement & 20{,}000 & 4 & 10 seeds & 16 & 200/400 & 0.10 & 0.05 \\
Price discovery & 5{,}000\rlap{$^a$} & 3 & 500 & 24 & 200/300 & 0.10/0.50\rlap{$^b$} & ---\rlap{$^c$} \\
Sachs, real + semi-synth. & 7{,}466 / 2{,}000 & 11 & 10 seeds & 96--128 & 200/600\rlap{$^d$} & 0.05\rlap{$^e$} & 0.01 \\
EEG artifact removal & ---\ (60\,s)\rlap{$^f$} & 20 & 1 & 50 & 200/400 & 0.05 & 0.01/0.10\rlap{$^g$} \\
Comp.\ scaling, Low Compute & 10{,}000 & 10--60 & 3 & $\min(4d,150)$ & 150/200 & 1.00\rlap{$^h$} & 0 \\
Comp.\ scaling, High Compute & 10{,}000 & 10--60 & 3 & $\min(15d,600)$ & 300/500 & 1.00\rlap{$^h$} & 0 \\
\bottomrule
\end{tabular}
\caption{OT-ICA hyperparameters by experiment, taken from the code supplement notebooks.
$^a$$N=T=5{,}000$ VECM observations per replication. $^b$Deflation learning rate 0.1,
symmetric (Stiefel) learning rate 0.5. $^c$Dither not applied (continuous VECM residuals).
$^d$Symmetric phase run in 10 stages of 60 iterations each. $^e$Initial learning rate,
decayed geometrically by a factor $0.99$ per stage. $^f$Sample count set by the $60$\,s
recording segment rather than a fixed $N$. $^g$Deflation dither $0.01$, symmetric dither
$0.10$, as coded. $^h$No dithering (continuous Laplace sources only); symmetric phase run
full-batch ($N_{\text{batch}}$ not applicable here, unlike every other row).}
\label{tab:all_params}
\end{table}

\end{document}